\documentclass[twoside]{article}

\usepackage[accepted]{aistats2022}
\usepackage[round]{natbib}

\usepackage[utf8]{inputenc} 
\usepackage[T1]{fontenc}    
\usepackage{hyperref}       
\usepackage{url}            
\usepackage{booktabs}       
\usepackage{amsfonts,amsmath,amssymb,dsfont,amsthm}       
\usepackage{nicefrac}       
\usepackage{microtype}      
\usepackage{mathtools}
\usepackage{comment}
\usepackage{xcolor}
\usepackage{caption}
\usepackage{subcaption}
\usepackage{siunitx}
\usepackage[ruled,vlined]{algorithm2e}
\usepackage{enumitem}
\usepackage{tabularx}

\newcommand\indep{\protect\mathpalette{\protect\independenT}{\perp}}
\def\independenT#1#2{\mathrel{\rlap{$#1#2$}\mkern2mu{#1#2}}}

\DeclareMathOperator*{\argmin}{arg\,min}

\newtheorem{thmlem}{Lemma}
\newtheorem{thmcol}{Corollary}

\newtheorem{thmthm}{Theorem}

\newtheorem{thmasmp}{Assumption}
\newtheorem*{thmidasmp*}{Identifying assumptions}

\theoremstyle{definition}

\newtheorem{thmrem}{Remark}
\newtheorem*{proofsketch*}{Proof sketch}

\newtheorem{manualtheoreminner}{Theorem}
\newenvironment{manualtheorem}[1]{%
  \renewcommand\themanualtheoreminner{#1}%
  \manualtheoreminner
}{\endmanualtheoreminner}
\newtheorem{manuallemmainner}{Lemma}
\newenvironment{manuallemma}[1]{%
  \renewcommand\themanuallemmainner{#1}%
  \manuallemmainner
}{\endmanuallemmainner}

\def\E{\mathbb{E}}
\def\V{\mathrm{Var}}
\def\cov{\mathrm{Cov}}
\def\trace{\mathrm{Tr}}
\def\mse{\mathrm{MSE}}
\def\OLS{\mathrm{OLS}}

\def\LuPTS{\mathrm{LuPTS}}

\def\bfX{\mathbf{X}}
\def\bfY{\mathbf{Y}}

\def\bfx{\mathbf{x}}
\def\bfy{\mathbf{y}}

\def\bfR{\mathbf{R}}
\def\bfeps{\boldsymbol{\epsilon}}

\def\cF{\mathcal{F}}
\def\cG{\mathcal{G}}
\def\cH{\mathcal{H}}

\def\cN{\mathcal{N}}

\def\bbR{\mathbb{R}}

\def\hbeta{\hat{\beta}}
\def\htheta{\hat{\theta}}

\def\hf{\hat{f}}
\def\hg{\hat{g}}
\def\hh{\hat{h}}

\def\hA{\hat{A}}

\def\hX{\hat{X}}
\def\hY{\hat{Y}}

\def\olR{\overline{R}}

\begin{document}

%

%

\twocolumn[

\aistatstitle{Using Time-Series Privileged Information for Provably Efficient Learning of Prediction Models}

\aistatsauthor{Rickard K.A. Karlsson$^\ast$ \And Martin Willbo$^\ast$ \And  Zeshan Hussain}
\aistatsaddress{ Delft University of Technology \And  Research Institute of Sweden \And Massachusetts Institute \\of Technology}

\aistatsauthor{Rahul G. Krishnan \And David Sontag \And Fredrik D. Johansson}
\aistatsaddress{University of Toronto \And Massachusetts Institute \\of Technology \And Chalmers University of Technology} 
 \begin{center} $\ast$ equal contribution \end{center}
]

\runningauthor{Karlsson, Willbo, Hussain, Krishnan, Sontag, Johansson}

\begin{abstract}
    We study prediction of future outcomes with supervised models that use privileged information during learning. The privileged information comprises samples of time series observed between the baseline time of prediction and the future outcome; this information is only available at training time which differs from the traditional supervised learning. Our question is when using this privileged data leads to more sample-efficient learning of models that use only baseline data for predictions at test time.
    We give an algorithm for this setting and prove that when the time series are drawn from a non-stationary Gaussian-linear dynamical system of fixed horizon, learning with privileged information is more efficient than learning without it.  On synthetic data, we test the limits of our algorithm and theory,  both when our assumptions hold and when they are violated. 
    On three diverse real-world datasets, we show that our approach is generally preferable to classical learning, particularly when data is scarce. Finally, we relate our estimator to a distillation approach both theoretically and empirically.
\end{abstract}

\section{INTRODUCTION}

\begin{figure}[t]
    \centering
    \vspace{.5em}
    \includegraphics[width=\columnwidth]{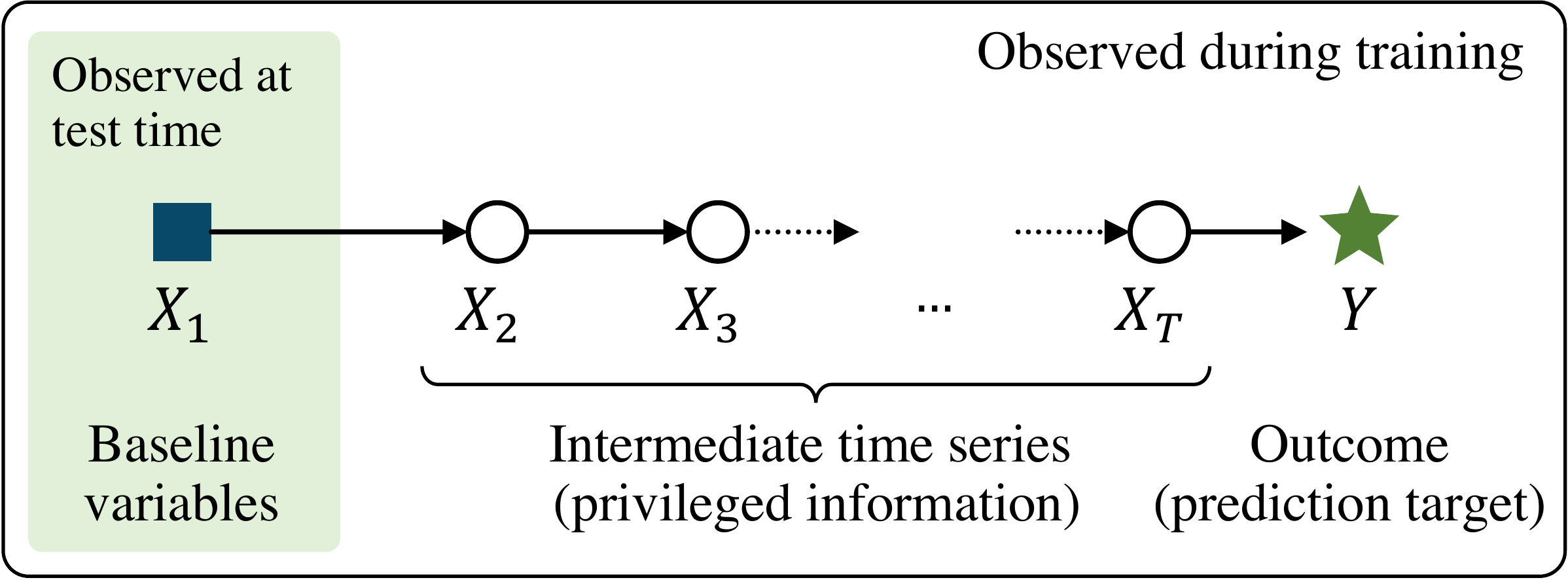}
    \caption{\label{fig:illustration}\small Prediction with intermediate time-series privileged information. The goal at test time is to predict $Y$ based only on $X_1$ but the learner has access to samples from the full series $X_1, X_2, ..., X_T, Y$ during training.}
\end{figure}

Prediction of future outcomes is a central learning problem in many domains. For example, accurate prediction of the progression of chronic disease allows for identification of patients at higher risk and may be used to trigger interventions. Standard supervised learning algorithms for this task minimize the empirical risk in predicting the outcome using features collected at a baseline time point. When data is scarce, variance can plague this approach and reduce its potential impact. However, in practice, data is often collected not only at the time for prediction and the time of the outcome, but at multiple time points between them; in healthcare, disease markers, lab values and treatments  of patients are recorded at regular intervals. These data that could be used for more efficient learning. 

Making use of variables for learning that are unavailable at test time has been called \emph{learning using privileged information} (LuPI)~\citep{vapnik2009new} or \emph{learning with side information}~\citep{jonschkowski2015patterns}. A general way to utilize privileged information is via distillation ~\citep{LopSchBotVap16, hayashi2019long}, where a student model is trained to minimize its error in predicting both true labels and soft targets generated by a teacher trained on the privileged information.
While these paradigms have shown promise both theoretically~\citep{vapnik2009new} and empirically~\citep{hayashi2019long,tang2019retaining}, performance guarantees for practical algorithms remain elusive~\citep{serra2014exploring}. 

In particular, it remains unclear \emph{when} learning using privileged information is preferable to learning without it---as discussed by~\citet{jonschkowski2015patterns}, incorporating privileged information in learning can harm more than it helps. 
This work studies a special case in which the privileged information constitutes the intermediate part of a time series starting with baseline features and ending with the target outcome, see Figure~\ref{fig:illustration}. 

\paragraph{Contributions.} We study a strategy that uses privileged information to learn the dynamics of the full time series
and makes test-time predictions from baseline by recursively simulating the dynamics and the outcome. Instantiating this idea in Gaussian-linear dynamical systems, we compare it to the best unbiased estimator that uses only baseline data---ordinary least squares regression.
In this (well-specified) setting, we prove using a Rao-Blackwell argument~\citep{radhakrishna1945information,blackwell1947conditional}  that our recursive strategy is \emph{always} more sample efficient, without bias and with lower variance, compared to learning without privileged information. Additionally, we show that combining this strategy  with distillation learning leads to a principled way of trading off bias and variance in the misspecified case.   

We study the limits of our theory in synthetic experiments, where our assumptions hold and where they do not. We find that the gap identified by our theory, between our method and the best baseline, grows when more time steps are available and assumptions hold and that it decreases when assumptions are violated (bias is non-zero). Finally, we apply the idea to three diverse real-world problems, where the underlying data-generating process is unknown, showing multiple cases where the approach improves over both non-LuPI and LuPI baselines, and cases where performance is worse.

%
%
\section{PROBLEM SETTING}

We learn models that use baseline variables $X_1 \in \mathbb{R}^d$ to predict outcomes $Y\in \mathbb{R}$, see Figure~\ref{fig:illustration}. For a given loss function, $\mathcal{L} : \bbR \times \bbR \rightarrow \bbR$, our goal is to find a function $h \in \cH \subseteq \{h : \bbR^d \rightarrow \bbR\}$ \emph{of only the baseline variables $X_1$}, which minimizes the expected risk over the variables with respect to a distribution~$p$, 
\begin{equation*}
R(h) \coloneqq \E_{X_1, Y}[\mathcal{L}(h(X_1), Y)]~.
\end{equation*}
This goal is shared with classical supervised learning. However, in our setting, the learner has access to \emph{privileged information} in the form of time series sampled from states $X_2, ..., X_T$, observed chronologically after $X_1$ and before $Y$. The information is \emph{privileged} as it is unavailable when the model is used. Time series $x_{i,1}, ..., x_{i,T}, y_i$, indexed by $i=1, ..., m$, are observed as independent random samples from an unknown distribution $p(X_1, ..., X_T, Y)$. For simplicity,\footnote{That $X_1, ..., X_T$ have the same dimension is not necessary for our main result but simplifies exposition.} we assume that $X_t \in \mathbb{R}^d$ and define the data matrices $\mathbf{X}_t = [x_{1, t}, ..., x_{m, t}]^\top$, for $t=1, ..., T$, and $\mathbf{Y} = [y_1, ..., y_m]^\top$, where rows represent  different series. We let $D = (\bfX_1, ..., \bfX_T, \bfY)$ denote the full dataset.

Without additional assumptions, learning using privileged information (LuPI) need not lead to smaller risk~\citep{vapnik2009new}.
Here, we set out to identify conditions on the distribution $p$ under which there is a LuPI algorithm which is provably better than any algorithm learning only from samples of $(X_1, Y)$. 
Throughout, unless stated otherwise, we assume that the full time series $X_1, X_2, ..., X_T, Y$ is Markov.

\begin{thmasmp}[Markov time series]\label{asmp:markov}%
For all time points $t \in \{3, ..., T\}$, 
$$
X_t \indep X_{1}, ..., X_{t-2} \mid X_{t-1} \;\mbox{  }\;
Y \indep X_{1}, ..., X_{T-1} \mid X_T~.
$$%
\end{thmasmp}%
Under Assumption~\ref{asmp:markov},  $X_1$ is predictive of $Y$ \emph{only through} the privileged information.

\section{LEARNING USING PRIVILEGED TIME SERIES IN LINEAR DYNAMICAL SYSTEMS}
\label{sec:learning}
\SetKwInput{KwAssmp}{Assumption}
\SetKwInput{KwFlag}{Flag}
\SetKwComment{Comment}{$\triangleright$\ }{}

\begin{algorithm}[t!]
    \KwFlag{Stationarity (True / False)}
    \KwData{$ \{(x_{i,1}, ..., x_{i,T}, y_i) \}^{m}_{i=1} \sim p^m(X_1, ..., X_T, Y)$}
    \vspace{.5em}
    \uIf{Stationarity}{
    $\displaystyle{
    \tilde{A} = \argmin_{A} \sum_{t=1}^{T-1} \sum_{i=1}^m \frac{\|  A^\top x_{i,t} - x_{i,t+1}\|_2^2}{m(T-1)}
    }$
    $\hA = \tilde{A}^{T-1}$
    }
    \Else {
        \For{$t = 1, ..., {T-1}$}{
            $\displaystyle{
            \hA_t = \argmin_{A} \sum_{i=1}^m \frac{\|A^\top x_{i,t} - x_{i,t+1}\|_2^2}{m}
            }
            $ 
        }
        $\hA = \hA_1 \hA_2 \cdots \hA_{T-1}$
    }
    $\hbeta = \argmin_{\beta} \frac{1}{m} \sum_{i=1}^m ( \beta^\top x_{i,T} -  y_i)^2$ 
    \vspace{.5em}
    
    \Return $\htheta = \hA \hbeta$
 \caption{Learning using privileged time series (LuPTS) in linear dynamical systems}
 \label{alg:lupts_lin}
\end{algorithm}

A natural strategy for predicting $Y$ in a Markov system is to successively predict $X_2$ from $X_1$, then $X_3$ from the prediction $\hX_2$, and so on. In time-series modeling, this is referred to as \emph{recursive} prediction, in contrast to \emph{direct} prediction~\citep{chevillon2007direct}. Unlike time-series modeling, we study prediction of outcomes $Y$, which are at a fixed time $T$ and distinct in nature from $X$. A survey of related work is found in Section~\ref{sec:related}.

We use the recursive strategy in a linear estimator which learns using privileged time series (LuPTS, Algorithm~\ref{alg:lupts_lin}). The prediction at each step $t$ is made using a learned linear model $\hX_{t+1} = \hA_t^\top X_t$, with $\hA_t \in \bbR^{d \times d}$. The final prediction is given by another linear model, $\hY = \hbeta^\top X_T$, with $\hbeta\in \bbR^d$, learned from samples of $(X_T, Y)$. At test time, only $X_1$ is observed, and the models are combined to form $\hY = (\hA_1 \dots \hA_{T-1}\hbeta)^\top X_1$. Algorithm~\ref{alg:lupts_lin} has a flag to indicate whether the transitions are assumed to be stationary. We begin by analyzing the non-stationary case.

We study Algorithm~\ref{alg:lupts_lin} in discrete-time Gaussian-linear dynamical systems, where the time series $X_1, ..., X_T$ and the outcome $Y$ evolve according to noisy linear Markov dynamics, as follows.

\begin{thmasmp}[Gaussian-linear system]
The privileged information and outcome evolve as
\begin{equation}
    \begin{aligned}
    X_t =& A_{t-1}^\top X_{t-1} + \epsilon_t \;\; \mbox{ for } t=2, ..., T \\
    Y =& \beta^\top X_T + \epsilon_Y
    \end{aligned}
    \label{eq:lin_dyn}
\end{equation}
where $A_t$ are a set of transition matrices that determine the behavior (and stability) of the system. Noise terms are assumed to be zero-mean Normal random variables, $\epsilon_t \sim \mathcal{N}(0, \Sigma)$, for t=2, ..., T, and $\epsilon_Y \sim \mathcal{N}(0,\sigma_Y^2)$. We make no assumption on the distribution of $X_1$.
\label{asmp:gauss}
\end{thmasmp}

Due to the linearity of the transitions and outcome, $Y$ is a linear function of the variable $X_t$ for any $t$. Most importantly, it is easy to show that $Y$ is also a Gaussian-linear function of $X_1$,
$$
Y = \theta^\top X_1 + \tilde{\epsilon} \;\; \mbox{ with } \;\; \theta = A_1 \cdots A_{T-1}\beta
$$
and
$$
\tilde{\epsilon} = \beta^\top \bigg(\sum_{t=2}^{T-1}\bigg[\prod_{t'=t}^{T-1}A_{t'}\bigg]^\top \epsilon_t + \epsilon_T\bigg) + \epsilon_Y~.
$$

Our goal is now to learn estimates of $\theta$ as efficiently as possible, i.e., with the smallest error and/or risk for a given number of samples. It is well-known that the OLS estimator $\htheta_\OLS := (\bfX_1^\top \bfX_1)^{-1} \bfX_1^\top \bfY$ is the minimum-variance mean-unbiased estimator that is based only on samples of $(X_1, Y)$, see e.g., \citet[Chapter 7]{johnson2002applied}; This makes $\htheta_\OLS$ the strongest possible baseline among estimators that do not make use of privileged information.

\subsection{Variance Reduction Through Rao-Blackwellization}
Next, we show that, in the Gaussian-linear setting of  Assumption~\ref{asmp:gauss}, with the additional assumption of isotropic noise in transitions, the output of the LuPTS algorithm is a Rao-Blackwell estimator~\citep{radhakrishna1945information,blackwell1947conditional} of $\theta$, which has improved statistical properties over $\htheta_\OLS$. However, to show this, we first prove the following lemma.

\begin{thmlem}
\label{lem:ols_general}
Let $\hat{K}=(\hA_1,\dots,\hA_{T-1}, \hbeta)$ be the parameters learned by  Algorithm~\ref{alg:lupts_lin} without stationarity, and let $(\bfX_1, \dots, \bfX_T, \bfY)$ be a random dataset from the Gaussian-linear system defined in Assumption~\ref{asmp:gauss}, with isotropic noise, $\forall t : \epsilon_t \sim \mathcal{N}(0, \sigma_t^2I)$. 
Then, for any $t=2,\dots,T$ we have that
\begin{equation*}
    \E[\bfX_t|\bfX_{t-1}, \hat{K}] = \bfX_{t-1}\hA_{t-1} \;\;\;\; \E[\bfY|\bfX_{T}, \hat{K}] = \bfX_{T}\hbeta~.
\end{equation*}
\end{thmlem}
A proof is given in the Appendix. The isotropic noise assumption simplifies the analysis, although it is feasible to prove this lemma in the anisotropic case as well. We give a brief remark in the Appendix highlighting how the analysis differs.

\begin{thmrem}
Lemma~\ref{lem:ols_general} says that the expected state at $t$, across datasets of the same size for which Algorithm~\ref{alg:lupts_lin} returns $\hA\hbeta$, is equal to the estimated state at $t$ given the previous state at $t-1$. This is a result of the fact that the same OLS estimator would be obtained if we had samples that were mirrored along the estimated plane of best fit, and that such a dataset is equally likely to occur. Lemma~\ref{lem:ols_general} can be used to prove a second lemma, which is found in the Appendix, stating that the output of the algorithm is indeed a Rao-Blackwell estimator, i.e. $\E[\htheta|\hA, \hbeta]=\hA\hbeta$.
With these two lemmas, our main results in Theorem~\ref{thm:rao} follow using mostly standard arguments~\citep{radhakrishna1945information,blackwell1947conditional}. 
\end{thmrem}

We evaluate estimates $\htheta$ using the mean squared error (MSE) w.r.t. $\theta$, and the prediction risk $\olR(\htheta)$, as defined below, where expectations are taken over the randomness in $\htheta$, determined by the dataset $D$,
\begin{align*}
\mse(\htheta) &\coloneqq \E_D [\|\htheta - \theta\|^2_2]~, \\
\olR(\htheta) &\coloneqq \E_D [\E_{X_1, Y}[(\htheta^\top X_1 - Y)^2]]~.
\end{align*}
We can now state the following result, relating $\htheta_{\mathrm{OLS}}$ and $\htheta_\LuPTS$, the output of Algorithm~\ref{alg:lupts_lin}.

\begin{thmthm}\label{thm:rao}%
Let $D = (\bfX_1, \bfX_2, ..., \bfX_T, \bfY)$ be a random dataset with $\htheta_\OLS \coloneqq (\bfX_1^\top \bfX_1)^{-1} \bfX_1^\top \bfY$, and let  $\htheta_\LuPTS = \hat{A}\hat{\beta}$ be the output of Algorithm~\ref{alg:lupts_lin} without stationarity. 
Under the Gaussian-linear system of Assumption~\ref{asmp:gauss} with isotropic noise as in Lemma~\ref{lem:ols_general}, $\htheta_\LuPTS$ is unbiased, and 
\begin{equation*}
    \begin{aligned}
    \mse{}(\htheta_\LuPTS) = & \mse{}(\htheta_{\mathrm{OLS}})  - \underset{D}{\E}[\V(\htheta_\OLS \mid \htheta_\LuPTS)]
    \end{aligned}
\end{equation*}
where $\V$ is the sum of element-wise conditional variances and the expectation is taken over datasets $D$, since both estimators are functions of them. 
Further, 
\begin{equation*}
\begin{aligned}
    \olR(\htheta_\LuPTS) = & \olR(\htheta_\OLS)  - \underset{D, X_1}{\E}[\underset{\;\htheta_\OLS}{\V}({\htheta_\OLS} X_1 \mid \htheta_\LuPTS)]~.
\end{aligned}
\end{equation*}%
Since the variances are non-negative, $\htheta_\LuPTS$ is at least as good as $\htheta_\OLS$ in both metrics.
\end{thmthm}

Theorem~\ref{thm:rao} states that, in the Gaussian-linear case, the LuPTS estimator is never worse on average across same-size datasets than the best unbiased estimator learning only from $(X_1, Y)$, irrespective of the distribution of $X_1$. In other words, \emph{privileged information is provably useful in this case}. If there is significant uncertainty about $\htheta_\OLS$ after $\htheta_\LuPTS$ is determined, $\V(\htheta_\OLS \mid \htheta_\LuPTS)$ is high and LuPTS is favored more strongly. 

As a byproduct of the proof of Theorem~\ref{thm:rao},  we further have for the gap in MSE, $G$, 
$$
G \coloneqq \E_D[\V(\htheta_\OLS \mid \htheta_\LuPTS)] = \E_D[\|\htheta_\OLS - \htheta_\LuPTS \|^2].
$$ 

We can express this gap more explicitly when $T=2$.

\begin{thmcol}
Under Assumption~\ref{asmp:gauss}, for $T=2$, with $H_1 = (\bfX_1^\top \bfX_1)^{-1}\bfX_1^\top $, $H_2 = (\bfX_2^\top \bfX_2)^{-1}\bfX_2^\top$, both functions of the dataset $D$, it holds for the MSE gap $G$,
\begin{align*}
G = \E_D[\|(A H_2 - H_1  + H_1\epsilon_2H_2 ) \epsilon_Y\|^2]~.
\end{align*}
\end{thmcol}
Whenever $\epsilon_Y = 0$, $G=0$ irrespective of other factors. If $\epsilon_2 = 0$ and $A$ is invertible, $G=0$ as well ($\epsilon_2=0 \implies$ $X_2 = X_1A$ and $AH_2 = H_1$). In other words, in the case where either the dynamics or the outcome are noiseless, the LuPTS estimator reduces to the OLS estimator. More importantly, if neither noise term is 0, the difference will not be 0 in general. As a consequence, $\V(\htheta_\OLS \mid \htheta_\LuPTS) > 0$, and LuPTS is strictly preferable over OLS on average. 
In Section~\ref{sec:experiments}, we confirm empirically that LuPTS is more efficient and examine the gap as a function of problem parameters.

\paragraph{Bias \& Variance.}
When  the models of both system dynamics and the outcome are well-specified, the gains from the LuPTS estimator come from variance reduction, since both the OLS and LuPTS are unbiased and the irreducible risk due to noise is shared between them. However, in misspecified settings, $\theta_\LuPTS$ may be biased even when $\theta_\OLS$ is not. 
For example, let $Y = \sqrt{X_2} + \epsilon$ and $X_2 = (X_1)^2$ for $X_1$ with support on the positive real line.  The Markov condition holds, OLS is unbiased, but LuPTS is not. In the misspecified case, the benefits of LuPTS come down to a tradeoff between bias and variance. This is explored in Section~\ref{sec:cities}.
    
\paragraph{Learning From Stationary Systems.}
When the stationarity flag is false in Algorithm~\ref{alg:lupts_lin}, the estimator treats transitions at different time points $t, t'$ as independent mechanisms. Then, while the privileged information provides additional samples with increasing $T$, the number of functions to estimate increases with $T$ as well.
When we apply Algorithm~\ref{alg:lupts_lin} with the stationarity flag set to true, we exploit the assumption that we observe $m \times (T-1)$ (dependent) samples of the same linear transformation. This dependency is the primary reason for why Theorem~\ref{thm:rao} does not readily extend to the stationary case. However, we observe improvements over baseline and the non-stationary LuPTS model for real-world experiments in Section~\ref{sec:experiments}.

\subsection{Relation To Distillation Approaches}
\label{sec:distill}

\textit{Generalized distillation}~\citep{LopSchBotVap16} is a  technique for learning using privileged information, utilized by \citet{hayashi2019long} in the context of privileged time series. Distillation methods train a student model to minimize its prediction error on both true labels and soft targets provided by a teacher model trained on the privileged data, in the hope to increase sample efficiency by transferring knowledge from teacher to student. However, to the best of our knowledge there are no results proving gains from distillation of privileged information which apply in our setting. 

In the linear setting with squared loss, the distillation loss function is defined as
\begin{equation}
\label{eq:distill_loss}
 \min_\theta \;\; \lambda ||\bfY-\bfX_1\theta||_2^2 + (1-\lambda)||\hat{\bfY}_{soft}-\bfX_1\theta||_2^2
\end{equation}
where $\lambda\in [0,1]$ and $\hat{\bfY}_{soft}$ comprises predictions made by a teacher model. We will now consider the special case of distillation where the LuPTS estimator is used as teacher model, i.e., $\hat{\bfY}_{soft}=\bfX_1\htheta_\LuPTS$. In this case, we can present the following theorem.

\begin{thmthm}\label{thm:convex_distillation}
Let $\htheta_\LuPTS$ be the output of Algorithm~\ref{alg:lupts_lin} and $\htheta_\OLS=(\bfX_1^\top \bfX_1)^{-1}\bfX_1^\top \bfY$. For $\htheta_{Dist}$, the solution to \eqref{eq:distill_loss} with $\hat{\bfY}_{soft}=\bfX_1\htheta_\LuPTS$ and $\lambda\in[0,1]$, it holds that
\begin{equation}
     \htheta_{Dist} = \lambda \htheta_\OLS + (1-\lambda)\htheta_\LuPTS~.
     \label{eq:convex_combination}
\end{equation}
Additionally, under Assumption~\ref{asmp:gauss}, it holds that
\begin{equation*}
 \mse (\htheta_\LuPTS) \leq \mse(\htheta_{Dist}) \leq \mse (\htheta_\OLS)~.
\end{equation*}
\end{thmthm}
A proof can be found in the Appendix. 
Theorem~\ref{thm:convex_distillation} states that using distillation with a linear student model and LuPTS as teacher leads to an estimate $\htheta_{Dist}$ that is a convex combination of $\htheta_\OLS$ and $\htheta_\LuPTS$. 
In the well-specified case (under Assumption~\ref{asmp:gauss}),  since both $\htheta_{\OLS}$ and $\htheta_{\LuPTS}$ are unbiased, $\htheta_{Dist}$ is unbiased as well, and the MSE of $\htheta_{Dist}$ is bounded between the MSE of $\htheta_{\OLS}$ and $\htheta_{\LuPTS}$. Interestingly, in the misspecified case, eq.~\eqref{eq:convex_combination} shows that using distillation leads to a principled way of trading off bias and variance. For an optimal choice of $\lambda$,  $\htheta_{Dist}$ is never worse than using either $\htheta_\OLS$ or $\htheta_\LuPTS$, and may improve on both.

In Section~\ref{sec:cities}, we implement two variants of the distillation approach: Distill-Seq, where $\hat{\bfY}_{soft}$ comprises predictions made by LuPTS, and  Distill-Concat, where $\hat{\bfY}_{soft}$ are the predictions made by a traditional linear model trained on a concatenation of the privileged time points, akin to the teacher in \citet{hayashi2019long} (Theorem~\ref{thm:convex_distillation} does not hold for the latter).

\section{EXPERIMENTS}
\label{sec:experiments}
We evaluate properties of the LuPTS estimator in a series of experiments\footnote{Code available at \\ \href{https://github.com/RickardKarl/LearningUsingPrivilegedTimeSeries}{github.com/RickardKarl/LearningUsingPrivilegedTimeSeries}.}. First, in a synthetic setting, we verify our theoretical findings under the assumptions stated in Section~\ref{sec:learning}. An example of what happens when the Markov assumption is violated is also shown. Second, on three real-world datasets, the PM$_{2.5}$ pollution dataset (Section~\ref{sec:cities}) and two clinical datasets (Section~\ref{sec:clinical}), we study the gain in predictive performance with the LuPTS estimator compared to the baseline OLS estimator.
Our results on real-world data demonstrate the bias-variance tradeoff implied by our approach as well as its utility in improving predictive performance for both regression and binary classification tasks.
Third, we compare LuPTS to the distillation approaches described in Section~\ref{sec:distill}.
Finally, we compare the stationary and non-stationary versions of LuPTS, demonstrating that the preferred version depends on the domain and prediction task.

\subsection{Experimental Setup}

The LuPTS algorithm computes several OLS estimates (see Algorithm~\ref{alg:lupts_lin}).
All OLS estimates, including the baseline model used for comparison, use the (unregularized) implementation LinearRegression of the Python module scikit-learn \citep{scikit-learn}. 
Although it would be of interest to also study regularized variants of these models, we leave it as future work for a more thorough investigation, both theoretically and experimentally.  
When extending the algorithm to binary classification tasks, the baseline model and the outcome model in the LuPTS algorithm are implemented using the LogisticRegressionCV class from scikit-learn. We perform 5-fold cross-validation on the training portion to tune the $L_2$ regularization parameter, which we vary from \num{1e-4} to \num{1e4}. Models are evaluated using the coefficient of determination ($R^{2}$) for regression tasks and the Area Under the ROC Curve (AUC) for classification tasks. In all plots, Baseline refers to OLS or Logistic Regression (depending on the task), LuPTS refers to the output of Algorithm~\ref{alg:lupts_lin} \emph{without} stationarity, and Stat-LuPTS to the output \emph{with} stationarity.

\subsection{Synthetic Experiments}
\label{sec:synthetic}

\begin{figure*}[t!]
    \centering
    \begin{subfigure}[t]{0.225\textwidth}
        \centering
        \includegraphics[height=.125\textheight]{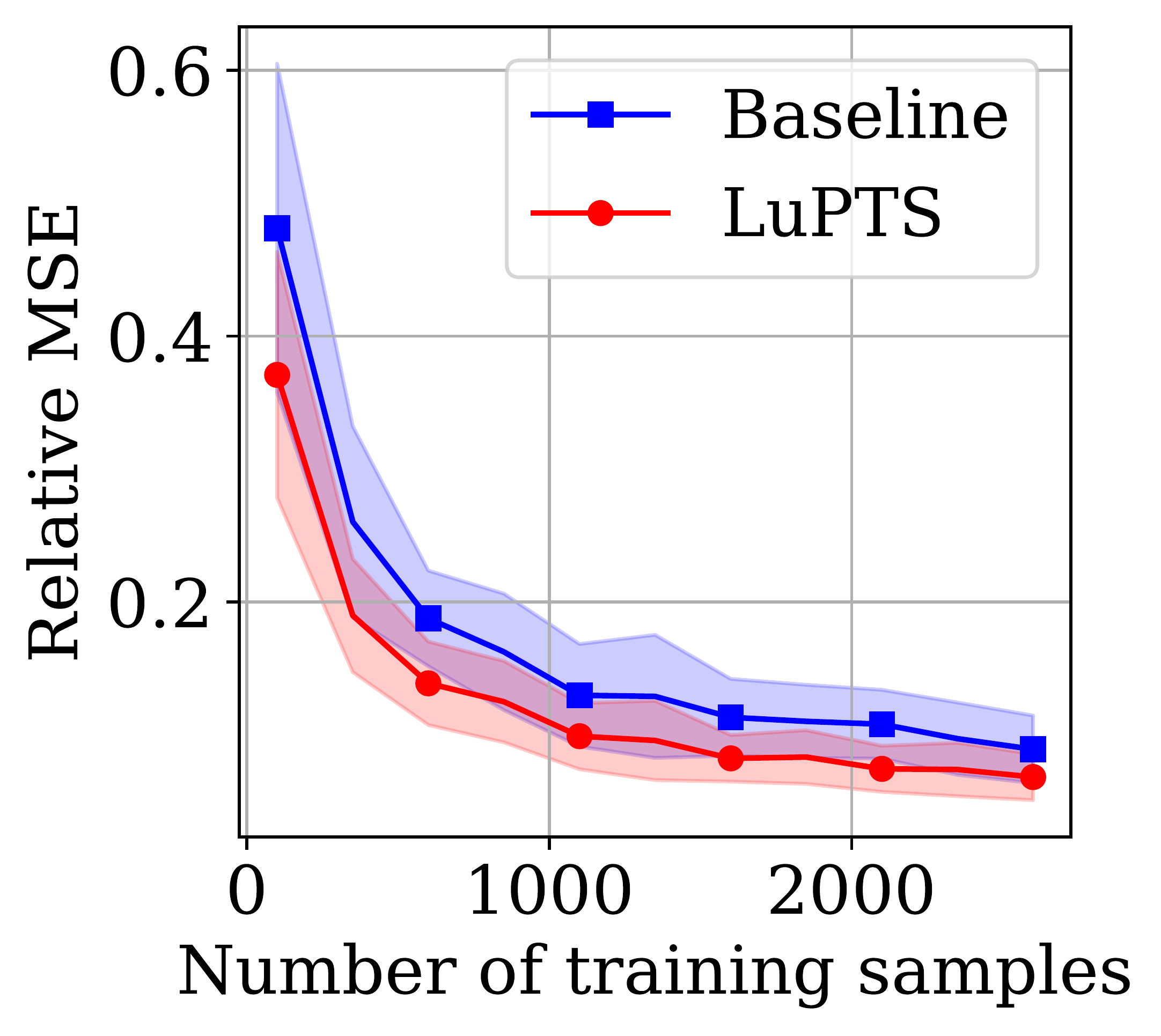}
        \caption{Parameter recovery}
        \label{fig:vary_samples}
    \end{subfigure}%
    ~ 
    \begin{subfigure}[t]{0.239\textwidth}
        \centering
        \includegraphics[height=.125\textheight]{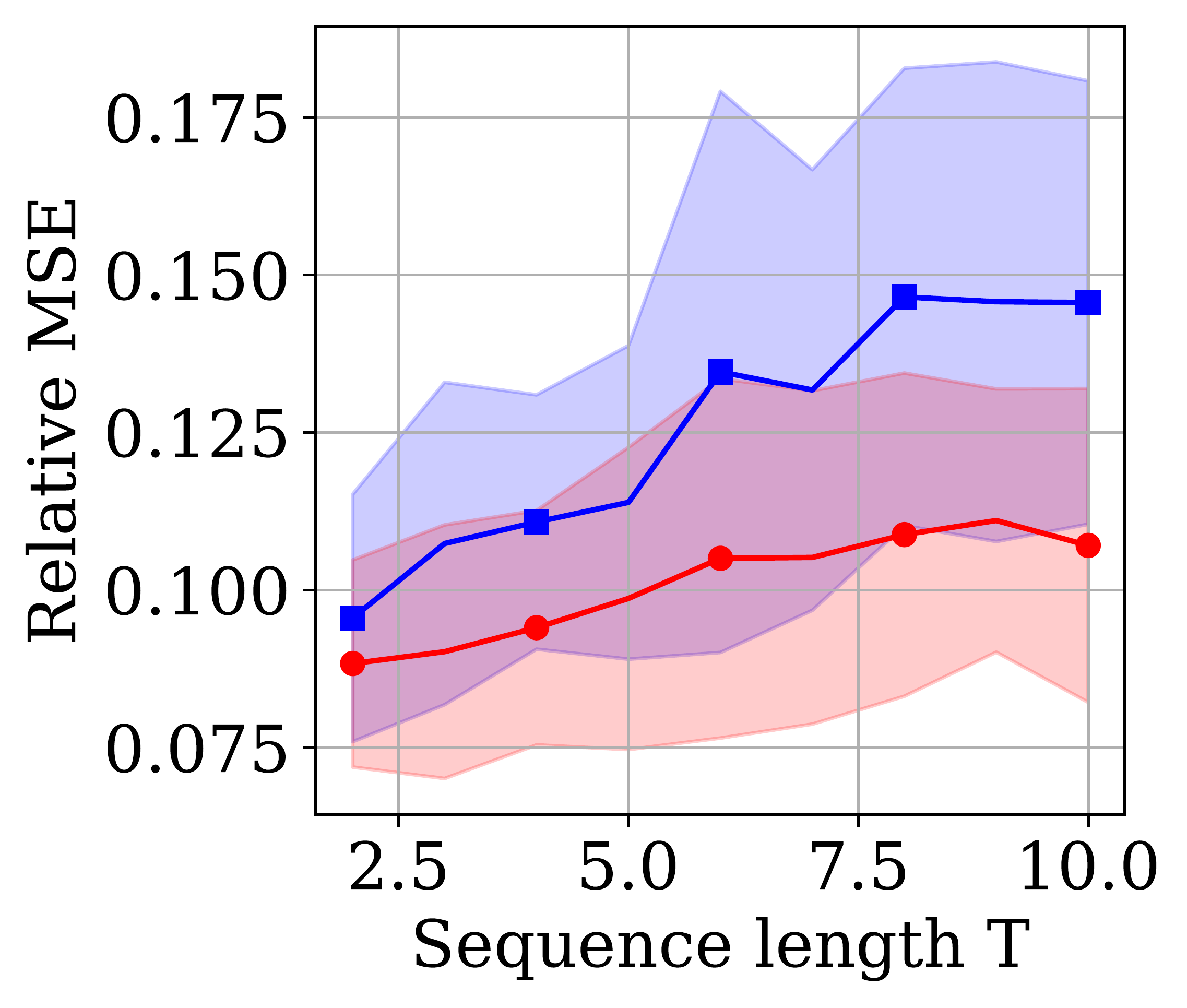}
        \caption{Parameter recovery}
        \label{fig:vary_dim}
    \end{subfigure}
    ~ 
    \begin{subfigure}[t]{0.225\textwidth}
        \centering
        \includegraphics[height=.125\textheight]{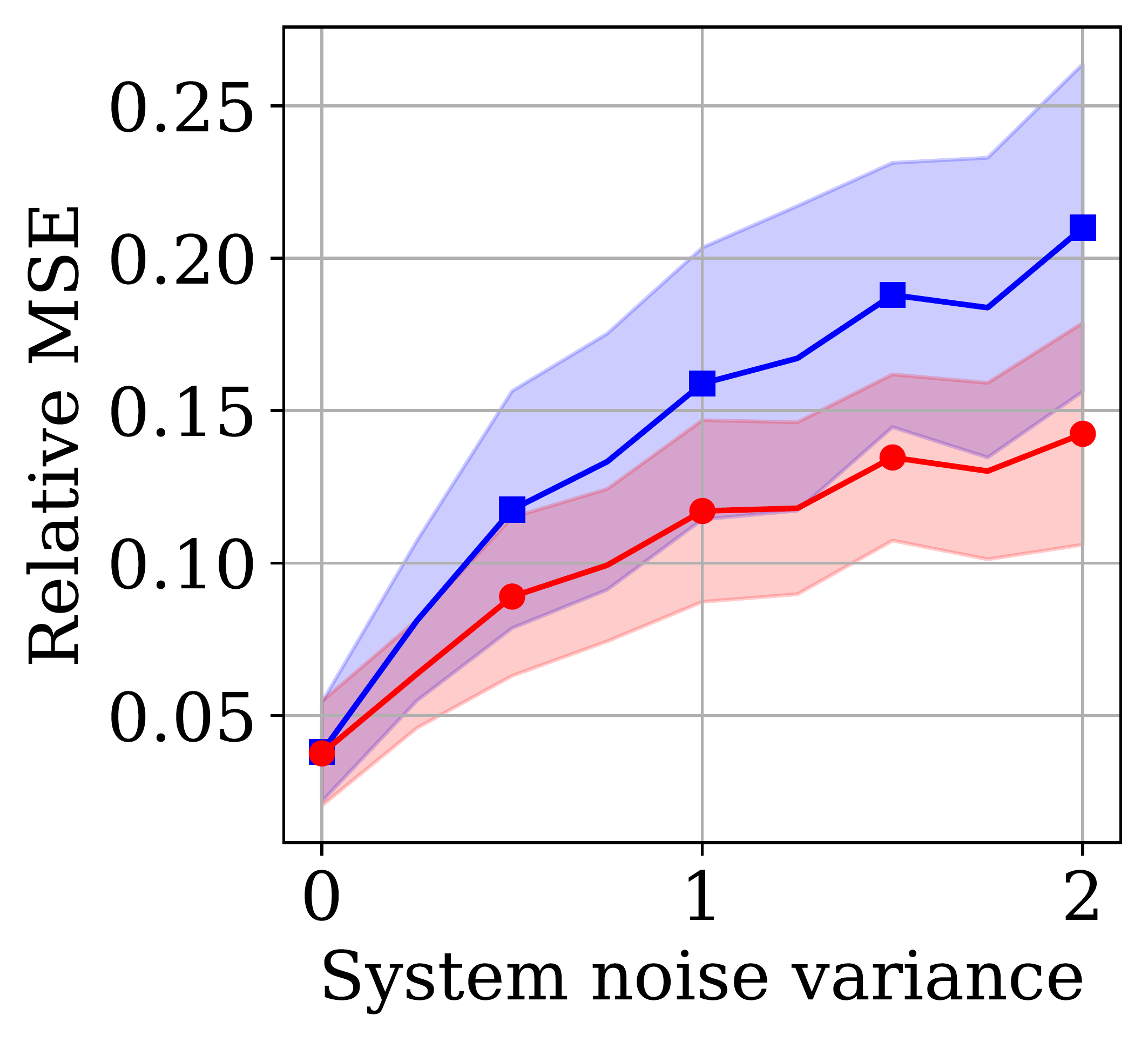}
        \caption{Parameter recovery}
        \label{fig:vary_noise}
    \end{subfigure}
    ~
    \begin{subfigure}[t]{0.23\textwidth}
    \centering
    \includegraphics[height=.125\textheight]{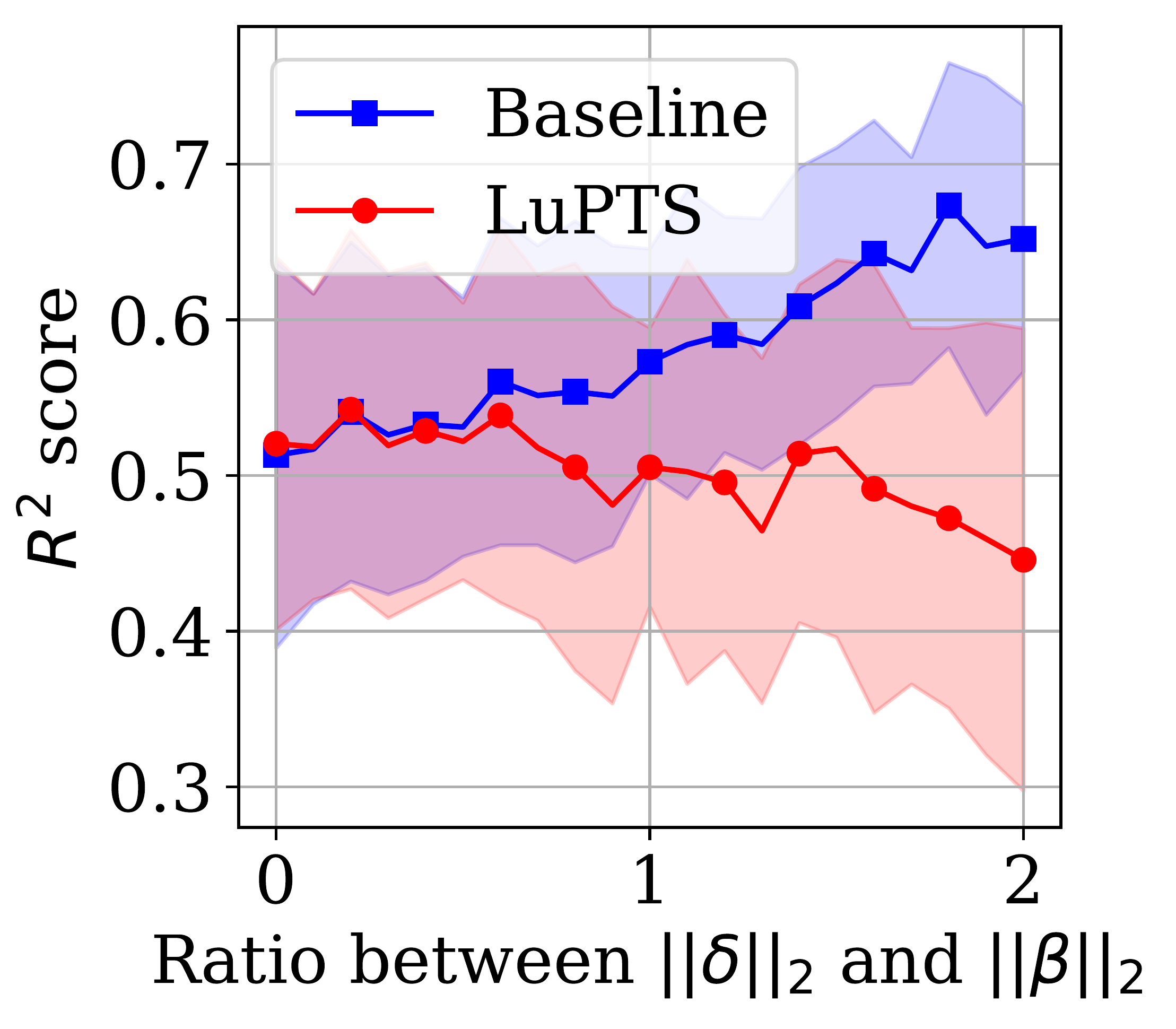}
    \caption{Prediction error}
    \label{fig:break_markov}
    \end{subfigure}
    \caption{ \small Synthetic system showing relative MSE of the parameter estimate or the prediction error and one standard deviation. \textbf{Parameter recovery} (Left three; lower is better): Varying either only number of samples $n$, sequence length $T$, or noise standard deviation $\V(\epsilon_t)$. \textbf{Breaking the Markov Assumption} (Right-most figure; higher is better): Adding a direct linear relationship between $X_1$ and $Y$ with coefficient $\delta$. Larger value on x-axis leads to more deviation from Markov assumption.}
    \label{fig:results_synthetic}
\end{figure*}

To verify and further investigate our theoretical results, we sample from a synthetic dynamical system where Markovianity and linearity with additive isotropic Gaussian noise hold. The elements of $A_t\in \mathbb{R}^{d\times d}$ for $t=1,\dots,T-1$ and $\beta\in \mathbb{R}^{d\times 1}$ are drawn from a Normal distribution with the exception of the diagonal elements in the transition matrices, which are set to 1. The eigenvalues $(\lambda_1,\dots, \lambda_d)$ of $A_t$ influence the system's behavior and stability. Unstable linear systems, i.e., those with large eigenvalues, are easier to estimate ~\citep{simchowitz2018withoutmixing}, and therefore we enforce the spectral radius $\rho(A_t)$ to equal $\kappa$ for all $t$, with $\kappa > 1$. We refer the reader to the Appendix for a more in-depth description of the system generation. For all experiments, we use the following default values unless otherwise stated: $\kappa=1.5$, $n=1000$, $T=10$, $d=25$, and $\V(\epsilon_t)=\V(\epsilon_Y)=1$ for $t=1,\dots,T-1$. Finally, the input distribution is $p(X_1)=\cN(\mu=0, \sigma^2=5)$.

\paragraph{Parameter Recovery} Figures~\ref{fig:vary_samples}, \ref{fig:vary_dim} and \ref{fig:vary_noise} present the relative MSE, $||\theta - \htheta||_2^2 / ||\theta||_2^2$, of the Baseline (OLS) and LuPTS estimates of the synthetic system described above. We investigate the impact of the number of training samples $n$, sequence length $T$, and variance of system noise on the MSE by varying one variable and keeping the other two fixed.
Compared to the baseline estimates, the LuPTS estimates are closer in general to the true parameter, as predicted by Theorem~\ref{thm:rao}. Both methods improve as we increase the number of training samples, but LuPTS is consistently better or equally good. The difference between them increases for larger $T$, which can be explained by the fact that there is more uncertainty between $X_1$ and $Y$ as $T$ gets larger. 
Notably, when the system noise is removed completely, 
LuPTS and OLS coincide, as expected.

\paragraph{Breaking The Markov Assumption} In Figure~\ref{fig:break_markov}, we introduce a coefficient $\delta$, generated in the same way as $\beta$, which controls a direct dependence of $Y$ on $X_1$, i.e., $Y=X_T\beta + X_1\delta$. We scale $\delta$ to vary the ratio of the norms $\frac{||\delta||_2}{||\beta||_2}$ (x-axis), and observe that predictions from LuPTS get worse in terms of $R^2$ score (y-axis) as the ratio increases. In spite of the bias, we see that LuPTS still performs equally well as the baseline for small non-zero ratios. This result can be explained by the fact that the privileged information contains useful knowledge, which offsets the bias when it is small.

\subsection{Forecasting Air Quality}
\label{sec:cities}

\begin{figure*}[ht]
    \centering
    \begin{subfigure}[t]{0.22\textwidth}
        \centering
        \includegraphics[width=\textwidth]{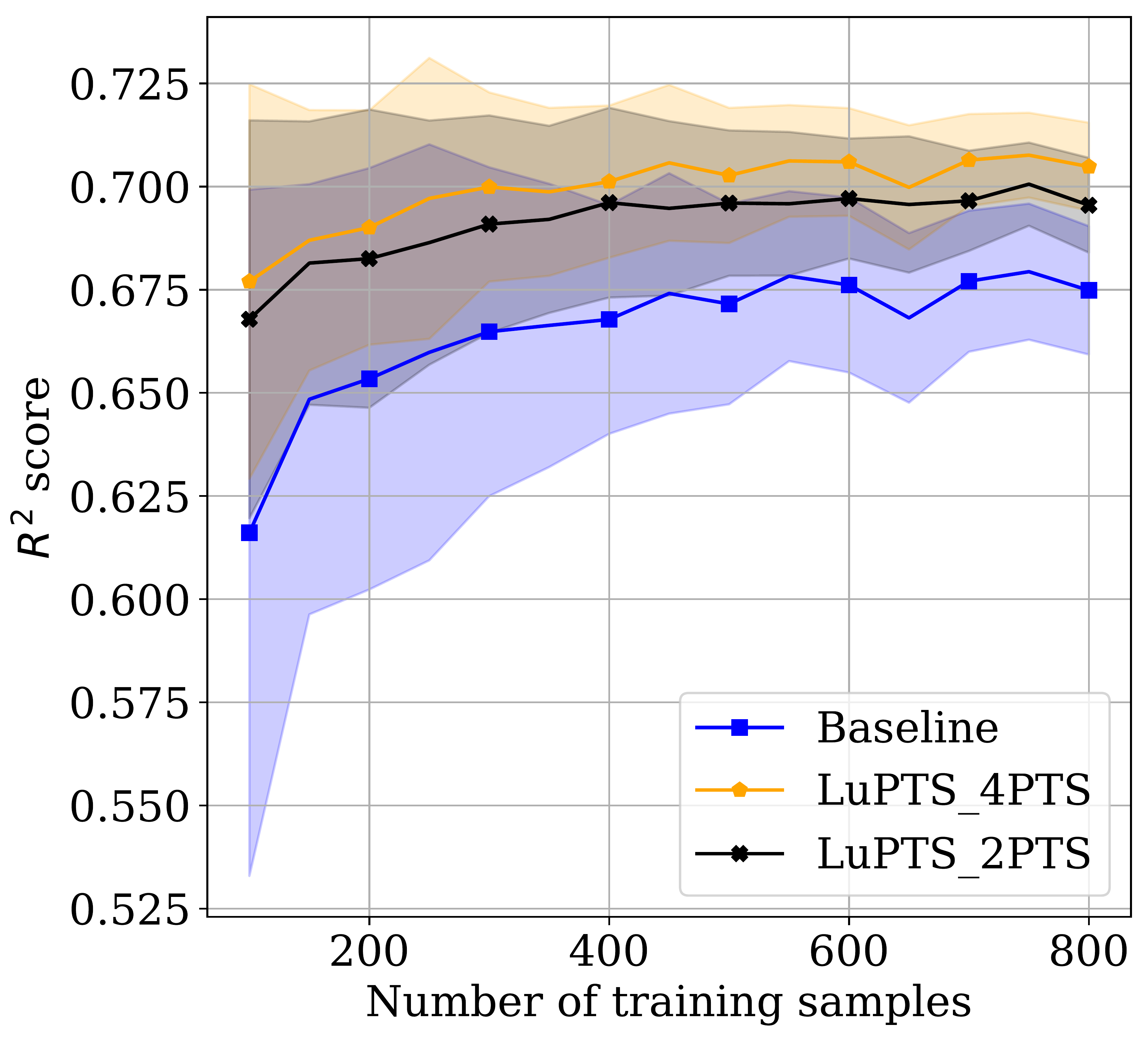}
        \caption{6 hour forecast \\ Shenyang}
        \label{fig:shenyangT6}
    \end{subfigure}%
    ~
    \begin{subfigure}[t]{0.22\textwidth}
        \centering
        \includegraphics[width=\textwidth]{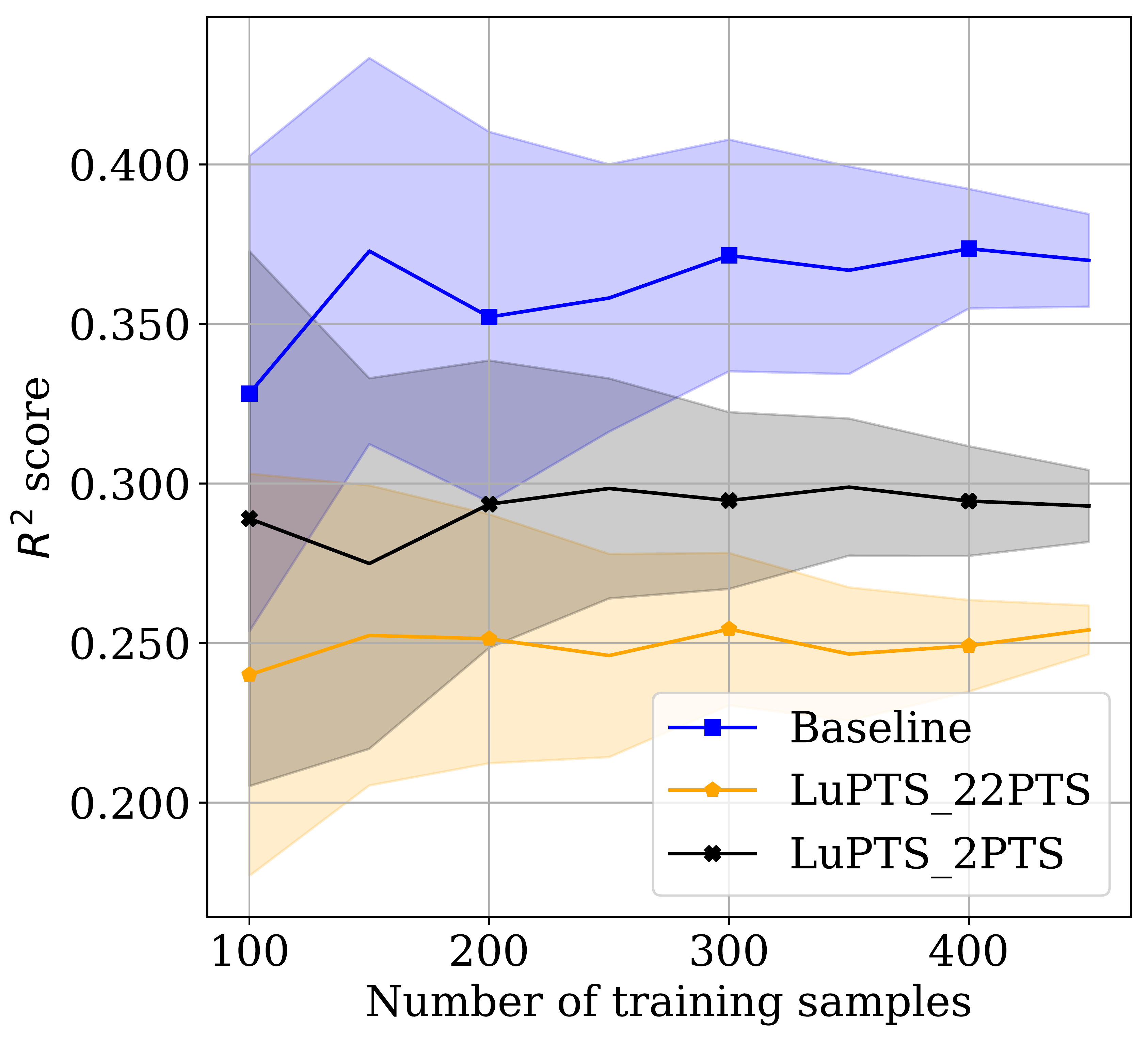}
        \caption{24 hour forecast \\ Shenyang}
        \label{fig:shenyangT24}
    \end{subfigure}%
    ~
    \begin{subfigure}[t]{0.22\textwidth}
        \centering
        \includegraphics[width=\textwidth]{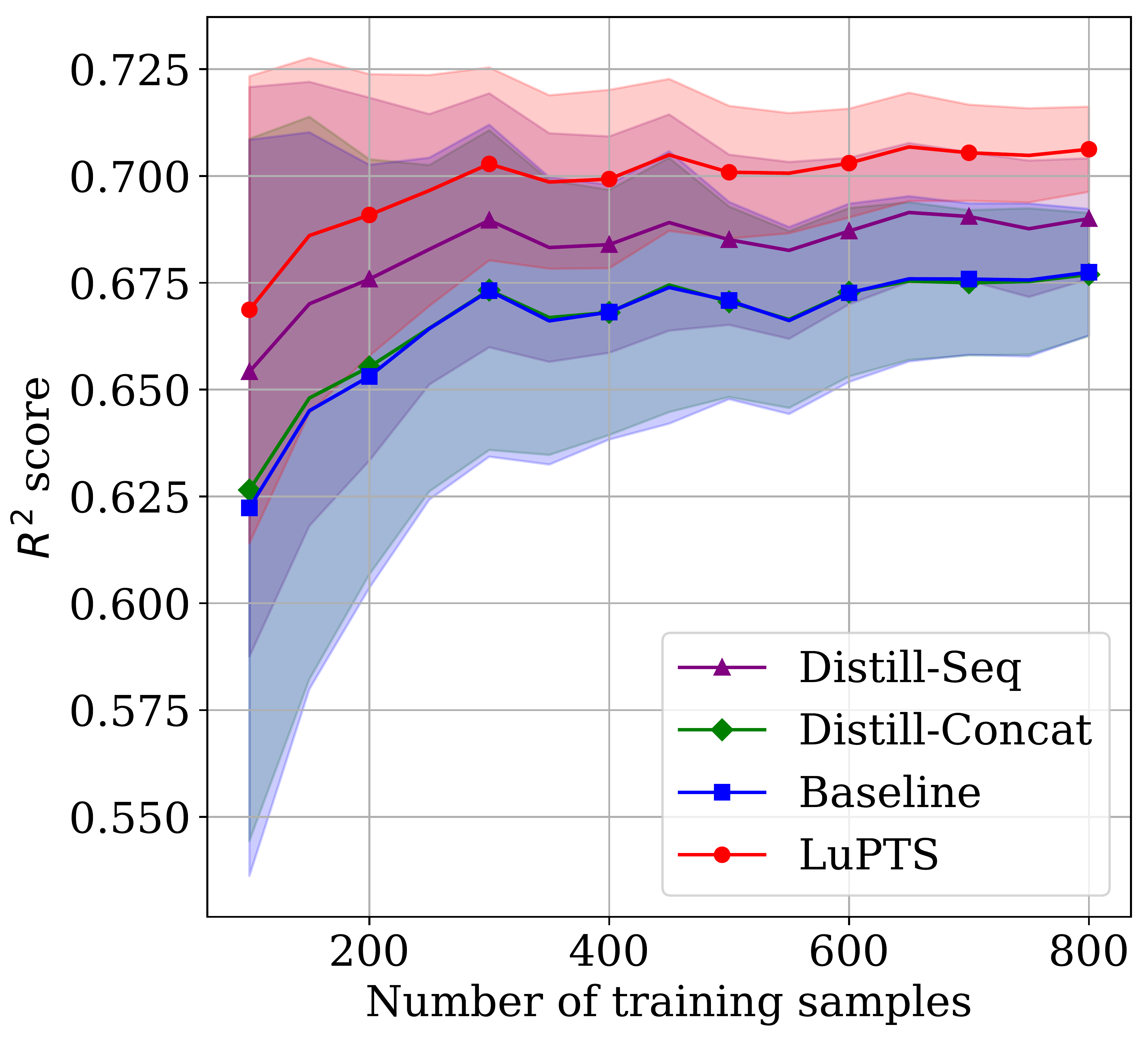}
        \caption{6 hour forecast \\ Shenyang (Distillation)}
        \label{fig:shenyangdistT6}
    \end{subfigure}%
    ~
    \begin{subfigure}[t]{0.22\textwidth}
        \centering
        \includegraphics[width=\textwidth]{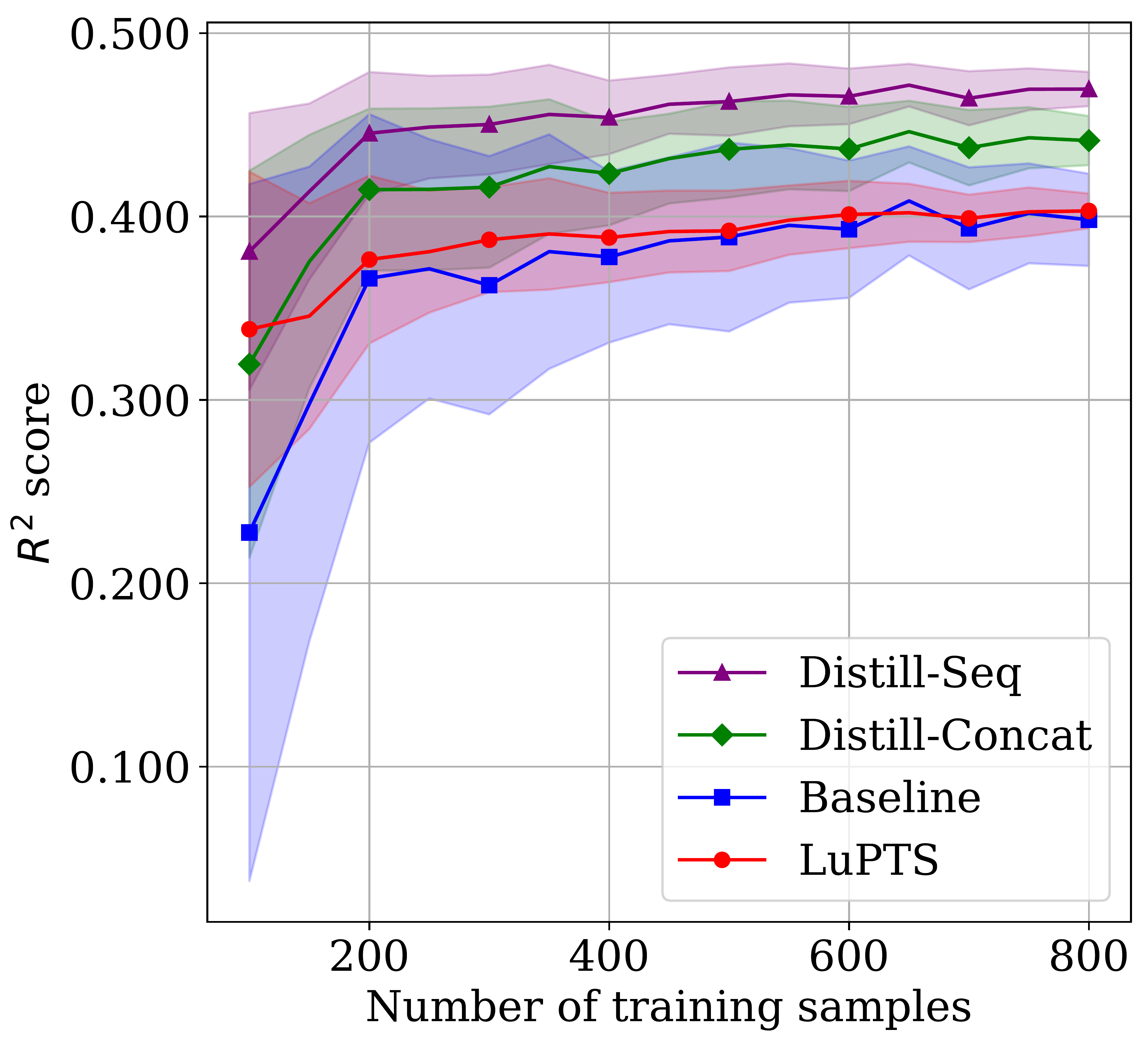}
        \caption{12 hour forecast \\ Chengdu (Distillation)}
        \label{fig:chengdudistT12}
    \end{subfigure}
    \caption{\ref{fig:shenyangT6}, \ref{fig:shenyangT24}) Changing the amount of privileged information for the LUPTS for different time horizons, where the \textit{X} in LuPTS\_\textit{X}PTS indicates the number of privileged time points. \ref{fig:shenyangdistT6}, \ref{fig:chengdudistT12}) Comparing LuPTS to the distillation-based approaches, which use the same privileged information.
    Metric used is $R^2$ (Higher is better); shaded region indicates one standard deviation across 75 iterations.\label{fig:results_pm25_shenyang}}
    \label{fig:pm25_shenyang}
\end{figure*}

Due to the serious health risks caused by chronic air pollution in China, predicting how air quality changes over time is vital ~\citep{frank2012pollution}.
The PM$_{2.5}$ dataset contains hourly meteorological recordings between the years 2012 and 2015 of the fine particle (PM$_{2.5}$) concentration in Beijing, Chengdu, Guangzhou, Shanghai and Shenyang ~\citep{fivecities2016liang}, as well as seven weather features, including temperature, humidity and combined wind direction. We refer the reader to the Appendix for a complete description of the data and pre-processing steps. We also report additional results for all five cities.

We forecast the PM$_{2.5}$ concentration for several horizons, 6, 12, or 24 hours into the future. Comparing results across different horizons is informative since 1) predictions further into the future are more challenging, and 2) for longer horizons, more time points can be used as privileged information. We further compare LuPTS to Distill-Concat and Distill-Seq, introduced in Section \ref{sec:distill}, as well as non-linear baselines models. For the distillation methods, we tune $\lambda$ on the validation set, varying it across $0.25$, $0.5$, and $0.75$. At time $t=1$, we observe the features $X_1$, which also contains the current PM$_{2.5}$ concentration. Based on this information, we wish to predict the PM$_{2.5}$ concentration $T+1$ hours into the future. 
The spacing between adjacent intermediate measurements $X_2, \dots, X_T$---the privileged information---is one hour.

\paragraph{Bias-variance Trade-off When Varying Sequence Length And Privileged Information}
Results for two forecast horizons with a different number of evenly spaced time points used as privileged information are shown in Figures~\ref{fig:shenyangT6} and \ref{fig:shenyangT24}. The results depict an interesting example of the bias-variance trade-off of LuPTS. 
For the 6 hour forecast (Figure~\ref{fig:shenyangT6}), we see improved performance using LuPTS for all sample sizes. In addition, adding more privileged time points is beneficial.
For the 24 hour forecast (Figure~\ref{fig:shenyangT24}), LuPTS is consistently worse than baseline. Interestingly, in the case where LuPTS performs worse already, adding more privileged time points is not beneficial. 
This result may be due to the learned dynamical system being biased, and consequently, the bias compounds when the predicted ``roll-out'' is longer. This argument also explains why using more privileged time points is subpar if the bias is large already. On the other hand, adding more privileged information reduces the variance, as seen in both Figure~\ref{fig:shenyangT6} and \ref{fig:shenyangT24}.

\paragraph{Combining LuPTS And Distillation Can Lead To Even Greater Performance}
Figure~\ref{fig:shenyangdistT6} and \ref{fig:chengdudistT12} show the results comparing the distillation-based methods, Distill-Concat and Distill-Seq, which use the same privileged information as LuPTS during training. 
When forecasting 6 hours ahead (Figure~\ref{fig:shenyangdistT6}), we see that 
LuPTS performs better than both distillation-based methods. Distill-Seq, which uses LuPTS as teacher model, also has a higher $R^2$ score than Distill-Concat, where the latter method lies close to the baseline. As previously posited, the bias is likely small in this case, and the empirical result conforms closely to Theorem \ref{thm:convex_distillation}, which states that the MSE, or equivalently $R^2$, of Distill-Seq is bounded by the MSEs of LuPTS and OLS in the well-specified case. 

For the 12 hour forecast (Figure~\ref{fig:chengdudistT12}), the distillation-based methods perform best, with Distill-Seq still outperforming Distill-Concat. As before, LuPTS is preferable to the baseline although the difference is only observed for a lower number of samples. This result is a good example of how, in the misspecified setting, Distill-Seq can do no worse than LuPTS or OLS, given that $\lambda$ is well chosen. Finally, we find that Distill-Seq does better than Distill-Concat in both cases, which can be attributed to the benefit of using LuPTS as a teacher model to derive better soft targets.

\paragraph{Comparison to non-linear baselines}

In addition to the distillation-based baselines, we compare LuPTS to non-linear baselines in the form of random forest (RF) and k-nearest neighbors regression (KNN). Results with a fixed sample size of 200 and a prediction horizon of 6 hours for all cities are shown in Table~\ref{tab:fc_nonlinear_n200_6hour_main}.
For all cities, the LuPTS or it's distillation equivalent empirically performs better than these non-linear models without access to privileged information in a setting where linearity, Markovianity or Gaussianity are likely not to hold. One possible reason for this is that non-linear models tend to overfit in low-data settings. See Appendix ~\ref{app:experiment_details} for implementation details of the non-linear models and results for a prediction horizon of 12 hours (Table~\ref{tab:fc_nonlinear_n200_12hour}).


\begin{table*}[t!]
    \centering
    \caption{Comparison of regression methods on the air quality forecasting task with a fixed sample size $n=200$ and a prediction horizon of 6 hours. Metric used is $R^2$ (Higher is better); mean value for each method with standard deviation across 200 iterations. The methods with highest $R^2$ and lowest variance are marked in bold for each city.}
    \label{tab:fc_nonlinear_n200_6hour_main}
    \begin{tabular}{llllll}
    \toprule
        Method &     Beijing &    Shanghai &    Shenyang &     Chengdu &   Guangzhou \\
    \midrule
      Baseline  & 0.64 $\pm$ 0.02 & 0.58 $\pm$ 0.06 & 0.66 $\pm$ 0.04 & 0.63 $\pm$ 0.04 & 0.45 $\pm$ 0.06 \\
         LuPTS  & 0.64 $\pm$ 0.03 & \textbf{0.62 $\pm$ 0.03} & \textbf{0.70 $\pm$ 0.03} & 0.65 $\pm$ 0.03 & 0.49 $\pm$ 0.04 \\
    Stat-LuPTS  & 0.64 $\pm$ 0.02 & 0.62 $\pm$ 0.04 & 0.69 $\pm$ 0.03 & 0.65 $\pm$ 0.03 & \textbf{0.50 $\pm$ 0.04} \\
    Distill-Seq & \textbf{0.65 $\pm$ 0.02} & \textbf{0.62 $\pm$ 0.03} & 0.68 $\pm$ 0.03 & \textbf{0.67 $\pm$ 0.02} & 0.49 $\pm$ 0.05 \\
Distill-Concat  & \textbf{0.65 $\pm$ 0.02} & 0.60 $\pm$ 0.04 & 0.66 $\pm$ 0.04 & 0.66 $\pm$ 0.03 & 0.46 $\pm$ 0.07 \\
             RF & 0.62 $\pm$ 0.03 & 0.58 $\pm$ 0.07 & 0.53 $\pm$ 0.06 & 0.61 $\pm$ 0.04 & 0.48 $\pm$ 0.05 \\
            KNN & 0.57 $\pm$ 0.04 & 0.51 $\pm$ 0.23 & 0.49 $\pm$ 0.05 & 0.51 $\pm$ 0.04 & 0.26 $\pm$ 0.09 \\
    \bottomrule
    \end{tabular}
\end{table*}

\begin{figure}[t!]
    \centering
    \captionsetup[subfigure]{justification=centering}
    \begin{subfigure}[t]{0.24\textwidth}
        \centering
        \includegraphics[width=\textwidth]{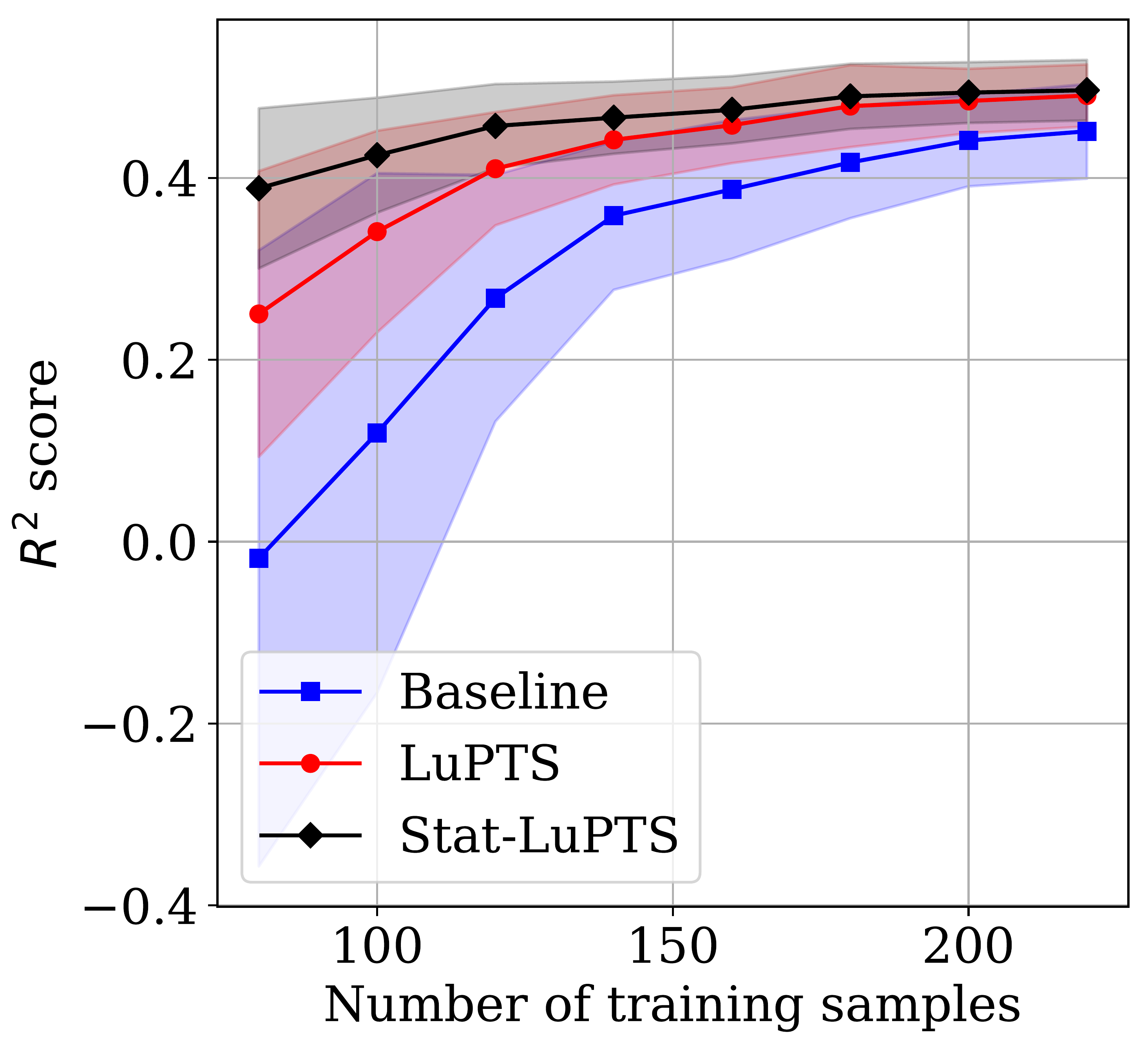}
        \caption{Predicting MMSE}
        \label{fig:adnimmse2}
    \end{subfigure}%
    \begin{subfigure}[t]{0.24\textwidth}
        \centering
        \includegraphics[width=\textwidth]{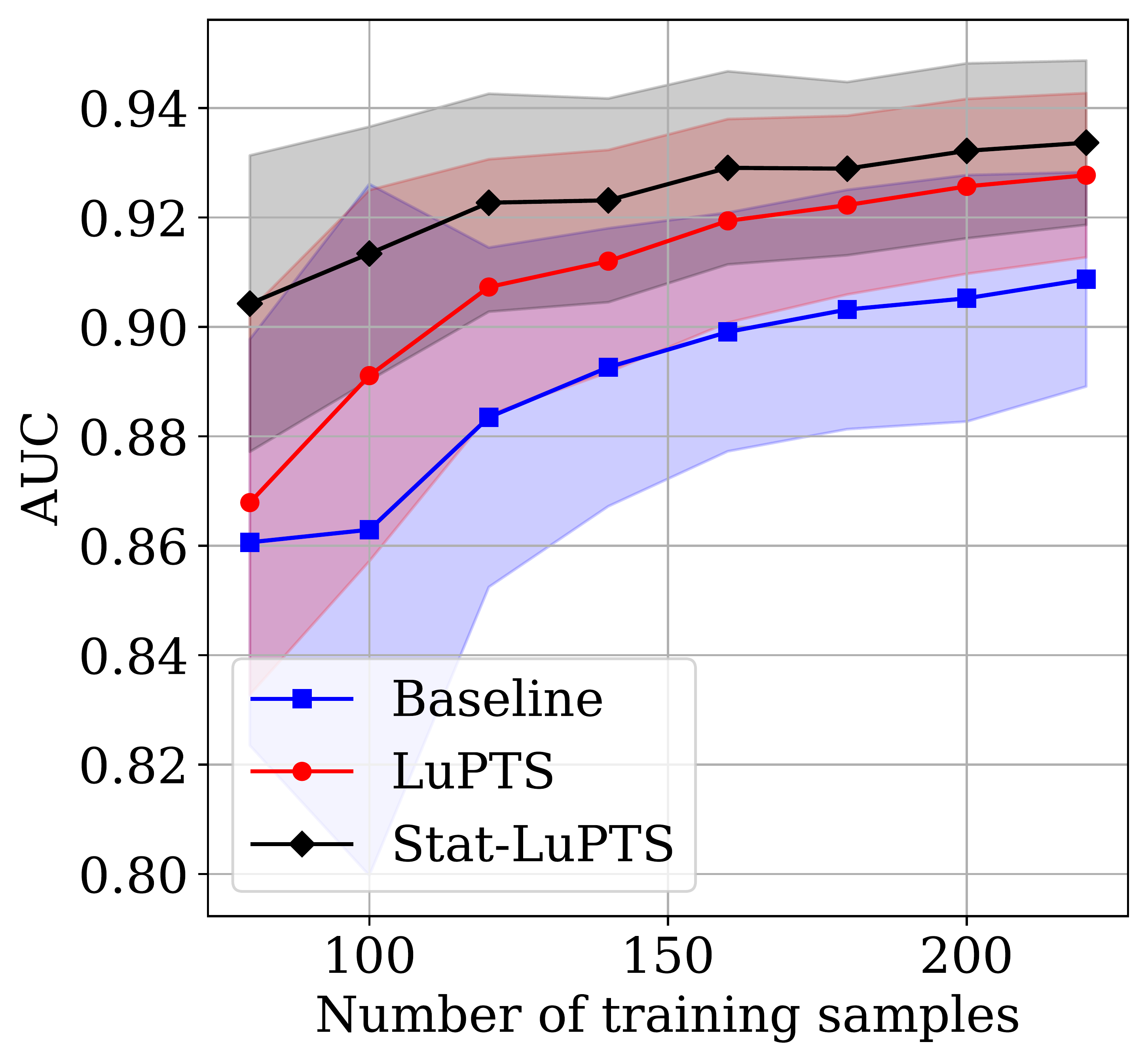}
        \caption{Predicting AD}
        \label{fig:adniad2}
    \end{subfigure}%
    \caption{\small \textbf{Alzheimer's disease progression tasks}. Follow-up at 12, 24 and 36 months after baseline as privileged information. Metrics used are $R^2$/AUC; shaded region corresponds to one standard deviation across 100 iterations.}
    \label{fig:adniexp}
\end{figure}

\subsection{Modeling Progression Of Chronic Disease}
\label{sec:clinical}

\begin{figure}[t!]
    \centering
    \captionsetup[subfigure]{justification=centering}
    \begin{subfigure}[t]{0.24\textwidth}
        \centering
        \includegraphics[width=\textwidth]{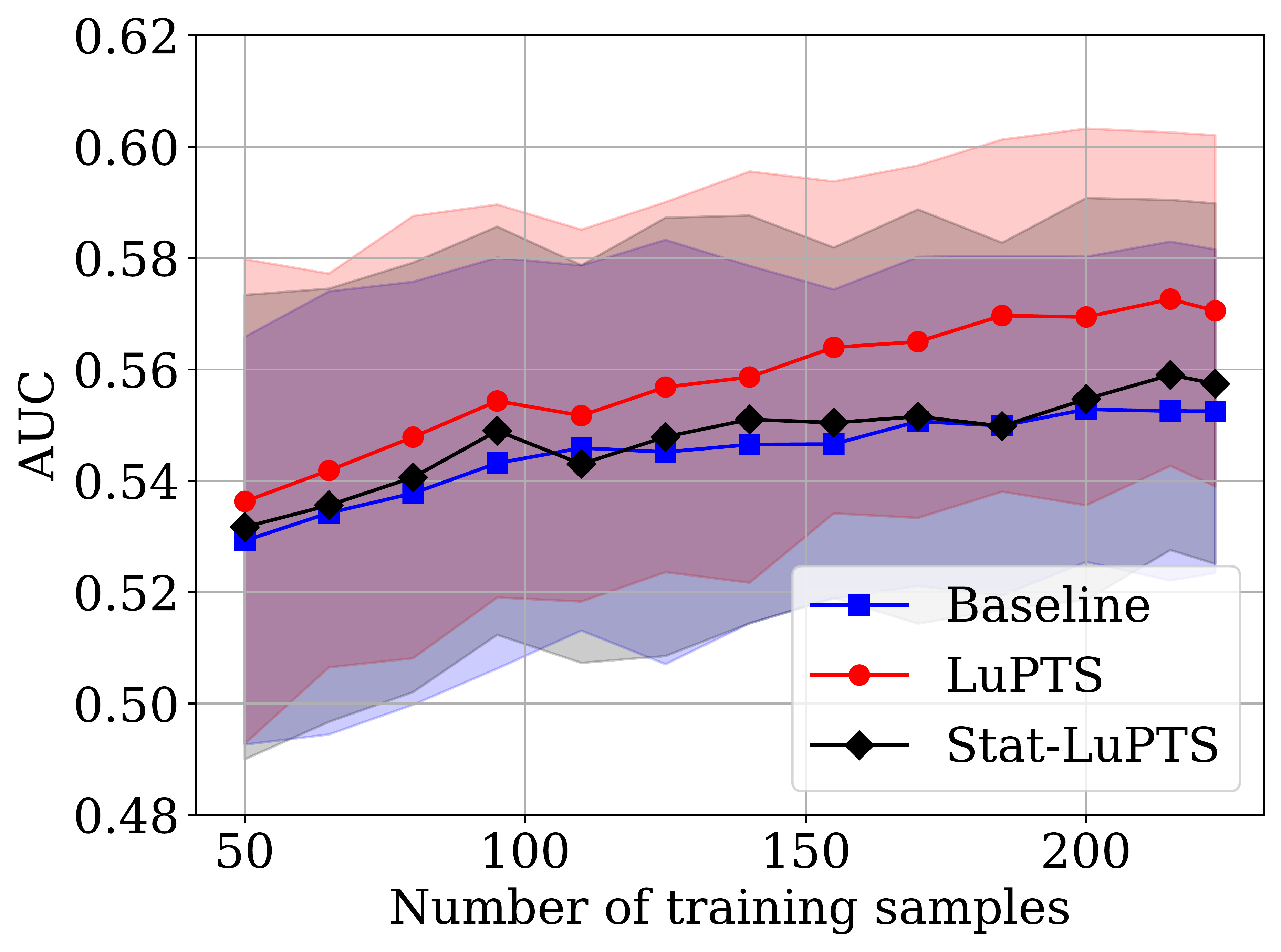}
        \caption{Early/Late Progressor}
        \label{fig:mmfigB}
    \end{subfigure}%
    \begin{subfigure}[t]{0.24\textwidth}
        \centering
        \includegraphics[width=\textwidth]{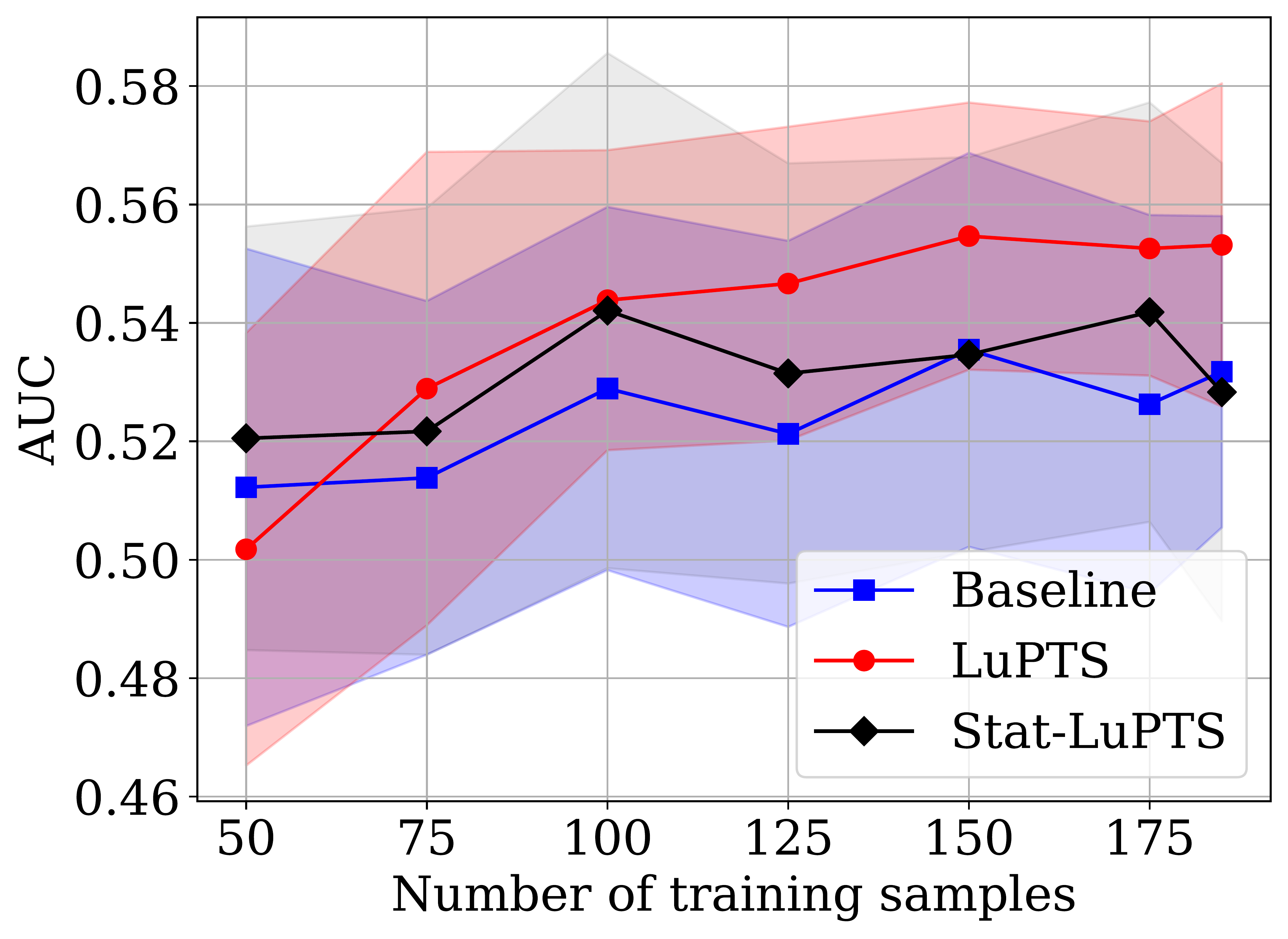}
        \caption{Treatment Response}
        \label{fig:mmfigC}    \end{subfigure}%
    \caption{\small \textbf{MM Progression Task} (Left): AUC as a function of the training set size is shown for 4 privileged time points.  \textbf{Treatment Response Task} (Right) for 2 privileged time points are used. Shaded region for both tasks indicates one standard deviation over all trials.}
    \label{fig:mmexp}
\end{figure}

\paragraph{Alzheimer's Progression Modeling}
In this experiment, we  predict progression of Alzheimer's disease (AD) as measured by two different subject outcomes (i) the diagnostic status (AD or not), and (ii) cognitive function as assessed by the Mini Mental State Examination (MMSE)   score~\citep{galea2005mini}. 
The Alzheimer’s Disease Neuroimaging Initiative (ADNI) is a large multi-site research study tracking the brains of over 2000 participants through genetic, imaging and biospecimen biomarkers.\footnote{ADNI: \url{http://adni.loni.usc.edu}} Subjects are followed over several years with measurements taken every three months, with some missingness in the observations. 

We observe outcomes ($Y$) at a fixed follow-up time, 48 months after baseline measurements ($X_1$) are taken. Privileged information is collected at intermediate time points, here restricted to samples from follow-ups at 12, 24, and 36 months after baseline. 
See the Appendix for all selected features and details on data~pre-processing.

\paragraph{Multiple Myeloma Progression Modeling} Given the limited samples available due to the relative rarity of the disease, Multiple Myeloma (MM) provides a suitable setting to demonstrate the utility of using temporal, post-baseline data in improving predictive performance. We use a data registry released by the Multiple Myeloma Research Foundation (MMRF) through the CoMMpass clinical trial \citep{usrelating}, which contains de-identified clinical data collected at 2-month intervals for 1143 patients. Preprocessing of the data was done through ML-MMRF, an open-source library provided by \citet{hussain2021neural}. 
We refer the reader to the Appendix for a detailed description of the features. We focus on two clinically important predictive tasks for multiple myeloma:

\textit{Early/Late Progression Task}: The first task is predicting whether or not a patient will progress "early" ($Y=1$) (before 18 months post-treatment induction) or "late" ($Y=0$) (after 18 months post-treatment induction). 
We experiment with using privileged time points within the first "line" (or sequence) of treatment.

\textit{Treatment Response Task}: The second task is predicting a patient's treatment response (given a fixed treatment policy) as either "Progressive Disease" (PD, $Y=1$) or "non-Progressive Disease" (non-PD, $Y=0$). These
are used by oncologists to make treatment decisions and assess disease burden. The outcome is recorded after two lines of treatment have been completed. We have privileged information at the end of first line treatment ($t=2$) and at the end of second line treatment~($t=3$).

For all tasks outlined in this section, we do repeated (50 repeats) 2-fold cross validation with different training and test splits across multiple training set sizes. Note that when $T=2$ (one privileged time point), LuPTS and Stat-LuPTS return the same estimator. Hence, only LuPTS is shown in these figures.

\paragraph{LuPTS For Disease Progression Modeling}
Using LuPTS improves predictive performance for all of the clinical tasks
and leads to a reduced variance in estimation, as shown in Figures~\ref{fig:adniexp} and~\ref{fig:mmexp}. This result intuitively implies that it is easier to predict clinical outcomes of chronic diseases from an intermediate set of longitudinal features than from baseline features alone. Indeed, in the context of the MM tasks, the PD category is often determined by temporal changes in a patient's lab values, and recurrence of disease is often measured by temporal changes in a patient's serum immunoglobulins \citep{kyle2009criteria}. We perform additional experiments comparing Stat-LuPTS with LuPTS. For the AD progression tasks, using a stationary transition matrix results in further performance gains (see Figure~\ref{fig:adniexp}). However, for the MM tasks, Stat-LuPTS does not outperform the baseline model (see Figure \ref{fig:mmexp}). This makes sense since the longitudinal dynamics of a myeloma patient may differ across different lines of treatment, justifying a separate transition matrix for each line.

\paragraph{Assessing Feature Importance For MM Early/Late Progression Task}
In Figure \ref{fig:mm_qual} in the Appendix, we show the feature weights of the LuPTS and baseline outcome models in the top and bottom rows of the heatmap, respectively. We find that for the LuPTS estimator, the highest weighted features are the ISS stage and the projected serum M-protein of the patient. This result is consistent with current clinical understanding of myeloma, which measures disease burden based on a patient's M-protein and ISS risk score. On the other hand, we find that the relevant features for the baseline estimator are the myeloma subtypes and not the biomarkers. This is most likely due to the fact that the baseline model takes the biomarkers at the first time step as input, which may be less associated with the overall progression of the patient. This result indicates that using the LuPTS estimator results in a more clinically intuitive explanation for its prediction.

\section{RELATED WORK}
\label{sec:related}
Making use of information only available at training time was first systematically studied in the context of Learning using Privileged Information (LuPI)~\citep{vapnik2009new}. 

We study prediction of future outcomes~\citep{makridakis1994time,ing2003multistep,sorjamaa2007methodology}, where the interval between baseline and prediction target is assumed sufficiently long to collect intermediate privileged information.
This is related to multi-step prediction in time-series forecasting, with common strategies including direct~\citep{chevillon2007direct} and recursive prediction~\citep{kunitomo1985properties,ing2003multistep}, corresponding to the baseline and LupTS strategies used in this work. However, unlike time-series forecasting, which predicts a future state of a continually evolving variable, our goal is to predict a distinct outcome variable at a fixed finite horizon. \citet{ing2003multistep} showed for time-series forecasting that in stationary Gaussian-linear systems, recursive prediction is asymptotically preferable to direct prediction. This is consistent with our findings, but these are for the non-stationary case and non-asymptotic.

Our main analysis tool is the Rao-Blackwell theorem~\citep{radhakrishna1945information,blackwell1947conditional}, used widely for variance reduction of statistical estimators, such as in the Rao-Blackwellization of MCMC sampling schemes~\citep{casella1996rao}. This use is distinct from ours. It has also been used to improve policy evaluation in RL~\citep{li2018policy}, general variational inference~\citep{ranganath2014black}, and estimation of field goal percentage in basketball~\citep{daly2019rao}. However, to the best of our knowledge, the result has not previously been used to prove gains from learning using privileged information.

The question of leveraging explicit models of dynamics commonly arises in reinforcement learning (RL)~\citep{sutton2018reinforcement}. In model-based RL, a learned model of system dynamics is used to simulate state transitions in order to predict (long-term) future rewards. This problem maps onto ours when there is a single available action with the reward being given at a fixed future time step. The question of \emph{when} the bias due to the use of a model in model-based RL is preferable to the higher variance model-free RL remains open~\citep{feinberg2018modelbased,thomas2016data}.

\section{DISCUSSION}
\label{sec:discussion}

In this work, we studied prediction of future outcomes in a setting where privileged information is available in the form of a time series observed between prediction and outcome time points.
We proved that a recursive estimator that makes use of this privileged information yields improved parameter recovery and improved expected risk compared to the best estimator that does not use privileged information.
Through experiments on synthetic and real-world data sets, we showed that our estimator, dubbed LuPTS, often results in better predictive performance and variance reduction in both regression and classification tasks. We also proposed a method for using LuPTS in combination with  distillation-based learning to reduce prediction risk in the misspecified case by trading off bias and variance.

Interestingly, compared to prior work on learning using privileged information, our results are qualitatively different. Instead of providing asymptotic bounds on generalization error as done before~\citep{vapnik2009new,LopSchBotVap16} we prove an explicit gap on the improvement in the finite-sample case. A possibly fruitful direction for research could be to further explore the connections between our results and previous results within the topic.

There are some notable limitations to our work. First, the theory is limited to time series from discrete-time linear dynamical systems with isotropic Gaussian noise where the transition matrices for each time step are estimated separately. Furthermore, we assume that the particular structure of the time series is Markov. 
Extending the theory and algorithm to include non-linear transitions  and estimators or exploit stationary series is interesting future work as it will broaden the understanding and applicability of learning using privileged time series. As a start, we provide a general algorithm for arbitrary estimators in the Appendix.

\subsection*{Acknowledgments}
The authors thank Alexander D'Amour and Chandler Squires for valuable feedback on initial versions of the manuscript. We also thank the Alzheimer’s Neuroimaging Initiative (ADNI) for collecting and providing the data used in this project. In addition, the MMRF data were generated as part of the Multiple Myeloma Research Foundation Personalized Medicine Initiatives (\url{https://research.themmrf.org} and \url{www.themmrf.org}). Fredrik Johansson was funded in part by the Wallenberg AI, Autonomous Systems and Software Program (WASP) funded by the Knut and Alice Wallenberg Foundation. Zeshan Hussain was supported by an ASPIRE award from The Mark Foundation for Cancer Research.

\bibliographystyle{plainnat}
\bibliography{main}


\clearpage
\appendix

\thispagestyle{empty}

\onecolumn \makesupplementtitle

\section*{Appendix}
\appendix

The appendix contains the following sections. We also highlight the key takeaways or descriptions associated with each section. 

\begin{enumerate}[label=\Alph*.]
\item \textbf{Proof of Theorem~\ref{thm:rao}}: This section provides the full proof of Theorem~\ref{thm:rao}, including two lemmas that are central to the argument. We also show how one can relax the isotropic noise assumption on the Gaussian-linear system and generalize the first lemma to the anisotropic case.
\item \textbf{Proof of Theorem~\ref{thm:convex_distillation}}: This section provides a full proof of Theorem~\ref{thm:convex_distillation}, showing how using a distillation-approach with LuPTS as a teacher model returns the convex combination of the OLS estimator and output from Algorithm~\ref{alg:lupts_lin}. We further show that the MSE of the estimator using LuPTS with distillation is bounded between the MSEs of the OLS and LuPTS estimators. 
\item \textbf{LuPTS with Non-linear Estimators}: This section includes a bound on the expected risk of the LuPTS estimator in the case where the transition functions and outcome model are non-linear.
\item \textbf{Experimental Details}
\begin{enumerate}[label=D\arabic*.]
    \item \textbf{Computational Resources} - We give a brief description of the computational resources that were used to generate the experimental results as well as the running time required to reproduce them. 
    \item \textbf{Synthetic Experiments} - We present a more detailed description of how the synthetic data is generated. Additionally, we test empirically in the synthetic setting how the Stat-LuPTS, LuPTS, and baseline OLS estimators compare when stationarity holds and when it does not.
    \item \textbf{Forecasting Air Quality} - We provide a detailed description of the features used for this task as well as the training and evaluation procedure. Finally, we present additional experimental results comparing the LuPTS and baseline OLS estimators for the other Chinese cities.  
    \item \textbf{Alzheimer's Progression Modeling} - Along with giving the full list of features and our pre-processing procedures used, we present the results from the main paper in tabular form. This gives a more granular look at which sample sizes LuPTS yields the most gain in AUC and reduction in variance. 
    \item \textbf{Multiple Myeloma Progression Modeling} - We give a description of the features used for the Multiple Myeloma experiments as well as the pre-processing procedures used to handle missingness and censorship. We then present a brief description of our evaluation procedure on this dataset. Finally, we end with a qualitative experiment looking at the most highly-weighted features in the LuPTS and baseline outcome models for the early/late progression task.
\end{enumerate}
\end{enumerate}
\newpage 

\section{PROOF OF THEOREM~\ref{thm:rao}}
\label{sec:proof-thm1}

To prove Theorem~\ref{thm:rao}, we begin by proving the following lemma of OLS estimates.

\begin{manuallemma}{1}
Let $K=(\hA_1,\dots,\hA_{T-1}, \hbeta)$ be the output of Algorithm~\ref{alg:lupts}, and let $(\bfX_1, \dots, \bfX_T, \bfY)$ be a random dataset from the Gaussian-linear Markov dynamical system as defined in Assumption~\ref{asmp:gauss}, with isotropic noise, $\epsilon_t \sim \cN(0, \sigma_t^2I)$. 
Then, for any $t=2,\dots,T$ we have that
$$
    \E[\bfX_t\mid \bfX_{t-1}, K] = \bfX_{t-1}\hA_{t-1} 
$$
and 
$$
    \E[\bfY\mid \bfX_{T}, K] = \bfX_{T}\hbeta~.
$$
The main difference between the above equations is the dimensionality of $\bfX_t$ and $\bfY$, respectively, where we have previously stated, without loss of generality, that $\bfX_t\in \mathbb{R}^{n\times d}$ and $\bfY\in\mathbb{R}^{n\times 1}$.
\end{manuallemma}

\begin{proof}
We will first show $\E[\bfX_t\mid \bfX_{t-1}, K] = \bfX_{t-1}\hA_{t-1} $ and then explain how the same arguments are applied to prove $\E[\bfY\mid \bfX_{T}, K] = \bfX_{T}\hbeta$.

Let $\bfR_t=\bfX_t-\bfX_{t-1}\hA_{t-1}$ be the residual of the OLS estimate, $\hA_{t-1}$. We will show that for all $\bfR_t$, we have that $p(\bfX_t=\bfX_{t-1}\hA_{t-1} + \bfR_t \mid \bfX_{t-1}, K) = p(\bfX_t'=\bfX_{t-1}\hA_{t-1} - \bfR_t \mid \bfX_{t-1}, K)$, which implies the statement in the lemma if we assume isotropic Gaussian noise.

To show this, we first use Bayes formula:
\begin{align*}
    p&(\bfX_t\mid \bfX_{t-1},K) = \frac{p(K\mid \bfX_t, \bfX_{t-1})p(\bfX_t\mid \bfX_{t-1})}{p(K\mid \bfX_{t-1})} \\
    &= \frac{p(\hbeta, \hA_{T-1},\dots,\hA_{t}\mid \bfX_t)p(\hA_{t-1}\mid  \bfX_{t}, \bfX_{t-1})p(\hA_1, \dots, \hA_{t-2} \mid  \bfX_{t-1})p(\bfX_t\mid \bfX_{t-1})}{p(K\mid \bfX_{t-1})}
\end{align*}
In the second equality, we have used the Markov property, which implies the following statements:
\begin{align*}
    \hA_1, \dots, \hA_{t-2}\indep \bfX_t &\mid  \bfX_{t-1} \\ 
    \hA_{t},\dots, \hA_{T-1},\hbeta \indep \bfX_{t-1} &\mid  \bfX_t \\
    \hA_{t-1} \indep \hA_1, \dots, \hA_{t-2}, \hA_{t},\dots,\hA_{T-1} \hbeta &\mid  \bfX_t, \bfX_{t-1}.
\end{align*}

For $p(\bfX_t\mid \bfX_{t-1},K)=p(\bfX'_t\mid \bfX_{t-1}, K)$ to hold, we look at the factors that depend on $\bfX_t$. This tells us that we need to prove the following three statements:
\begin{description}
    \item[(a)] $p(\bfX_t \mid  \bfX_{t-1}) = p(\bfX'_t \mid  \bfX_{t-1})$ 
    \item[(b)] $p(\hA_{t-1}\mid  \bfX_{t}, \bfX_{t-1}) = p(\hA_{t-1}\mid  \bfX'_{t}, \bfX_{t-1})$
    \item[(c)] $p(\hbeta, \hA_{T-1},\dots,\hA_{t}\mid \bfX_t) = p(\hbeta, \hA_{T-1},\dots,\hA_{t}\mid \bfX'_t)$
\end{description}

We will now prove each of these statements:

\textbf{Statement (a)}: We first define $\bfeps'_t=\bfeps - 2\bfR_t$ where we have that $\bfX_t=\bfX_{t-1}A_{t-1}+\bfeps$.  Then, note that
\begin{equation}\label{eq:sta}
\bfX_{t-1}A_{t-1}+\bfeps'_t = \bfX_{t-1}A_{t-1}+\bfeps - 2\bfR_t =\bfX_t - 2\bfR_t = \bfX_{t-1}\hA_{t-1} - \bfR_t = \bfX'_t
\end{equation}
Eq.~\ref{eq:sta} implies that showing (a) equates to showing that $p(\bfeps)=p(\bfeps'_t)$, since the noise is independent of $\bfX_{t-1}$. For Gaussian noise, these probabilities are determined by the inner product of the noise, hence it is sufficient to prove
$$
 \bfeps^\top\bfeps = {\bfeps'}^\top\bfeps'
$$
We have that
\begin{align*}
{\bfeps'}^\top \bfeps' = \bfeps^\top \bfeps - 4 \bfeps^\top \bfR_t + 4\bfR_t^\top \bfR_t
\end{align*}
and thus, we need to show 
$$
\bfR_t^\top(\bfeps - \bfR_t) = 0~.
$$
By definition, 
$$
\bfR_t = \bfX_t - \bfX_{t-1}\hA_{t-1}
$$
and so 
$$
\bfR_t^\top(\bfeps - \bfR_t) = \bfR_t^\top(\bfX_t - \bfX_{t-1}A_{t-1} - (\bfX_t - \bfX_{t-1}\hA_{t-1})) = -\bfR_t^\top( \bfX_{t-1}(A_{t-1} - \hA_{t-1})) = 0~
$$
since $\bfR_t^\top \bfX_{t-1}=0$ is a property of the OLS estimator. This proves statement (a).

\textbf{Statement (b)}: Now, we see that 
\begin{align*}
\hA'_{t-1} &= (\bfX_{t-1}^\top \bfX_{t-1})^{-1}\bfX_{t-1}^\top \bfX'_{t} \\ 
&= (\bfX_{t-1}^\top \bfX_{t-1})^{-1}\bfX_{t-1}^\top (\bfX_{t} - 2\bfR_t) \\
&= \hA_{t-1} - 2(\bfX_{t-1}^\top \bfX_{t-1})^{-1}\bfX_{t-1}^\top \bfR_t = \hA_{t-1}~.
\end{align*}
since, again, $\bfX_{t-1}^\top \bfR_t= 0$. This implies that the distribution of $\hA_{t-1}$ is the same if conditioned on $\bfX_t$ or $\bfX'_t$, which proves statement (b). 

\textbf{Statement (c)}: We can factorize $p(\hbeta, \hA_{T-1},\dots,\hA_{t}\mid \bfX_t)$ as
$$
     p(\hbeta \mid   \hA_{T-1},\dots,\hA_{t}, \bfX_t) p(\hA_{T-1} \mid  \hA_{T-2},\dots,\hA_{t},  \bfX_t)\dots p(\hA_{t+1} \mid \hA_{t},  \bfX_t)p(\hA_{t} \mid  \bfX_t)
$$
Let $C_k = (\hA_{k}, \hA_{k-1},\dots,\hA_t)$ for $k=t,\dots,T-1$. Then, we can summarize the problem as the following: we need to show that each factor in the above equation is the same for both $\bfX_t$ and $\bfX'_t$, i.e.,
\begin{align*}
    &p(\hbeta \mid  C_{T-1}, \bfX_t) = p(\hbeta \mid  C_{T-1}, \bfX'_t) \\
    &p(\hA_k \mid  C_{k-1}, \bfX_t) = p(\hA_k \mid  C_{k-1}, \bfX'_t),\quad  k=t+1, \dots, T-1 \\
    &p(\hA_{t} \mid  \bfX_t) = p(\hA_{t} \mid  \bfX'_t)
\end{align*}

The first two equations could be seen as the distribution of OLS estimators with a (conditional) random design, while the third is the distribution of $\hA_{t}$ with a fixed design matrix $\bfX_t$.
Assuming mean-zero and uncorrelated Gaussian noise $\epsilon_{t} \sim \cN(0, \sigma_t^2 I)$ and $\epsilon_{Y} \sim \cN(0, \sigma_Y^2 I)$, the distributions of the OLS estimators are known, see Chapter 14 in~\cite{rice2006mathematical}, and we can show,
\begin{align*}
\hbeta \mid  C_{T-1}, \bfX_t \sim &\mathcal{N} \left( \beta, \sigma_Y^2 \E\left[\left(\bfX_{T}^\top \bfX_{T} \right)^{-1} \mid C_{T-1}, \bfX_t \right] \right)
\\
\hA_{k}^{(\text{row }i)}  \mid  C_{k-1}, \bfX_t \sim &\mathcal{N} \left( A_k^{(\text{row }i)} , \sigma_{k+1}^2 \E\left[\left(\bfX_k^\top \bfX_k \right)^{-1} \mid C_{k-1}, \bfX_t \right] \right), \quad k=t+1,\dots,T-1
\\
\hA_{t}^{(\text{row }i)} \mid  \bfX_t \sim &\mathcal{N} \left( A_t^{(\text{row }i)} , \sigma_{t+1}^2\left(\bfX_t^\top \bfX_t \right)^{-1} \right)
\end{align*}
where $i=1,\dots,d$ corresponds to the OLS estimators which are $d\times d$ matrices.
Now, it is sufficient to show that 
$$
\E\left[\left(\bfX_k^\top \bfX_k \right)^{-1} \mid  C_{k-1}, \bfX_t \right] = \E\left[\left(\bfX_k^\top \bfX_k \right)^{-1} \mid  C_{k-1}, \bfX'_t \right], \quad k=t+1,\dots, T
$$
and the special case where $\bfX_t^\top \bfX_t = {\bfX'}_t^\top \bfX'_t$. For the latter, we have
\begin{align*}
 (\bfX_{t-1}\hA_{t-1} \pm \bfR_t)^\top& (\bfX_{t-1}\hA_{t-1} \pm \bfR_t) = \\
 &= (\bfX_{t-1}\hA_{t-1})^\top(\bfX_{t-1}\hA_{t-1}) \pm 2 (\bfX_{t-1}\hA_{t-1})^\top \bfR_t + \bfR_t^\top \bfR_t \\
 &= (\bfX_{t-1}\hA_{t-1})^\top(\bfX_{t-1}\hA_{t-1}) + \bfR_t^\top \bfR_t
\end{align*}
where we have used that the cross term $(\bfX_{t-1}\hA_{t-1})^\top \bfR_t = \hA_{t-1}^\top \bfX_{t-1}^\top \bfR_t = 0$ because $\bfX_{t-1}^\top \bfR_t=0$. As the cross term is the only thing which differs in $\bfX_t^\top \bfX_t$ and ${\bfX'}_t^\top \bfX'_t$, the above derivation implies that they must be equal. 

For $\E\left[\left(\bfX_k^\top \bfX_k \right)^{-1} \mid  C_{k-1}, \bfX_t \right]$ with $k=t+1, \dots, T$, we use the same expression as before but observe the following recursive relationship between the inner product of $\bfX_k$ and $\bfX_{k-1}$:
\begin{align*}
\bfX_k^\top \bfX_k &= (\bfX_{k-1}\hA_{k-1})^\top \bfX_{k-1}\hA_{k-1} + \bfR_k^\top \bfR_k = \hA_{k-1}^\top \underbrace{\bfX_{k-1}^\top \bfX_{k-1}}_{\text{Inner product}}\hA_{k-1} + \bfR_k^\top \bfR_k
\end{align*}

Hence, we get that
$$
    \bfX_k^\top \bfX_k = \prod_{i=t}^{k-1} \hA_i^\top (\bfX_t^\top \bfX_t) \prod_{i=t}^{k-1} \hA_i + \sum_{j=t+1}^{k-1} \bfR_j^\top \bfR_j \prod_{i=j}^{k-1} \hA_i +  \bfR_k^\top \bfR_k
$$
We see that $\bfX_k^\top \bfX_k$ is directly dependent on $\bfX_t^\top \bfX_t$, noting that the residuals and OLS estimators are fixed given that we condition upon them. Since we already have shown that $\bfX_t^\top \bfX_t={\bfX'}_t^\top \bfX'_t$, this means that $\E\left[\left(\bfX_k^\top \bfX_k \right)^{-1} \mid  C_{k-1}, \bfX_t \right]=\E\left[\left(\bfX_k^\top \bfX_k \right)^{-1} \mid  C_{k-1}, \bfX'_t \right]$ for $k=t+1,\dots,T$, which completes the proof for statement (c).

Since we have proven all three statements that were presented in the beginning of this proof, we have shown that  $p(\bfX_t=\bfX_{t-1}\hA_{t-1} + \bfR_t \mid \bfX_{t-1}, K) = p(\bfX'_t=\bfX_{t-1}\hA_{t-1} - \bfR_t \mid \bfX_{t-1}, K)$.

Finally, to show that $\E[\bfY\mid \bfX_{T}, K] = \bfX_{T}\hbeta$, we can use the same arguments as before to show $p(\bfY=\bfX_T\hbeta + \bfR_Y \mid \bfX_T, K)=p(\bfY'=\bfX_T\hbeta - \bfR_Y \mid \bfX_T, K)$, although only statements (a) and (b) are necessary for this case.

\end{proof}

\begin{thmrem}
For the anisotropic case, the analysis becomes slightly different.
The noise in the data $\bfX_t$ is $\bfeps_t = [\epsilon_{t,1}, \dots \epsilon_{t,n}]^\top\in \mathbb{R}^{n \times d}$. The rows corresponds the the noise in a particular sample, while the columns are for the different features. Furthermore, the covariance of the $i$th feature is $\cov(\bfeps_t^{(\text{column }i)}) = \sigma_{t,i}^{2}I_n$ for $i=1,\dots,d$ where $I_n$ is the $n$-dimensional identity matrix. With anisotropic noise, we have $\sigma_{t,i}\neq \sigma_{t,i'}$ for $i,i'=1,\dots,d$. Then, the above lemma will be feasible using a similar analysis as we can show that,
$$
\hA_k^{(\text{row }i)} \sim \cN \left(A_k^{(\text{row }i)} , \sigma_{k+1,i}^2\; \E\left[\left(\bfX_k^\top \bfX_k \right)^{-1} \mid C_{k-1}, \bfX_t \right] \right)~.
$$
Then, the analysis follows as in Lemma~\ref{lem:ols_general}.
\end{thmrem}

Now, we prove that $\hA\hbeta$ is a sufficient statistic for $\htheta$.

\begin{thmlem}
    Let $D = (\bfX_1, \bfX_2, \dots, \bfX_{T}, \bfY)$ be a random dataset from a Gaussian-linear Markov dynamical system, as defined in Assumption~\ref{asmp:gauss}.
    Then, let $\htheta_\LuPTS = \hat{A}\hat{\beta}$ be the output of Algorithm~\ref{alg:lupts_lin} without stationarity, and $\htheta_\OLS \coloneqq (\bfX_1^\top \bfX_1)^{-1} \bfX_1^\top \bfY$. It holds that,
    $$
    \E_D[\htheta \mid \hA, \hbeta] = \hA \hbeta ~.
    $$
    \label{lem:suff_general}
\end{thmlem}
\begin{proof}
Let smaller letters $(\bfx_1, \dots, \bfx_T, \bfy)$ indicate a value of the random dataset $D$ and $B=(\bfx_1^\top \bfx_1)^{-1}\bfx_1^\top$. Then, we have,
\begin{align*}
    &\E[\htheta\mid  \hA, \hbeta] = \int  p(\bfx_1,\dots,\bfx_T,\bfy\mid \hA, \hbeta) \htheta  d\bfX_1\dots d\bfX_Td\bfY \\
    & = \int  p(\bfy\mid \bfx_T,\hA, \hbeta) \prod_{t=2}^T p(\bfx_t\mid \bfx_{t-1},\hA, \hbeta) p(\bfx_1\mid \hA, \hbeta) \htheta d\bfX_1\dots d\bfX_T d\bfY \quad\text{(Markov property)} \\
    &= \int p(\bfy\mid \bfx_T,\hA, \hbeta) \prod_{t=2}^T p(\bfx_t\mid \bfx_{t-1},\hA, \hbeta) p(\bfx_1\mid \hA, \hbeta) \underbrace{B\bfy}_{=\htheta} d\bfX_1\dots d\bfX_T d\bfY \quad \text{(OLS definition)} \\
    &= \int \prod_{t=2}^T p(\bfx_t\mid \bfx_{t-1},\hA, \hbeta) p(\bfx_1\mid \hA, \hbeta) B  \underbrace{\left[\int \bfy p(\bfy\mid \bfx_T,\hA, \hbeta) d\bfY\right]}_{=\E[\bfY\mid \bfx_T,\hA, \hbeta]= \bfx_T\hat{\beta}} d\bfX_1\dots d\bfX_T \quad \text{(Lemma~\ref{lem:ols_general})} \\
    &= \int \prod_{t=2}^{T} p(\bfx_t\mid \bfx_{t-1},\hA, \hbeta) p(\bfx_1\mid \hA, \hbeta) B\bfx_T\hat{\beta}  d\bfX_1\dots d\bfX_T \\ 
    &=  \int \prod_{t=2}^{T-1} p(\bfx_t\mid \bfx_{t-1},\hA, \hbeta) p(\bfx_1\mid \hA, \hbeta) B \underbrace{\left[\int \bfx_T p(\bfx_T\mid \bfx_{T-1},\hA, \hbeta) d\bfX_T\right]}_{=\E[\bfX_T\mid \bfx_{T-1},\hA, \hbeta]=\bfx_{T-1}\hA_{T-1}} \hat{\beta}  d\bfX_1\dots d\bfX_{T-1} \quad \text{(Lemma~\ref{lem:ols_general})} \\
    &= \int \prod_{t=2}^{T-1} p(\bfx_t\mid \bfx_{t-1},\hA, \hbeta) p(\bfx_1\mid \hA, \hbeta) B\bfx_{T-1} \hA_{T-1}\hat{\beta}  d\bfX_1\dots d\bfX_{T-1} \\
    &=  \int \prod_{t=2}^{T-2} p(\bfx_t\mid \bfx_{t-1},\hA, \hbeta) p(\bfx_1\mid \hA, \hbeta) B\underbrace{\left[\int \bfx_{T-1}p(\bfx_{T-1}\mid \bfx_{T-2},K)d\bfX_{T-1}\right]}_{=\E[\bfX_{T-1}\mid \bfx_{T-2},\hA, \hbeta]=\bfx_{T-2}\hA_{T-2}} \hA_{T-1}\hat{\beta}  d\bfX_1\dots d\bfX_{T-2} \quad \text{(Lemma~\ref{lem:ols_general})} \\ 
    &= \dots = \quad \text{(recursively)}\\
    &= \int p(\bfx_1\mid \hA, \hbeta) \underbrace{B\bfx_1}_{=I} \hA_{1}\dots\hA_{T-1}\hbeta d\bfX_1 \\
    &= \hA_{1}\dots\hA_{T-1}\hbeta \int p(\bfx_1\mid \hA, \hbeta) d\bfX_1 = \hA_{1}\dots\hA_{T-1}\hbeta
\end{align*}
\end{proof}

For the final result, recall the definition of parameter mean squared error for an estimate $\htheta$ of $\theta$, where the expectation is taken over the dataset $D$ used to fit $\htheta$,
$$
\mse(\htheta) \coloneqq \E_D[\|\htheta - \theta\|_2^2]
$$
and the expected risk for a prediction function $h_D$ of baseline variables $X_1$ dependent on the random dataset $D$,
$$
\olR(h_D) \coloneqq \E_D[R(h_D)] = \E_D[\E_{X_1, Y}[(f_D(X_1) - Y)^2]]~.
$$
In the linear case, we let $\olR(\htheta)$ denote $\olR(\htheta^\top (\cdot) )$.

\begin{manualtheorem}{1}%

Let $D = (\bfX_1, \bfX_2, ..., \bfX_T, \bfY)$ be a random dataset with $\htheta_\OLS \coloneqq (\bfX_1^\top \bfX_1)^{-1} \bfX_1^\top \bfY$, and let  $\htheta_\LuPTS = \hat{A}\hat{\beta}$ be the output of Algorithm~\ref{alg:lupts_lin} without stationarity. 
Under the Gaussian-linear system defined in Assumption~\ref{asmp:gauss} with isotropic noise as in Lemma~\ref{lem:ols_general}, $\htheta_\LuPTS$ is unbiased, and 
\begin{equation}
\mse{}(\htheta_\LuPTS) = \mse{}(\htheta_{\mathrm{OLS}}) - \E_D[\trace(\cov(\htheta_\OLS \mid \htheta_\LuPTS))]~,
\end{equation}
where the expectation is taken over random datasets $D$, since both estimators are functions of them.  
Further, it holds for the expected risk that over new, unseen samples $(X_1, Y)$, 
\begin{equation}
\olR(\htheta_\LuPTS) = \olR(\htheta_\OLS) - \E_{D, X_1}[\V_{\htheta_\OLS}(\langle\htheta_\OLS,  X_1 \rangle \mid \htheta_\LuPTS)]~.
\end{equation}%
\end{manualtheorem}

\begin{proof}
Unbiasedness of $\htheta_\LuPTS$ follows from Lemma~\ref{lem:suff_general} or the standard proof for unbiasedness of $\htheta_\OLS$. The remaining result follows from Lemma~\ref{lem:suff_general} and standard Rao-Blackwell arguments,
\begin{align*}
\mse(\hA\hbeta) & = \E_D[\|\hA\hbeta - \theta\|^2] \\
& = \E_D[\|\E [\htheta \mid \hA\hbeta] - \theta\|^2] \quad \text{(Lemma~\ref{lem:suff_general})}\\
& = \E_D [\|\E[\htheta - \theta \mid \hA\hbeta]\|^2] \\
& = \E_D \left[\sum_{j=1}^d (\E[\htheta_j - \theta_j \mid \hA\hbeta])^2 \right] \\
& = \E_D \left[\sum_{j=1}^d \left( \E[(\htheta_j - \theta_j)^2 \mid \hA\hbeta] - \V\left[\htheta_j \mid \hA\hbeta \right] \right) \right] \\
& = \E_D[\E[\|\htheta - \theta\|^2 \mid \hA\hbeta]] - \E_D \left[\sum_{j=1}^d\V\left[\htheta_j \mid \hA\hbeta \right]\right] \\
& = \mse(\htheta)- \E_D\left[ \trace \left(\cov\left[\htheta \mid \hA\hbeta \right] \right) \right]~. \\
\end{align*}

Recall that $X_1$ represents a new test point, independent of the dataset $D$. For the second result, we note that for any estimator $\htheta$,
$$
\E_D[R(\htheta)] = \E_D[\E_{X_1,Y}[(\htheta^\top X_1 - Y)^2]] = \E_{X_1}[\E_D[\E_{Y\mid X_1}(\htheta^\top X_1 - Y)^2 \mid X_1]]~.
$$
Then, if $\htheta$ is unbiased for the Gaussian linear model, $Y = \theta^\top X_1 + \epsilon $ with $\epsilon \sim \mathcal{N}(0, \sigma^2)$, 
$$
\E_D[\E_{Y\mid X_1}[(\htheta^\top x_1 - Y)^2 \mid X_1=x_1]] = \underbrace{\E_D[( \htheta^\top x_1 -  \E_D[\htheta]^\top x_1 )^2]}_{=\; \mbox{variance}} + \underbrace{(\E_D[\htheta]^\top x_1 - \theta^\top x_1)^2}_{=\; \mbox{bias}^2\; = 0} + \sigma^2~.
$$
Since $\E_D[\htheta] = \theta$, the variance term can then be rewritten as, 
$$
\E_D[( \htheta^\top x_1 -  \E_D[\htheta]^\top x_1 )^2] = \E_D[\langle \htheta - \theta, x_1 \rangle^2]~.
$$
Then, since $\hA\hbeta$ is an unbiased estimator of $\theta$,
\begin{align*}
\E_{D}[R(\hA\hbeta)] & = \E_{X_1}[  \E_D[\langle \hA\hbeta - \theta,  X_1\rangle^2] ] + \sigma^2 \\
& = \E_{X_1}[ \E_D[\langle \E_{\htheta}[\htheta \mid \hA\hbeta] - \theta,  X_1\rangle^2] ] + \sigma^2 \quad \text{(Lemma~\ref{lem:suff_general})}\\
& = \E_{X_1}[ \E_D[\E_{\htheta}[\langle\htheta - \theta,  X_1\rangle \mid \hA\hbeta]^2] ] + \sigma^2 \\
& = \E_{X_1}[ \E_D[\E_{\htheta}[\langle\htheta - \theta,  X_1\rangle^2 \mid \hA\hbeta] - \V_{\htheta}(\langle\htheta - \theta,  X_1\rangle \mid \hA\hbeta)] ] + \sigma^2 \\
& = \E_D[R(\htheta)] - \E_{D, X_1}[\V_{\htheta}(\langle\htheta - \theta,  X_1\rangle \mid \hA\hbeta)] ~.
\end{align*}
In the last step, we make use of the fact that $\htheta$ is unbiased and
$$
\E_D[R(\htheta)] = \E_{X_1}[ \E_D[\langle\htheta - \theta,  X_1\rangle^2 ] + \sigma^2~.
$$
\end{proof}

\section{PROOF OF THEOREM 2}

We extended the distillation-based method as described by \citet{hayashi2019long} with LuPTS as teacher model, which we called Distill-Seq. In the linear setting with squared loss, the distillation loss function is defined as,
\begin{equation}
 \htheta_{Dist} = \argmin_\theta \lambda ||\bfY-\bfX_1\theta||_2^2 + (1-\lambda)||\bfY_{soft}-\bfX_1\theta||_2^2 
 \label{eq:distill_loss_appendix}
\end{equation}
where $\lambda \in [0,1]$ and $\bfY_{soft}$ is the soft target provided by the teacher. In the case where a student model minimizes the above loss function with LuPTS as teacher model, we can prove the following theorem.

\begin{manualtheorem}{2}
Let $\htheta_\LuPTS$ be the output of Algorithm~\ref{alg:lupts_lin} and $\htheta_\OLS=(\bfX_1^\top \bfX_1)^{-1}\bfX_1^\top \bfY$. Let $\htheta_{Dist}$ be the solution to \eqref{eq:distill_loss_appendix} with $\hat{\bfY}_{soft}=\bfX_1\htheta_\LuPTS$ and $\lambda\in[0,1]$. Then, it holds that
\begin{equation}
     \htheta_{Dist} = \lambda \htheta_\OLS + (1-\lambda)\htheta_\LuPTS~.
     \label{eq:convex_combination_appendix}
\end{equation}
Additionally, under Assumption~\ref{asmp:gauss}, it holds that
\begin{equation}
    \mse (\htheta_\LuPTS) \leq \mse(\htheta_{Dist}) \leq \mse (\htheta_\OLS)~.
    \label{eq:sandwich_mse_appendix}
\end{equation}
\end{manualtheorem}

\begin{proof}
We will first show the first part of the theorem, namely that equation~\eqref{eq:convex_combination_appendix} holds. Then, as a consequence, we will proceed with proving that equation~\eqref{eq:sandwich_mse_appendix} holds. 

Since the optimization problem in \eqref{eq:distill_loss_appendix} is convex, we can compute the derivative with respect to $\theta$ and find the value for which the derivative is zero,
\begin{align*}
    \frac{d}{d\theta} &= \left( \lambda ||\bfY-\bfX_1\theta||_2^2 + (1-\lambda)||\bfX_1\htheta_\LuPTS-\bfX_1\theta||_2^2  \right) \\
    &= \left( 2\lambda \bfX_1^\top (\bfY-\bfX_1\theta) + 2(1-\lambda) \bfX_1^\top (\bfX_1\htheta_\LuPTS-\bfX_1\theta) \right) \\
    &= 2\bfX_1^\top \left( \underbrace{\lambda \bfY + (1-\lambda)\bfX_1\htheta_\LuPTS}_{\tilde{\bfY}}  - \bfX_1 \theta \right) = 0~.
\end{align*}
The solution is given by the OLS estimate 
$$
    \htheta_{Dist} = (\bfX_1^\top \bfX_1)^{-1} \bfX_1^\top \tilde{\bfY}
$$
where we can expand $\tilde{\bfY}$ to get the following,
\begin{align*}
    \htheta_{Dist} &= (\bfX_1^\top \bfX_1)^{-1} \bfX_1^\top \left( \lambda \bfY + (1-\lambda)\bfX_1\htheta_\LuPTS \right) \\
    &= \lambda (\bfX_1^\top \bfX_1)^{-1} \bfX_1^\top \bfY + (1-\lambda)(\bfX_1^\top \bfX_1)^{-1} \bfX_1^\top \bfX_1\htheta_\LuPTS \\
    &= \lambda \htheta_\OLS + (1-\lambda)\htheta_\LuPTS~.
\end{align*}
This proves the first part of the theorem. 
Next, we prove equation~\eqref{eq:sandwich_mse_appendix}, which will be done element-wise, i.e. we prove the statement for $\htheta_{Dist}^{(j)}$ for some $j=1,\dots, d$. 

First, due to equation~\eqref{eq:distill_loss_appendix}, we note that since $\htheta_\OLS$ and $\htheta_\LuPTS$ are unbiased estimators, $\htheta_{Dist}$ is also unbiased. Hence, we can write

$$
    \mse(\htheta_{Dist}^{(j)}) = \V(\htheta_{Dist}^{(j)}) + \underbrace{\text{Bias}(\htheta_{Dist}^{(j)})^2}_{=0} = \V (\htheta_{Dist}^{(j)}) 
$$
Then, we use equation~\eqref{eq:convex_combination_appendix} again to rewrite the variance of $\htheta_{Dist}^{(j)}$,
\begin{align*}
    \V (\htheta_{Dist}^{(j)}) &= \V (\lambda \htheta_\OLS^{(j)} + (1-\lambda)\htheta_\LuPTS^{(j)}) \\
    &= \lambda^2 \V (\htheta_\OLS^{(j)} ) + (1-\lambda)^2 \V (\htheta_\LuPTS^{(j)}) + 2\lambda(1-\lambda)\cov(\htheta_\OLS^{(j)}, \htheta_\LuPTS^{(j)})~.
\end{align*}

We shall focus on the covariance term, and using the law of total covariance we can show the following,
\begin{align*}
    \cov(\htheta_\OLS^{(j)}, \htheta_\LuPTS^{(j)}) =& \E_{\htheta_\LuPTS^{(j)}} \left[ \cov(\htheta_\OLS^{(j)}, \htheta_\LuPTS^{(j)} \mid \htheta_\LuPTS^{(j)}) \right]   \\
    & + \cov_{\htheta_\LuPTS^{(j)}} \left( \E[ \htheta_\OLS^{(j)} \mid \htheta_\LuPTS^{(j)} ],  \E[ \htheta_\LuPTS^{(j)} \mid \htheta_\LuPTS^{(j)} ] \right) \\
    = & \E_{\htheta_\LuPTS^{(j)}} \left[ \E\left[ (\htheta_\OLS^{(j)} - \E[ \htheta_\OLS^{(j)} \mid \htheta_\LuPTS^{(j)} ]) \underbrace{( \htheta_\LuPTS^{(j)} - \E[ \htheta_\LuPTS^{(j)} \mid \htheta_\LuPTS^{(j)} ])}_{=0} \right] \mid \htheta_\LuPTS^{(j)}) \right] 
    \\
    & + \cov_{\htheta_\LuPTS^{(j)}} \left( \htheta_\LuPTS^{(j)} ,  \htheta_\LuPTS^{(j)} \right) \quad \text{ by Lemma 2 and definition of covariance}
    \\
    =& \V(\htheta_\LuPTS^{(j)})~.
\end{align*}

Hence, we end up with 

\begin{align*}
    \V (\htheta_{Dist}^{(j)}) &= \lambda^2 \V (\htheta_\OLS^{(j)} ) + (1-\lambda)^2 \V (\htheta_\LuPTS^{(j)}) + 2\lambda(1-\lambda)\cov(\htheta_\OLS^{(j)}, \htheta_\LuPTS^{(j)}) \\
    &= \lambda^2 \V (\htheta_\OLS^{(j)} ) + (1-\lambda)^2 \V (\htheta_\LuPTS^{(j)}) + 2\lambda(1-\lambda)\V(\htheta_\LuPTS^{(j)}) \\
    &= \lambda^2 \V (\htheta_\OLS^{(j)} ) + (1-\lambda^2) \V (\htheta_\LuPTS^{(j)})~.
\end{align*}

Looking at the last line of the previous equation, we know from Theorem~\ref{thm:rao} that $\V (\htheta_\LuPTS^{(j)} ) \leq \V (\htheta_\OLS^{(j)} )$. Hence, the lower bound of $ \V (\htheta_{Dist}^{(j)})$ is obtained by setting $\lambda=0$ and, similarly, the upper bound is obtained when $\lambda=1$. This is possible since we can choose $\lambda\in [0,1]$ freely. This leads to the following results, 
$$
\V (\htheta_\LuPTS ^{(j)}) \leq \V (\htheta_{Dist}^{(j)}) \leq \V (\htheta_\OLS^{(j)})
$$
which, due to the unbiasedness of the estimators, can be written as
$$
\mse (\htheta_\LuPTS^{(j)}) \leq \mse(\htheta_{Dist}^{(j)}) \leq \mse(\htheta_\OLS^{(j)})~.
$$
Lastly, since $\mse (\htheta_\LuPTS) = \sum_{j=1}^d \mse (\htheta_\LuPTS^{(j)})$ and that the inequality holds element-wise, we have that,
$$
    \mse (\htheta_\LuPTS) \leq \mse(\htheta_{Dist}) \leq \mse (\htheta_\OLS)~.
$$

\end{proof}

\section{LuPTS WITH NON-LINEAR ESTIMATORS}
Under Assumption~\ref{asmp:markov} (Markovianity), it is natural to consider the following generalized (non-linear) procedure: a) For each time-step $t$, fit a transition function $f_t$ predicting $X_{t+1}$ from $X_t$, b) Fit $g$ to predict $Y$ from $X_T$, c) Return $h = g \circ f_{T-1} \circ \cdots \circ f_1$. This approach outlined in Algorithm~\ref{alg:lupts}. The idea may be compared to model-based value estimates in reinforcement learning~\citep{sutton2018reinforcement}, in which predictions of future rewards are based on simulating roll-outs under a learned policy and model of state dynamics. 


\SetKwInput{kwParam}{Parameters}
\begin{algorithm}[t!]
    \kwParam{Function classes $\cF$, $\cG$, loss function $L$}
    \KwData{$D = \{(x_{1,1}, ..., x_{1,T}, y_1), ..., (x_{m,1}, ..., x_{m,T}, y_m)\} \sim p^m(X_1, ..., X_T, Y)$}
    \vspace{.5em}
    \For{$t = 1, ..., {T-1}$}{
        $\hat{f}_t = \argmin_{f_t \in \cF} \frac{1}{m} \sum_{i=1}^m L(f_t(x_{i, t}), x_{i,t+1})$ \\
    }
    $\hat{g} = \argmin_{g \in \cG} \frac{1}{m} \sum_{i=1}^m L(g(x_{i, T}), y_i)$ \\
    \vspace{.5em}
    \Return $h = \hat{g} \circ \hat{f}_{T-1} \circ ... \circ \hat{f}_1 $
 \caption{Learning using privileged time series (LuPTS)}
 \label{alg:lupts}
\end{algorithm}

In the general case, without assumptions on the data-generating process or the hypothesis classes $\cF$ and $\cG$, we may bound the expected risk of the LuPTS estimator in terms of the risk accumulated in simulating the system dynamics through $f$, and that of the outcome model $g$.

\begin{thmthm}[Risk expansion]
Let $\hh = \hg \circ \hf$ be the output of Algorithm~\ref{alg:lupts} with $\hat{f}=\hat{f}_{T-1} \circ \hat{f}_{T-1} \circ \dots \circ \hat{f}_{1}$ the estimated system dynamics and $\hat{g}$ the prediction model of $Y$ from $X_T$. Then, 
\begin{equation}
    R(\hf \circ \hg) \leq R_{X_T}(\hf) + R_{Y}(\hg) + 2\sqrt{R_{X_T}(\hf)R_{Y}(\hg)}
\end{equation}
where $R(\hf \circ \hg) = \E[(Y - \hg(\hf(X_1)))^2]$ is the expected risk of predicting $Y$ from $X_1$,
$$
R_{X_T}(f) = \E\left[\left(\hat{g}(\hat{f}(X_1)) - \hat{g}(X_T) \right)^2\right]
\;\;\mbox{ and }\;\;
R_{Y}(\hg) = \E\left[\left(Y - \hat{g}(X_T)\right)^2\right]~.
$$
Here, $R_{X_T}(f)$ is the mean squared error in predictions of $Y$ that stems from errors in the learned dynamical system while $R_{Y}(\hg)$ is due to the error in the outcome model $\hat{g}$.
\end{thmthm}
\begin{proof}
As seen in~\citep{feinberg2018modelbased}.
\begin{equation}
    \begin{aligned}
        \E\left[ \left(\hat{g}(\hat{f}(x_1)) - y \right)^2\right]
        =& \E\left[ \left(\hat{g}(\hat{f}(x_1)) - y + \hat{g}(x_T) - \hat{g}(x_T) \right)^2\right] \\
        =& \E\left[\left((\hat{g}(\hat{f}(x_1)) - \hat{g}(x_T)) - (y - \hat{g}(x_T))\right)^2\right] \\
        =& \E\left[\left(\hat{g}(\hat{f}(x_1)) - \hat{g}(x_T) \right)^2\right] + \E\left[\left(y - \hat{g}(x_T)\right)^2\right] \\
        & - 2 \E\left[\left(\hat{g}(\hat{f}(x_1)) - \hat{g}(x_T) \right)\left(y - \hat{g}(x_T)\right) \right] \\
        &\leq \E\left[\left(\hat{g}(\hat{f}(x_1)) - \hat{g}(x_T) \right)^2\right] + \E\left[\left(y - \hat{g}(x_T)\right)^2\right] \\
        & + 2 \sqrt{\E\left[\left(\hat{g}(\hat{f}(x_1)) - \hat{g}(x_T) \right)^2\right] \E\left[\left(y - \hat{g}(x_T)\right)^2 \right]}
    \end{aligned}
\end{equation}
The first equalities are algebra, and the inequality step comes from the Cauchy-Schwarz inequality for random variables. 
\end{proof}
\begin{thmrem}
A similar result appears in~\citet{feinberg2018modelbased} for the case of model-based value expansion for model-free reinforcement learning. Although the bound cannot be compared directly to the risk of the baseline method, it gives an indication about how the LuPTS algorithm behaves. $R_{X_T}(f)$ can be expected to increase as $T$ becomes larger since $X_T$ gets "further away" from $X_1$, making $X_T$ more difficult to predict. Meanwhile, $R_{Y}(\hg)$ is unaffected by this. 
\end{thmrem}

\section{EXPERIMENT DETAILS}
\label{app:experiment_details}

\subsection{Computational Resources}

All procedures for pre-processing real-world datasets, generating synthetic datasets, training and evaluating models are implemented in Python with the help of standard scientific modules such as NumPy and scikit-learn. The experiments are run on mid-tier laptops generally utilizing one CPU core. Running times for each individual experiment under this setup rarely exceed a couple of minutes. The full set of experiments can be reproduced in less than 48 hours. 

\subsection{Synthetic Experiments}
\label{app:synthetic_details}

We give a detailed description of how the synthetic data we use in the experiments is generated. As a reminder, the Gaussian-linear dynamical system of interest is 
\begin{equation*}
	\label{eq:synthetic_sys}
	\begin{aligned}
      X_t  &= A_{t-1}^{\top}X_{t-1} + \epsilon_t, \quad \text{ for } t=2,\dots,T \\
      Y  &= \beta^{\top}X_T + \epsilon_Y~.
    \end{aligned}
\end{equation*}

To verify and further investigate our theoretical results, we sample from a synthetic dynamical system where Markovianity and linearity with additive isotropic Gaussian noise hold. The parameters $A_t\in \mathbb{R}^{d\times d}$ and $\beta\in \mathbb{R}^{d\times 1}$ were generated in the following way: For each $t=1,\dots,T-1$, all elements in $A_t$ are sampled independently from a Normal distribution $\cN(\mu=0, \sigma=0.2)$, except for the diagonal elements of $A_t$ which were set to 1. The linear parameter for the outcome model $\beta$ was sampled from the same distribution, i.e. $\beta_j\sim \cN(\mu=0, \sigma=0.2)$ for $j=1,\dots,d$.

As mentioned, the eigenvalues of $A_{t}$ influence the system's behavior and stability, hence we enforce the spectral radius $\forall t: \rho(A_{t})=\kappa=1.5$ for all $t$ for the experiments in Section~\ref{sec:synthetic}. This is by factorizing $A_t$ into its spectral decomposition $U_t\Lambda_t U_t^{-1}$ and computing $\Lambda_{t}^{(new)}= \frac{\kappa}{\rho(A_t)} \Lambda_t$. Then, an update $A_{t}=U_t\Lambda_t^{(new)} U_t^{-1}$ is performed where $\kappa$ becomes the new spectral radius of $A_t$.

For all experiments, we use the following default values unless otherwise stated: $\kappa=1.5$, $n=1000$, $T=10$, $d=25$, and $\V(\epsilon_t)=\V(\epsilon_Y)=1$ for $t=1,\dots,T-1$. Finally, the input distribution is $p(X_1)=\cN(\mu=0, \sigma^2=5)$.

\subsubsection*{Additional Experiments: Testing Stationary Systems}
For the experiments in Section~\ref{sec:synthetic}, we solely consider synthetic systems with time-dependent transition $A_{t}$. We also include an experiment for systems where the transitions are stationary, that is $A_{t} = A_{t'}, t, t' =1,\dots,T-1$. The generation process is identical with the exception that only one transition matrix is sampled, $A$. In the case with more than one privileged time point, this enables us to evaluate the potential benefits that using Stat-LuPTS (instead of the non-stationary variant) have in a setting where the assumptions holds true.

When the stationary assumptions is true (Figure~\ref{fig:stat_true}), LuPTS does better than baseline, as before. More importantly, Stat-LuPTS is closer to the true parameter estimate than both of them. Meanwhile, as expected, when breaking the stationary assumption (Figure~\ref{fig:stat_false}), Stat-LuPTS performs significantly worse while LuPTS and baseline remain about the same. These experiments indicate that the stationary variant of LuPTS is preferable when the stationarity assumption is true. 

\begin{figure}[t]
    \centering
    \begin{subfigure}[t]{0.35\textwidth}
        \centering
        \includegraphics[width=\textwidth]{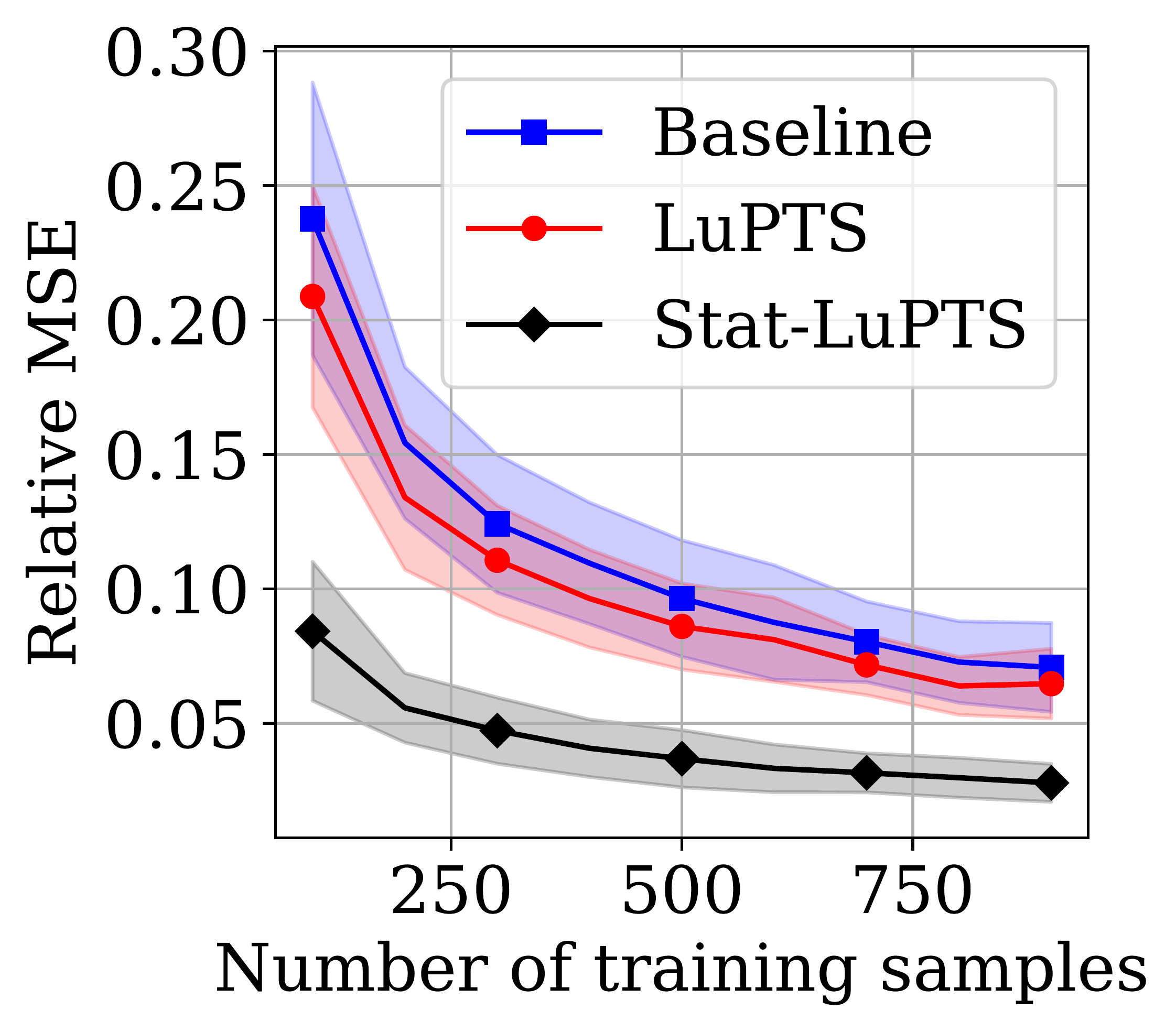}
        \caption{Stationarity is true.}
        \label{fig:stat_true}
    \end{subfigure}%
    ~
    \begin{subfigure}[t]{0.31\textwidth}
        \centering
        \includegraphics[width=\textwidth]{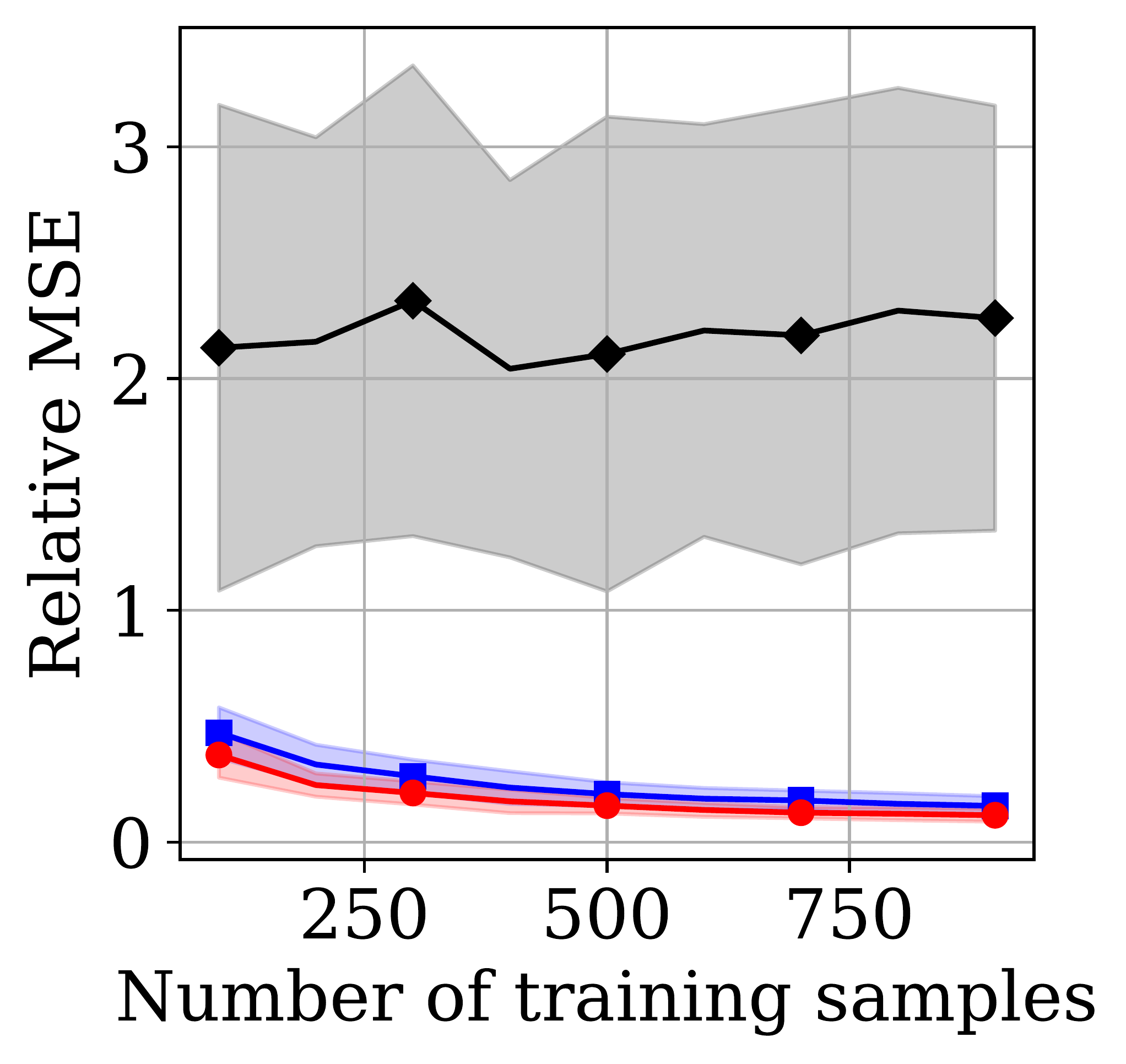}
        \caption{Stationarity is violated.}
        \label{fig:stat_false}
    \end{subfigure}%
    \caption{Parameter recovery with or without stationarity when varying the number of training samples $n$. $R^2$ used as metric; shaded region corresponds to one standard deviation over 200 iterations.}
    \label{fig:parameter_recovery_stationarity}
\end{figure}

\subsection{Forecasting Air Quality}
\paragraph*{Implementation Of Distillation Methods} We implement the distillation models and loss function, as described in Section \ref{sec:distill}, in PyTorch v1.7. The loss function is optimized using Adam \citep{kingma2014adam}, and the models are trained for 200 epochs. Error bars on all our plots are generated by training and evaluating on different train/test splits over 100 iterations.

\paragraph*{Pre-processing} Due to the prevalence of missing values for the PM$_{2.5}$ concentration levels in the dataset, the first pre-processing step is to extract all non-overlapping sequences of length $T$ that have no missing values for the PM$_{2.5}$ concentration. In addition, we enforce a rule that there must be at least a gap of six hours between adjacent sequences to decrease correlations between them. Finally, dummy encoding was used for the categorical features in the dataset.

\begin{table}[ht]
    \centering
    \caption{Features in the PM$_{2.5}$ dataset.}
    \label{tab:features_pm25}
    \begin{tabular}{ll}
        \toprule
        Feature & Type  \\
        \midrule
        Temperature & Numerical  \\
        Humidity & Numerial \\
        Dew Point & Numerical \\
        Pressure & Numerical \\
        Cumulated wind speed & Numerical \\
        PM2.5 concentration & Numerical \\
        Season & Categorical \\
        Combined wind direction & Categorical \\
        \bottomrule
    \end{tabular}
\end{table}

\paragraph*{Evaluation}
During training, the dataset was split into a training and test set portion consisting of 80\% and 20\% of the data respectively. The training procedure on the dataset for the forecasting task is the following: We vary the number of training samples, and for each sample size, data points are randomly sampled from the training set without replacement. Then, before training the algorithms on this set, we apply zero-mean unit-variance standardization and mean imputation where applicable. Each algorithm is then evaluated after training on a held-out test set, which is the same for every run, and the corresponding $R^2$ score is noted. This process is iterated 200 times per sample size.

\subsubsection*{Additional Experiments}
In this section, we show additional experiments performed on the air quality forecasting task.

\paragraph{More Cities}
We present the same results as shown in Section~\ref{sec:cities} for all of the Chinese cities in the dataset; Shenyang, Beijing, Chengdu, Shanghai and Guangzhou. These are shown in Figure~\ref{fig:pm25_shenyang_app}, \ref{fig:pm25_beijing_app}, \ref{fig:pm25_chengdu_app}, \ref{fig:pm25_shanghai_app} and \ref{fig:pm25_guangzhou_app}, respectively.

\paragraph{Comparison To Non-linear Baselines}

In addition to the distillation-based baselines in the main paper, we compare LuPTS to non-linear baselines in the form of random forest (RF) and k-nearest neighbors regression (KNN). These can be found in Table~\ref{tab:fc_nonlinear_n200_6hour} and Table~\ref{tab:fc_nonlinear_n200_12hour} where we have a fixed sample size $n=200$ and a prediction horizon of either 6 or 12 hours. First, we describe the implementation details for RF and KNN regression. We use the RandomForestRegressor and KNeighborsRegressor implementations of the Python module scikit-learn~\cite{scikit-learn}. The model parameters are tuned using randomized search with 2-fold cross validation. For the RandomForestRegressor we tune the number of trees, max depth of the trees, whether bootstrap sampling is used, the minimum number of samples required to split a node and the minimum number of samples required to be at a leaf node. For the KNeighborsRegressor we tune the number of neighbors used, weight function used in prediction, the size of the leaves and the power parameter for the Minkowski metric ($p=1$ or $p=2$). The specific ranges for each parameter can be found in the attached code to this paper. 

For the 6 hour predictions (see Table~\ref{tab:fc_nonlinear_n200_6hour}), we see that all linear methods (Baseline, LuPTS variants and distillation-based variants) perform better than both RF and KNN for all cities, although the gap between LuPTS and RF is relatively small for Beijing and Guangzhou in particular. For the 12 hour predictions (see Table~\ref{tab:fc_nonlinear_n200_12hour}), the results look similar except for Guangzhou where RF performs the best among all the methods. 
A possible explanation for why the non-linear methods in almost all cases perform worse could be due to the low-sample regime which is not as suitable for their larger flexibility. In particular, a benefit of the linear methods is that they either do not need model parameter tuning at all or to a smaller degree in comparison to the non-linear methods. 

\begin{figure}[t!]
    \centering
    \begin{subfigure}[t]{0.22\textwidth}
        \centering
        \includegraphics[width=\textwidth]{fig_png/chinaPM25/experiment_shenyang_varyPTS_T5-1.png}
        \caption{6 hour forecast}
        \label{fig:shenyangT6_app}
    \end{subfigure}%
    ~
    \begin{subfigure}[t]{0.22\textwidth}
        \centering
        \includegraphics[width=\textwidth]{fig_png/chinaPM25/experiment_shenyang_varyPTS_T23-1.png}
        \caption{24 hour forecast}
        \label{fig:shenyangT24_app}
    \end{subfigure}%
    ~
    \begin{subfigure}[t]{0.22\textwidth}
        \centering
        \includegraphics[width=\textwidth]{fig_png/chinaPM25/experiment_shenyang_ts1_T5-1.png}
        \caption{6 hour forecast}
        \label{fig:shenyangT6_distill_app}
    \end{subfigure}%
    ~
    \begin{subfigure}[t]{0.22\textwidth}
        \centering
        \includegraphics[width=\textwidth]{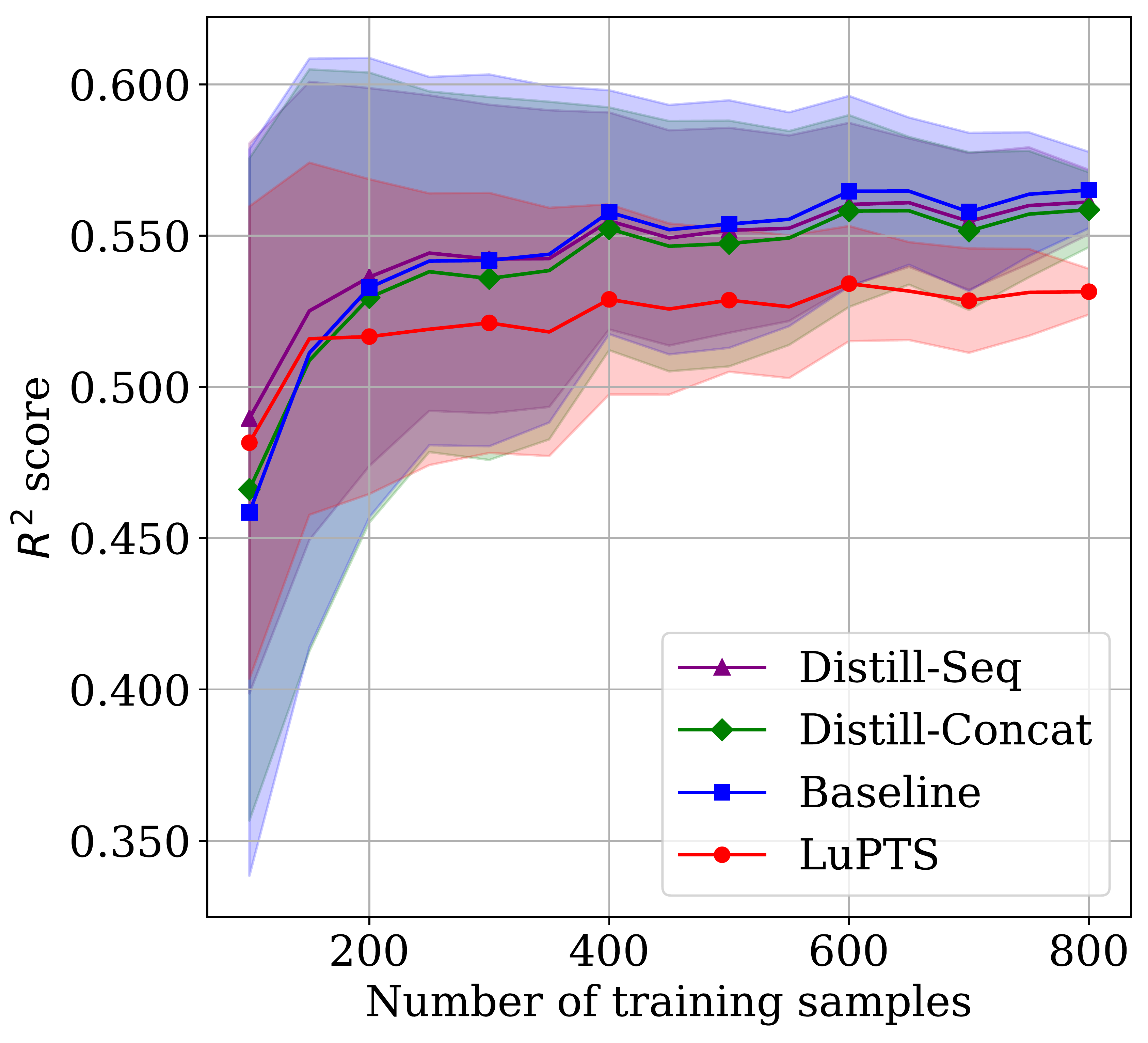}
        \caption{12 hour forecast}
        \label{fig:shenyangT12_distill_app}
    \end{subfigure}
    \caption{\textbf{Shenyang: }\ref{fig:shenyangT6_app}, \ref{fig:shenyangT24_app}) Changing the amount of privileged information for the LUPTS for different time horizons, where the \textit{X} in LuPTS\_\textit{X}PTS indicates the number of privileged time points. \ref{fig:shenyangT6_distill_app}, \ref{fig:shenyangT12_distill_app}) Comparing LuPTS to the distillation-based approaches, which use the same privileged information.
    Metric used is $R^2$ (Higher is better); shaded region indicates one standard deviation across 75 iterations.}
    \label{fig:pm25_shenyang_app}
\end{figure}

\begin{figure}[t!]
    \centering
    \begin{subfigure}[t]{0.22\textwidth}
        \centering
        \includegraphics[width=\textwidth]{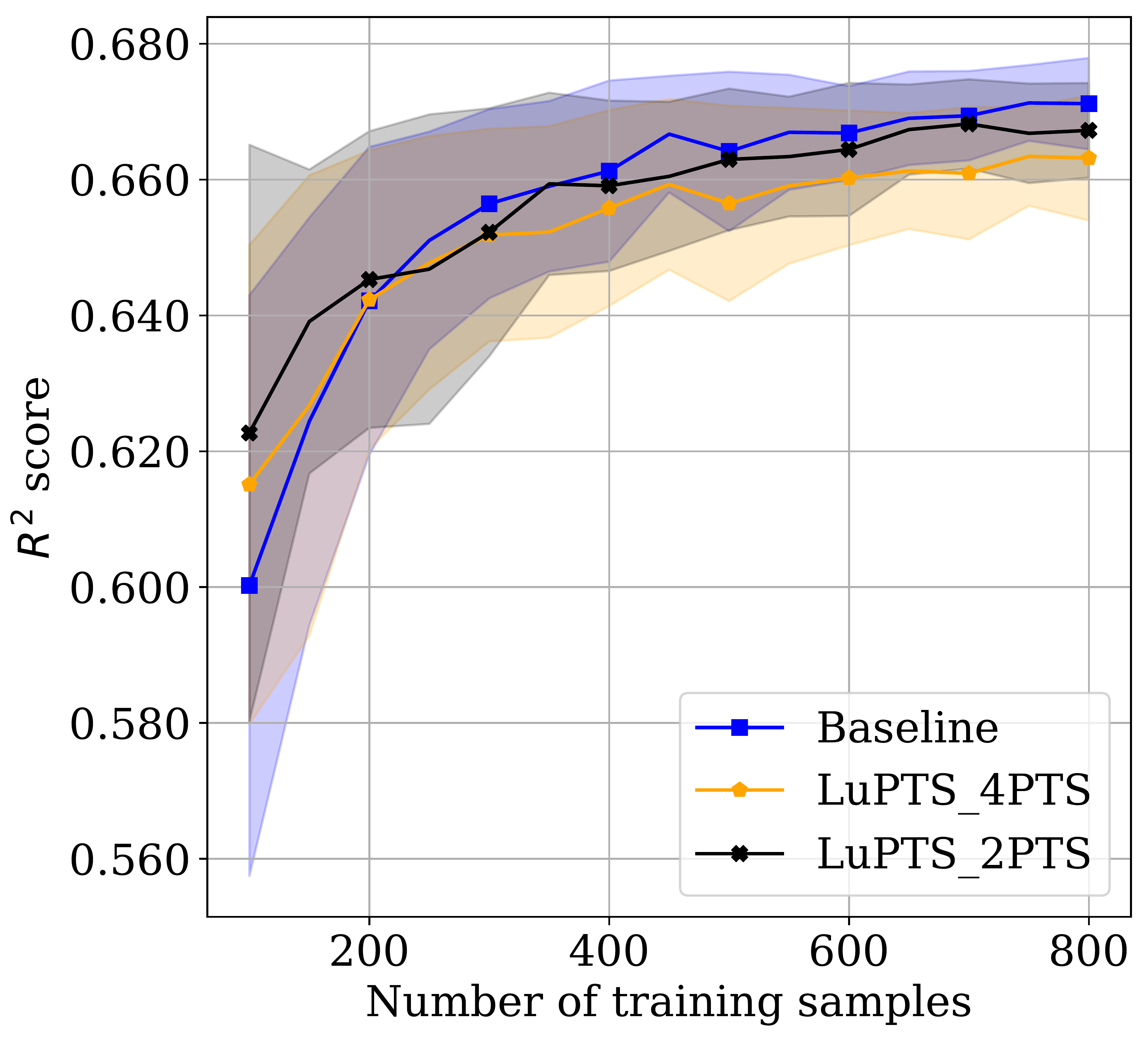}
        \caption{6 hour forecast}
        \label{fig:beijingT6_app}
    \end{subfigure}%
    ~
    \begin{subfigure}[t]{0.22\textwidth}
        \centering
        \includegraphics[width=\textwidth]{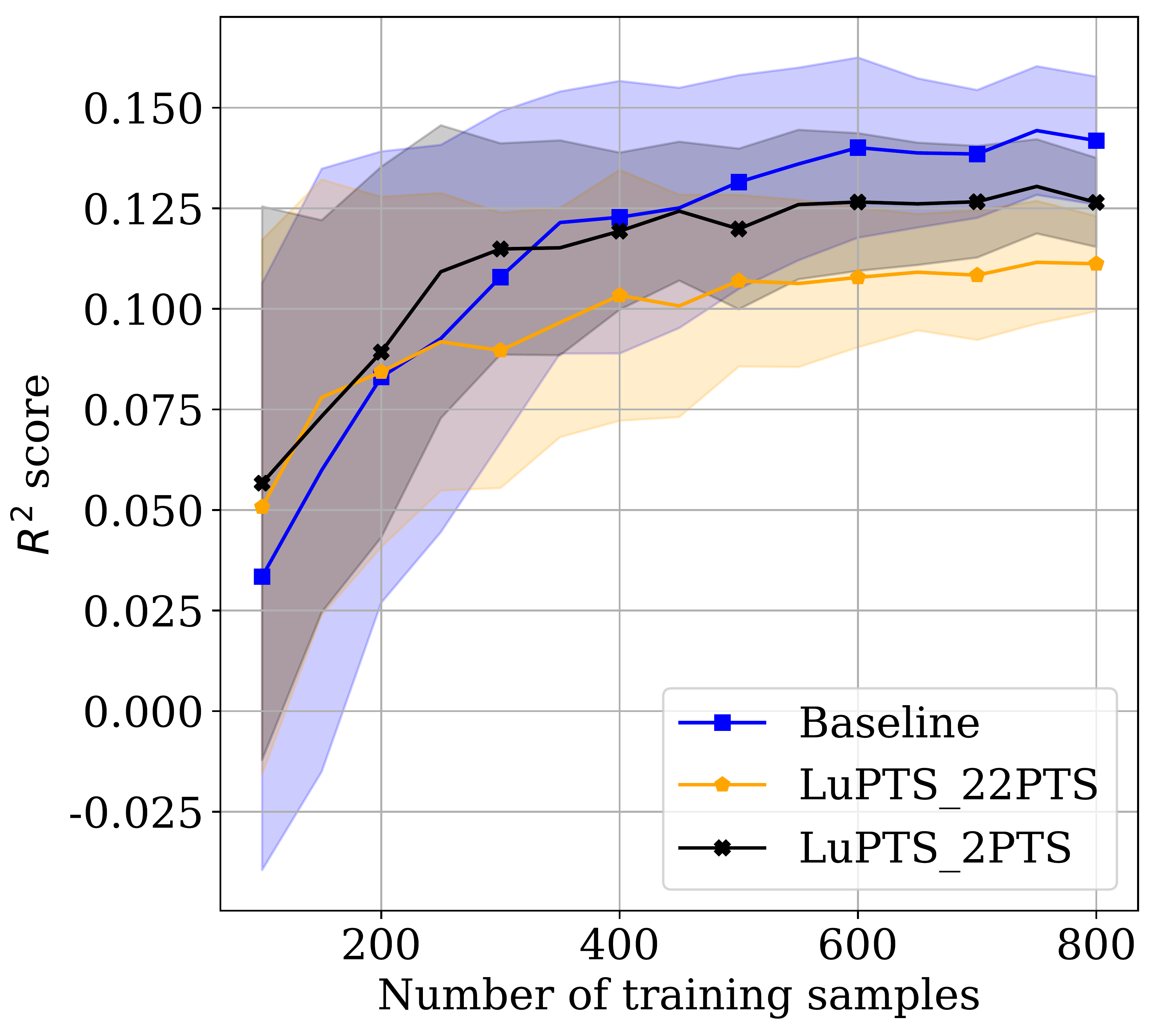}
        \caption{24 hour forecast}
        \label{fig:beijingT24_app}
    \end{subfigure}%
    ~
    \begin{subfigure}[t]{0.22\textwidth}
        \centering
        \includegraphics[width=\textwidth]{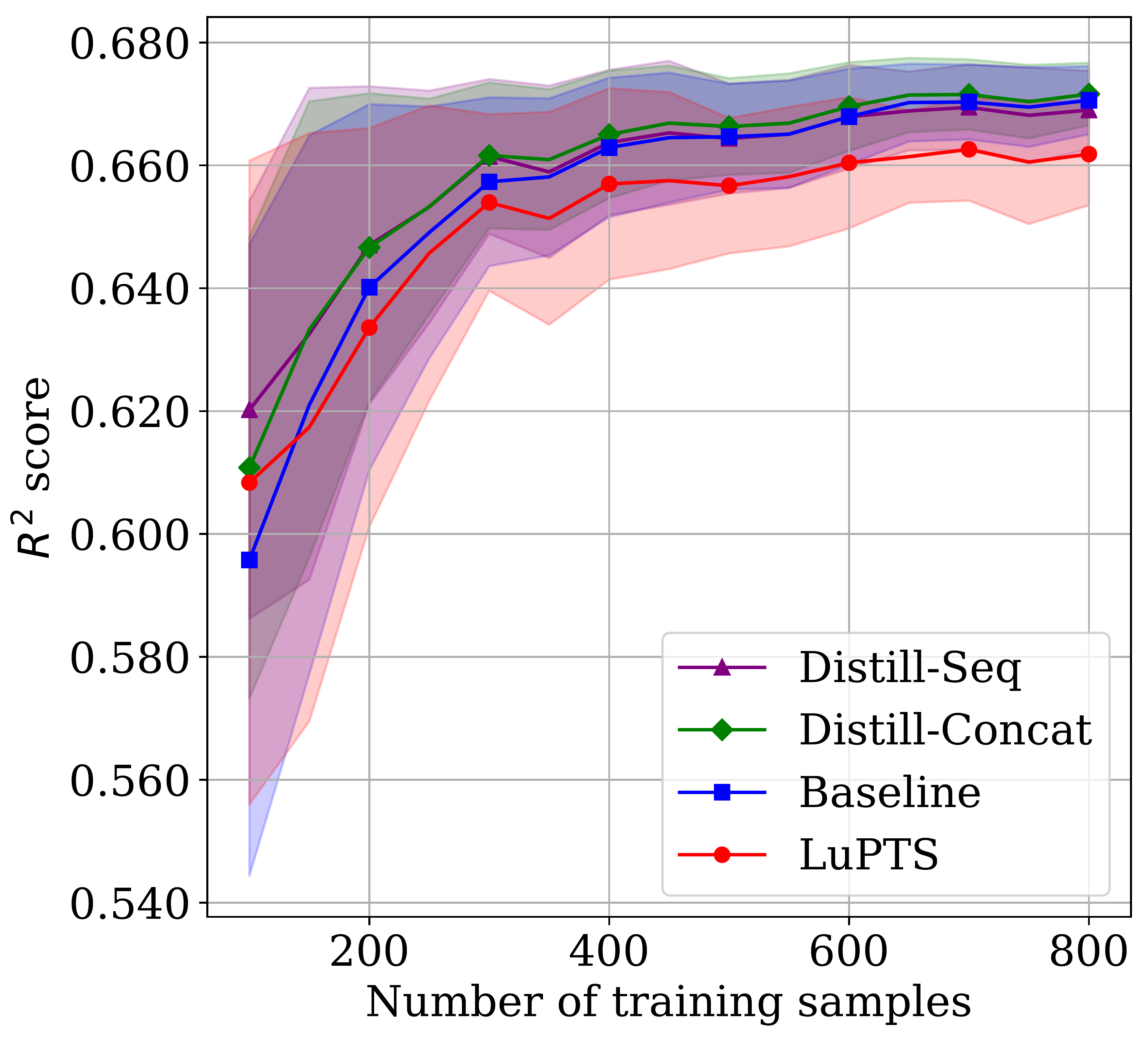}
        \caption{6 hour forecast}
        \label{fig:beijingT6_distill_app}
    \end{subfigure}%
    ~
    \begin{subfigure}[t]{0.22\textwidth}
        \centering
        \includegraphics[width=\textwidth]{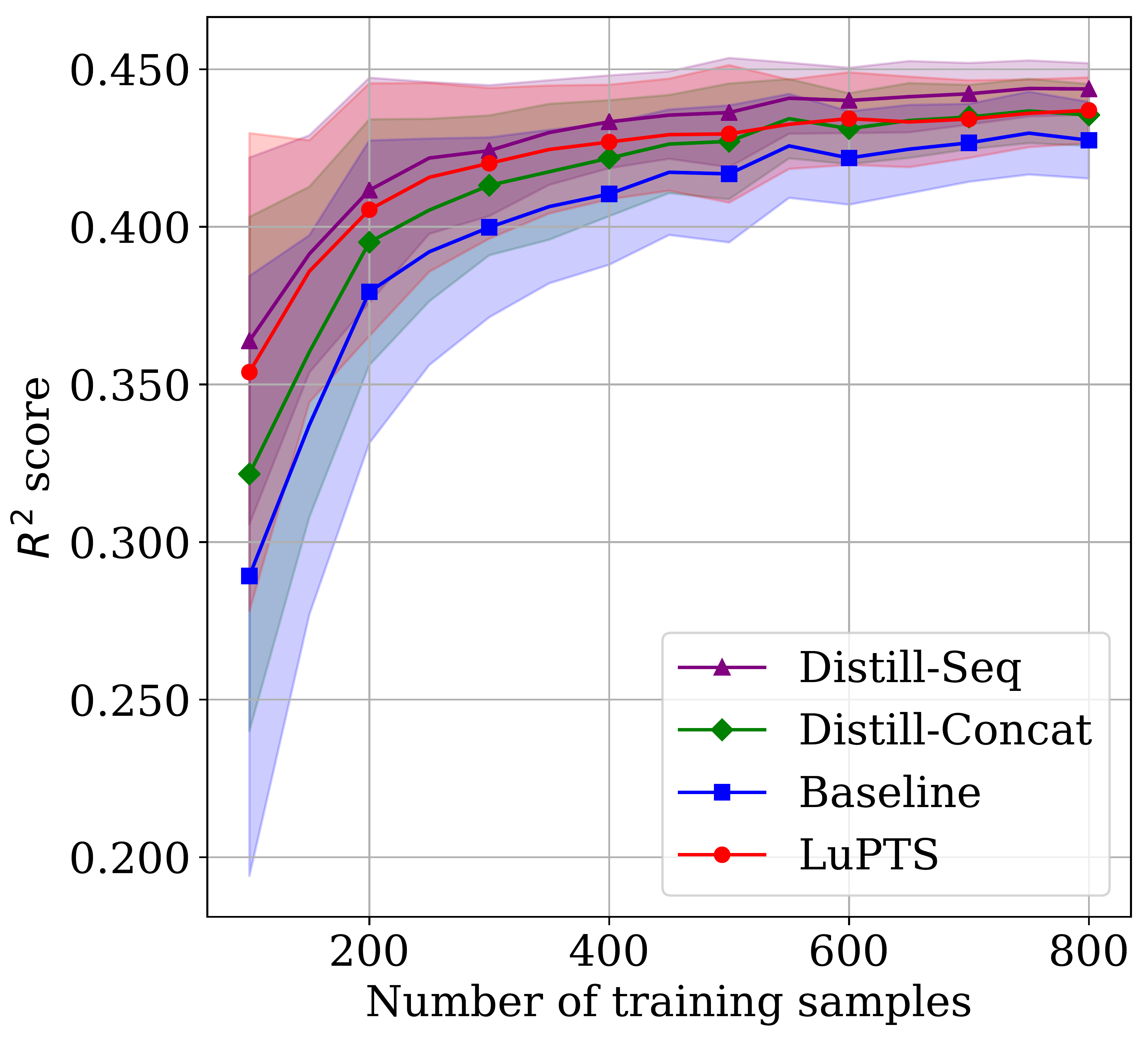}
        \caption{12 hour forecast}
        \label{fig:beijingT12_distill_app}
    \end{subfigure}
    \caption{\textbf{Beijing: }\ref{fig:beijingT6_app}, \ref{fig:beijingT24_app}) Changing the amount of privileged information for the LUPTS for different time horizons, where the \textit{X} in LuPTS\_\textit{X}PTS indicates the number of privileged time points. \ref{fig:beijingT6_distill_app}, \ref{fig:beijingT12_distill_app}) Comparing LuPTS to the distillation-based approaches, which use the same privileged information.
    Metric used is $R^2$ (Higher is better); shaded region indicates one standard deviation across 75 iterations.}
    \label{fig:pm25_beijing_app}
\end{figure}

\begin{figure}[t!]
    \centering
    \begin{subfigure}[t]{0.22\textwidth}
        \centering
        \includegraphics[width=\textwidth]{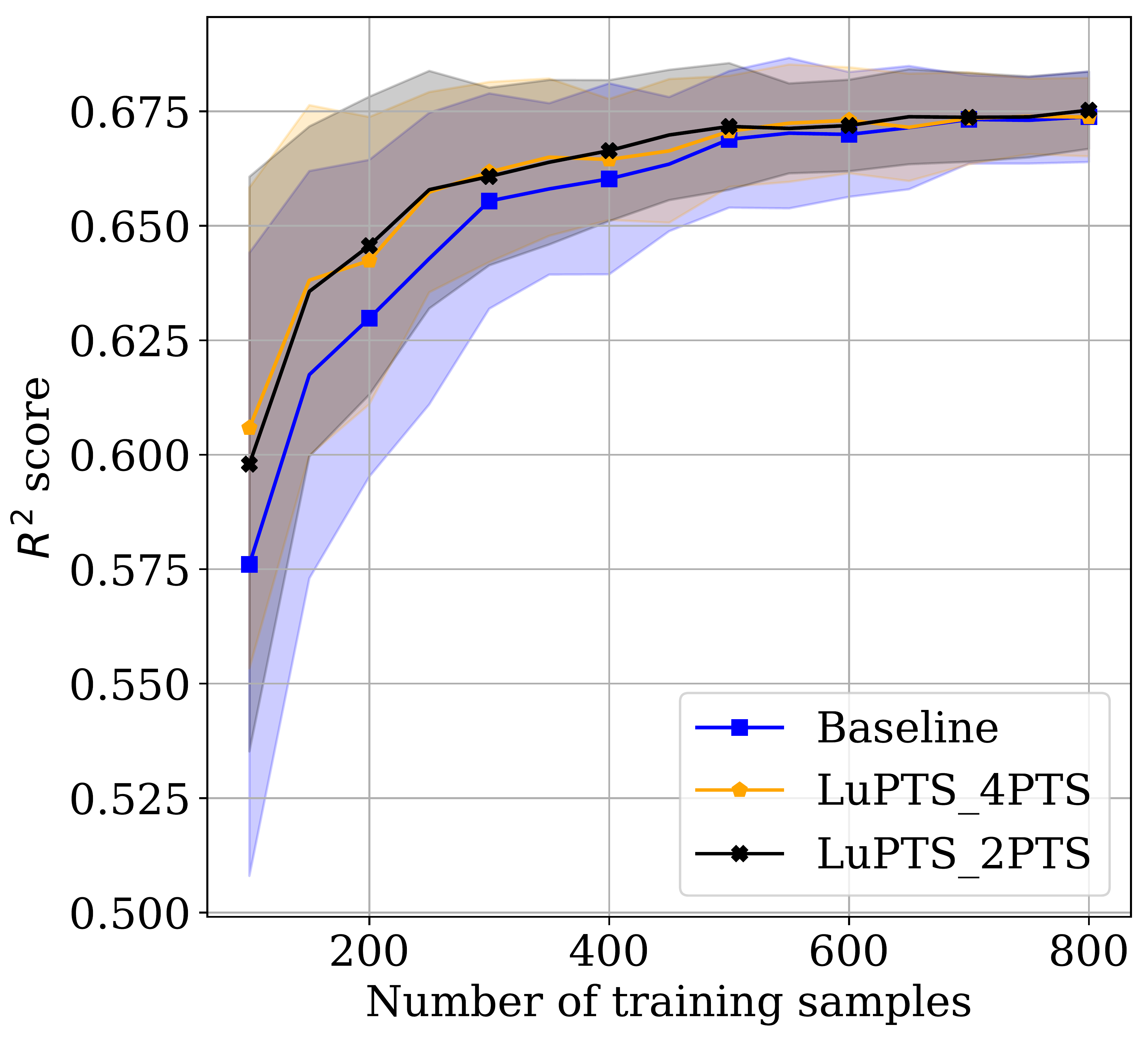}
        \caption{6 hour forecast}
        \label{fig:chengduT6_app}
    \end{subfigure}%
    ~
    \begin{subfigure}[t]{0.22\textwidth}
        \centering
        \includegraphics[width=\textwidth]{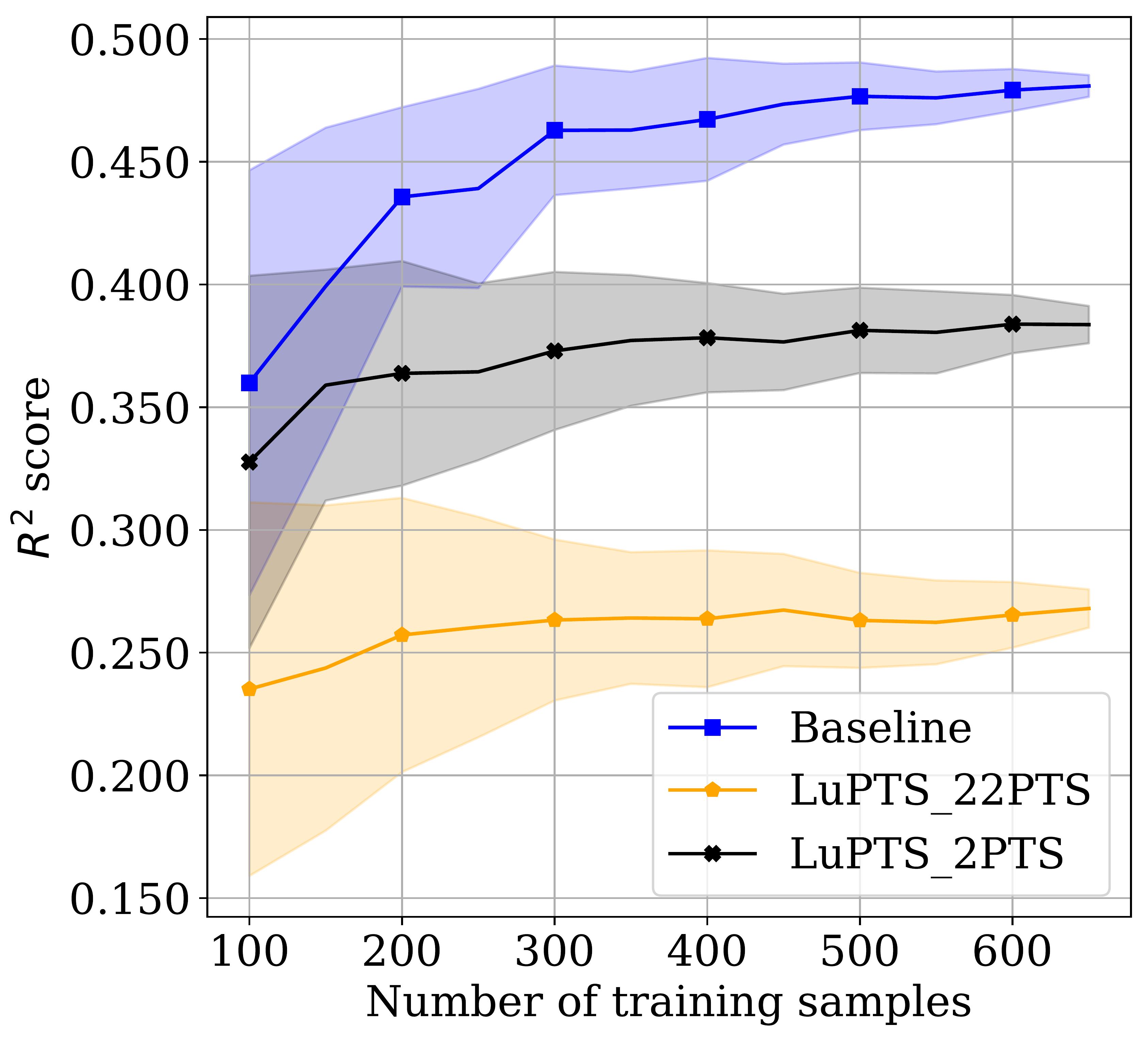}
        \caption{24 hour forecast}
        \label{fig:chengduT24_app}
    \end{subfigure}%
    ~
    \begin{subfigure}[t]{0.22\textwidth}
        \centering
        \includegraphics[width=\textwidth]{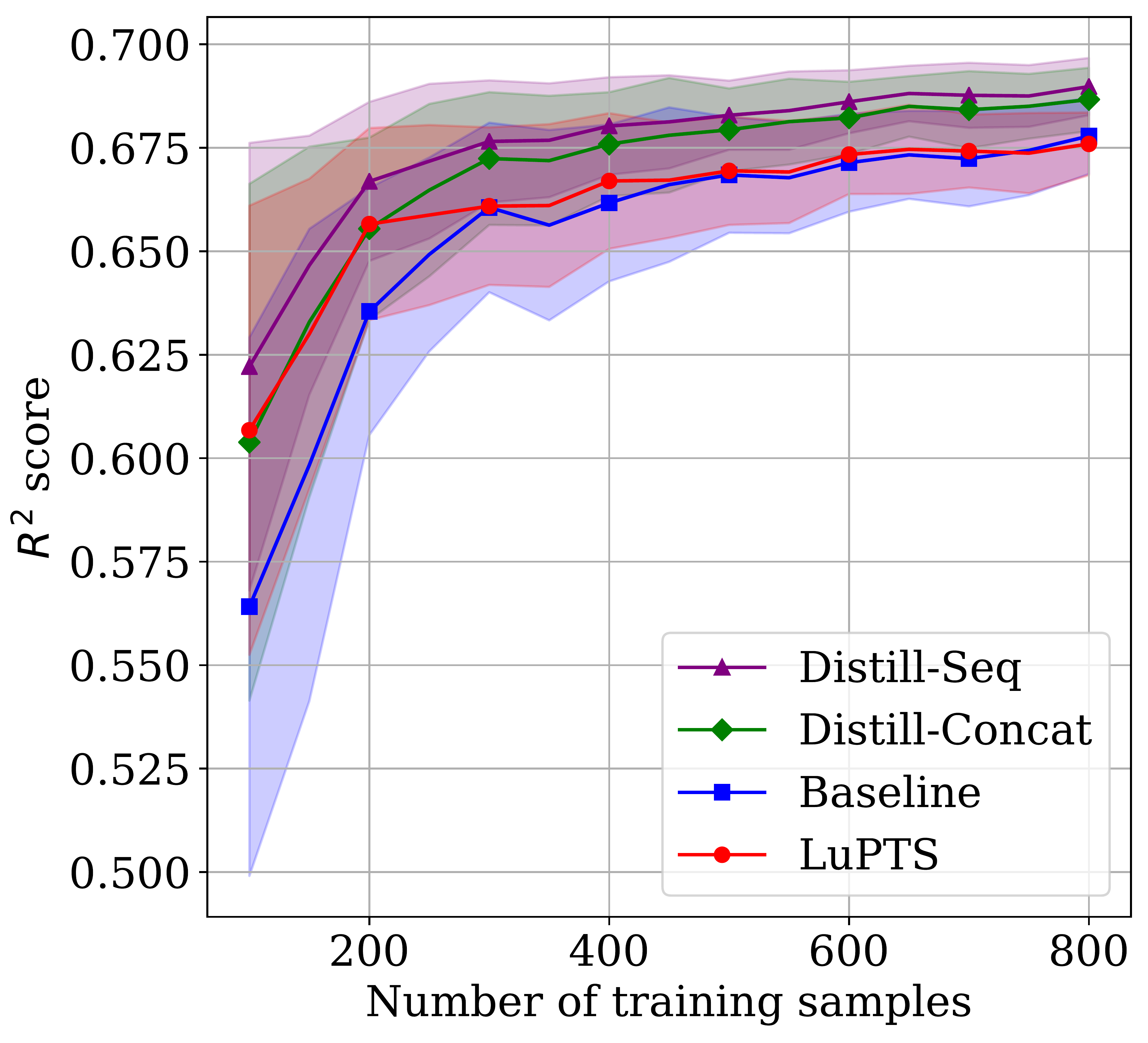}
        \caption{6 hour forecast}
        \label{fig:chengduT6_distill_app}
    \end{subfigure}%
    ~
    \begin{subfigure}[t]{0.22\textwidth}
        \centering
        \includegraphics[width=\textwidth]{fig_png/chinaPM25/experiment_chengdu_ts1_T11-1.png}
        \caption{12 hour forecast}
        \label{fig:chengduT12_distill_app}
    \end{subfigure}
    \caption{\textbf{Chengdu: }\ref{fig:chengduT6_app}, \ref{fig:chengduT24_app}) Changing the amount of privileged information for the LUPTS for different time horizons, where the \textit{X} in LuPTS\_\textit{X}PTS indicates the number of privileged time points. \ref{fig:chengduT6_distill_app}, \ref{fig:chengduT12_distill_app}) Comparing LuPTS to the distillation-based approaches, which use the same privileged information.
    Metric used is $R^2$ (Higher is better); shaded region indicates one standard deviation across 75 iterations.}
    \label{fig:pm25_chengdu_app}
\end{figure}

\begin{figure}[t!]
    \centering
    \begin{subfigure}[t]{0.22\textwidth}
        \centering
        \includegraphics[width=\textwidth]{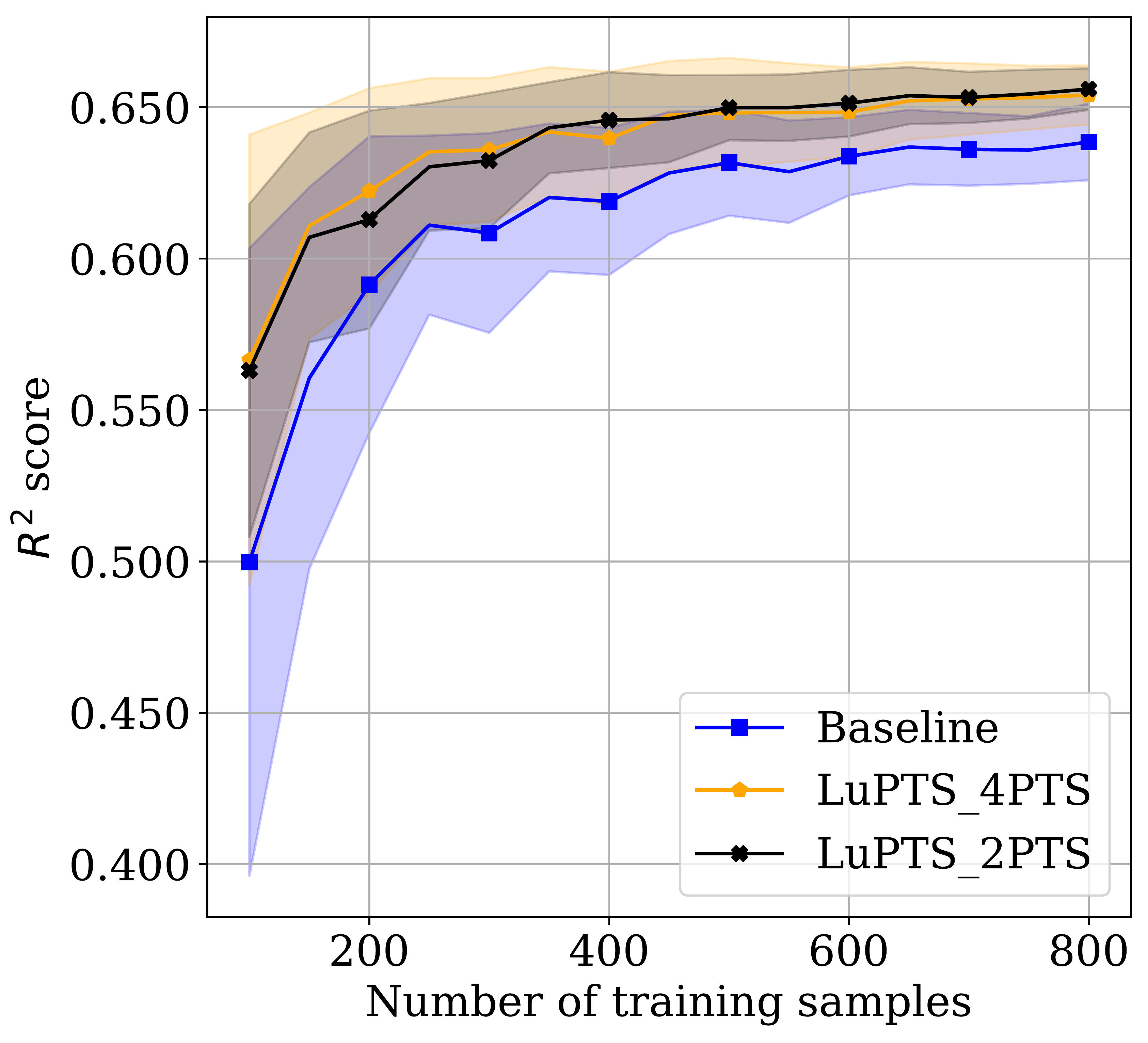}
        \caption{6 hour forecast}
        \label{fig:shanghaiT6_app}
    \end{subfigure}%
    ~
    \begin{subfigure}[t]{0.22\textwidth}
        \centering
        \includegraphics[width=\textwidth]{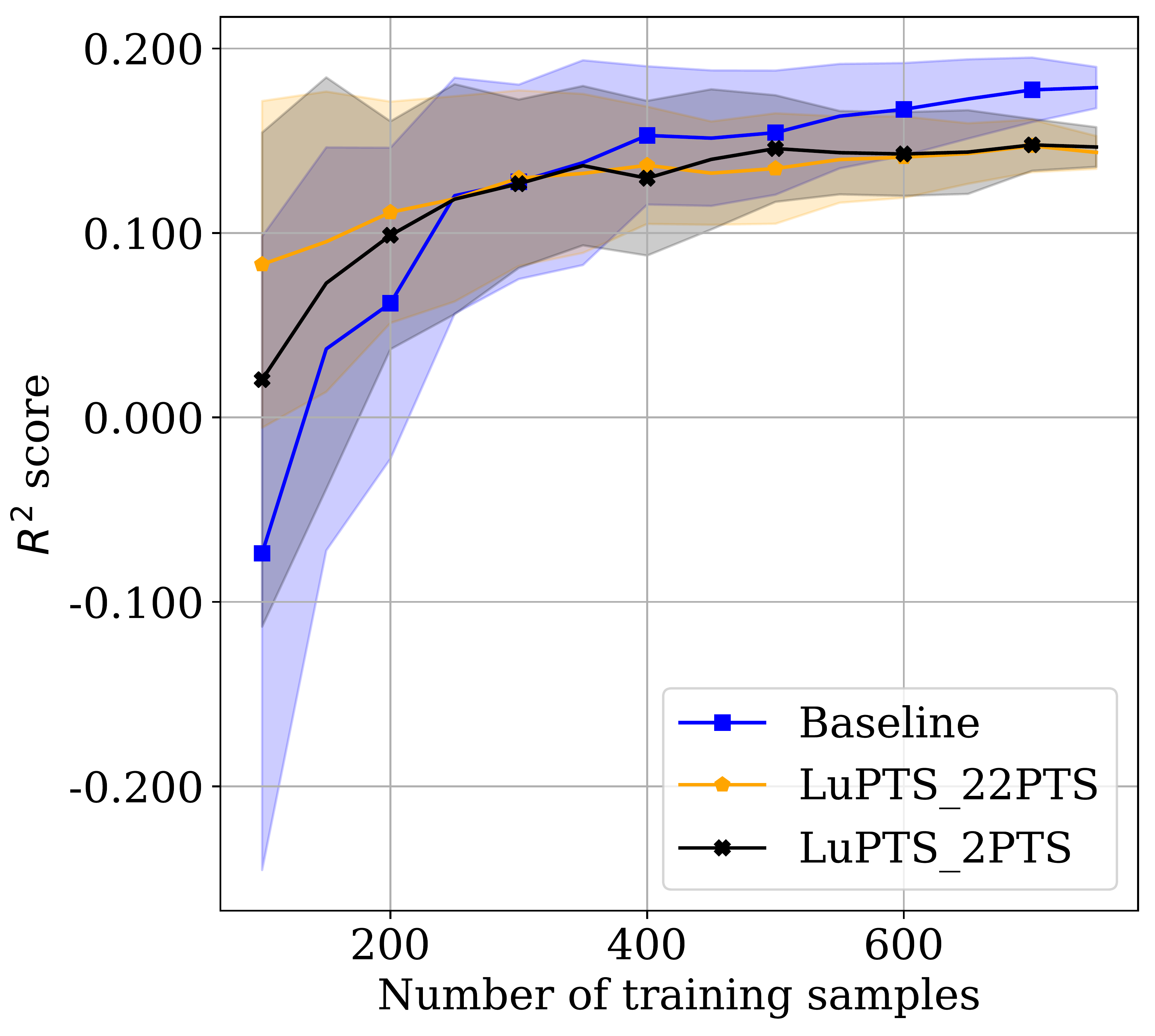}
        \caption{24 hour forecast}
        \label{fig:shanghaiT24_app}
    \end{subfigure}%
    ~
    \begin{subfigure}[t]{0.22\textwidth}
        \centering
        \includegraphics[width=\textwidth]{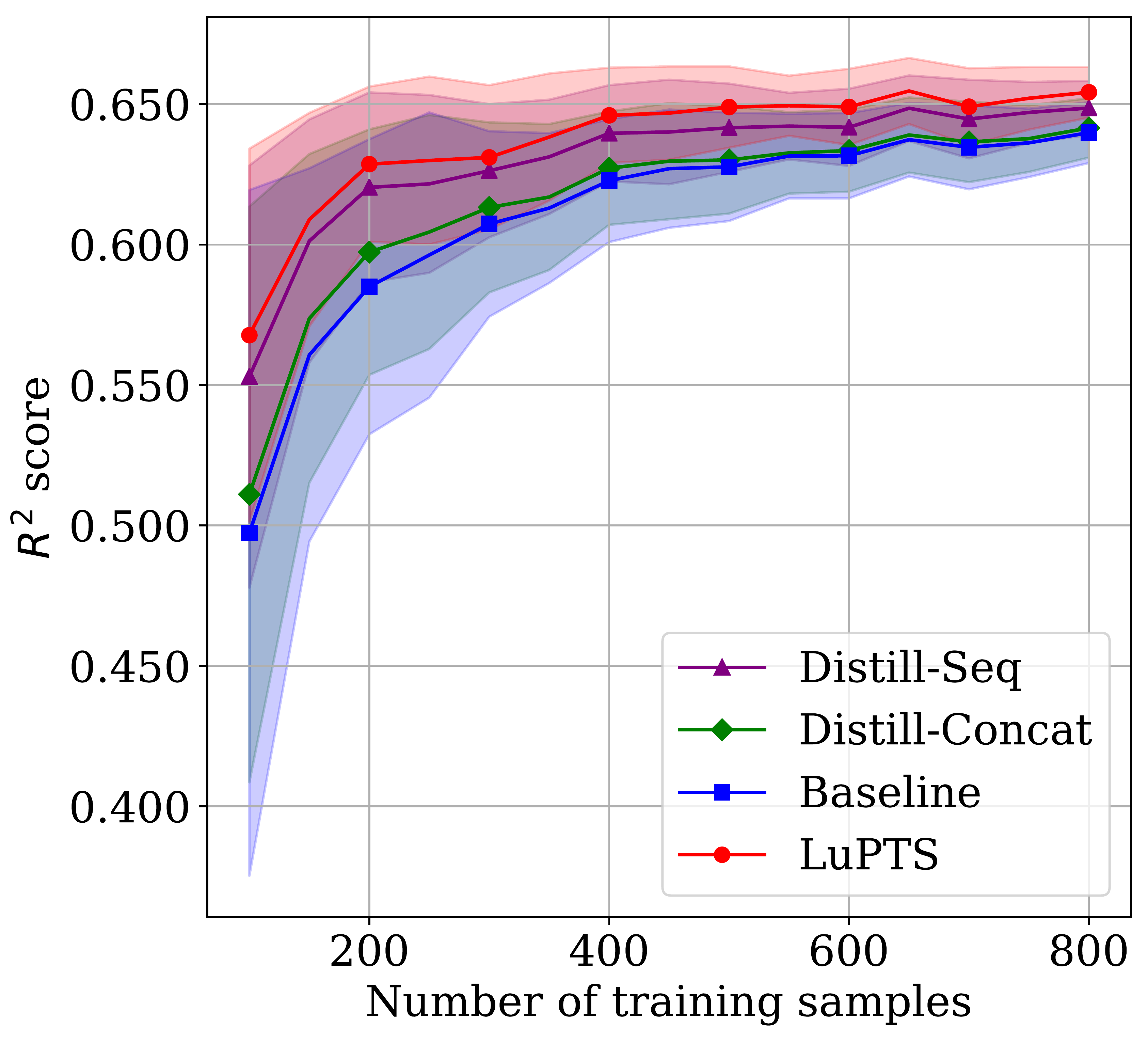}
        \caption{6 hour forecast}
        \label{fig:shanghaiT6_distill_app}
    \end{subfigure}%
    ~
    \begin{subfigure}[t]{0.22\textwidth}
        \centering
        \includegraphics[width=\textwidth]{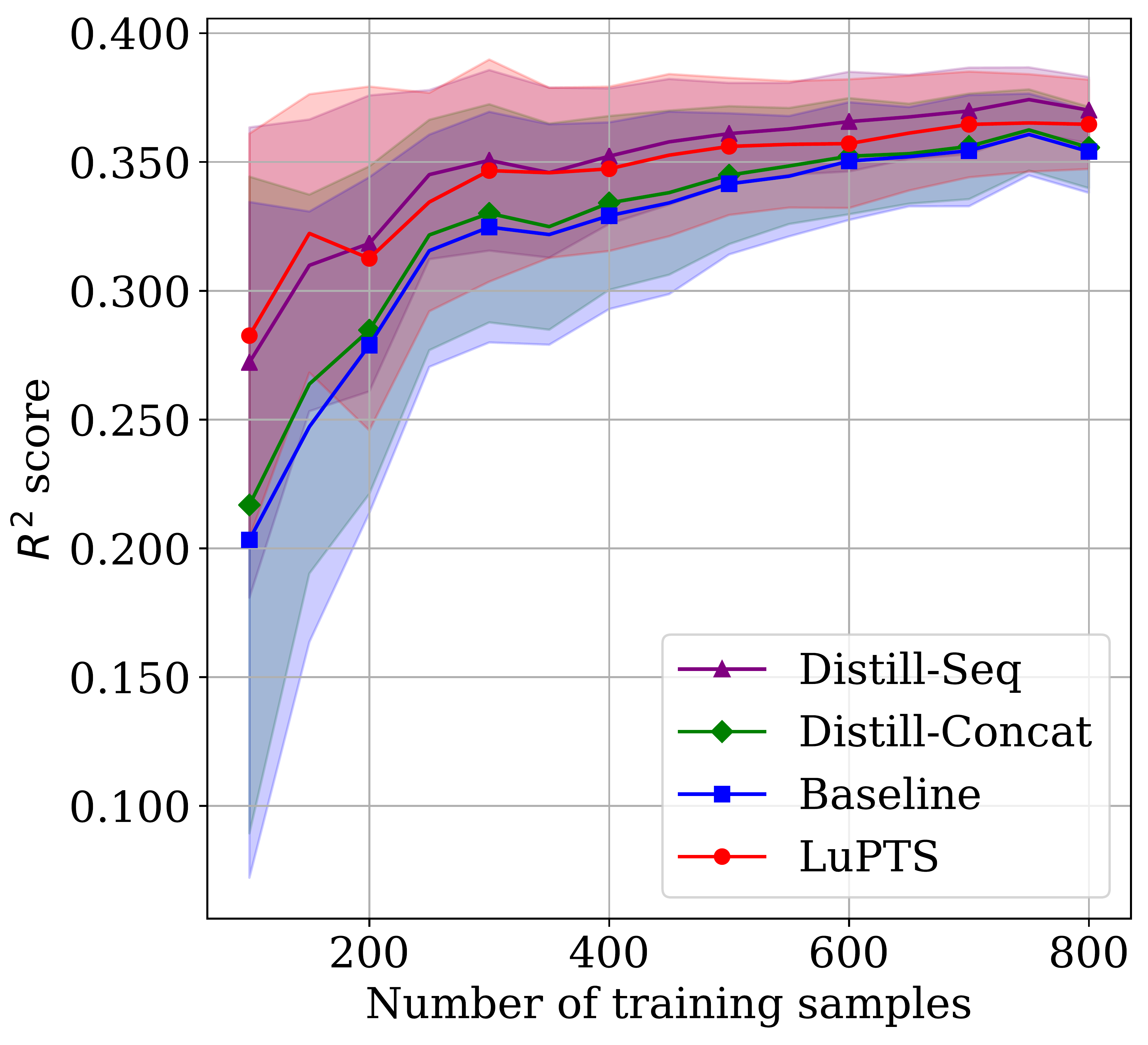}
        \caption{12 hour forecast}
        \label{fig:shanghaiT12_distill_app}
    \end{subfigure}
    \caption{\textbf{Shanghai: }\ref{fig:shanghaiT6_app}, \ref{fig:shanghaiT24_app}) Changing the amount of privileged information for the LUPTS for different time horizons, where the \textit{X} in LuPTS\_\textit{X}PTS indicates the number of privileged time points. \ref{fig:shanghaiT6_distill_app}, \ref{fig:shanghaiT12_distill_app}) Comparing LuPTS to the distillation-based approaches, which use the same privileged information.
    Metric used is $R^2$ (Higher is better); shaded region indicates one standard deviation across 75 iterations.}
    \label{fig:pm25_shanghai_app}
\end{figure}

\begin{figure}[t!]
    \centering
    \begin{subfigure}[t]{0.22\textwidth}
        \centering
        \includegraphics[width=\textwidth]{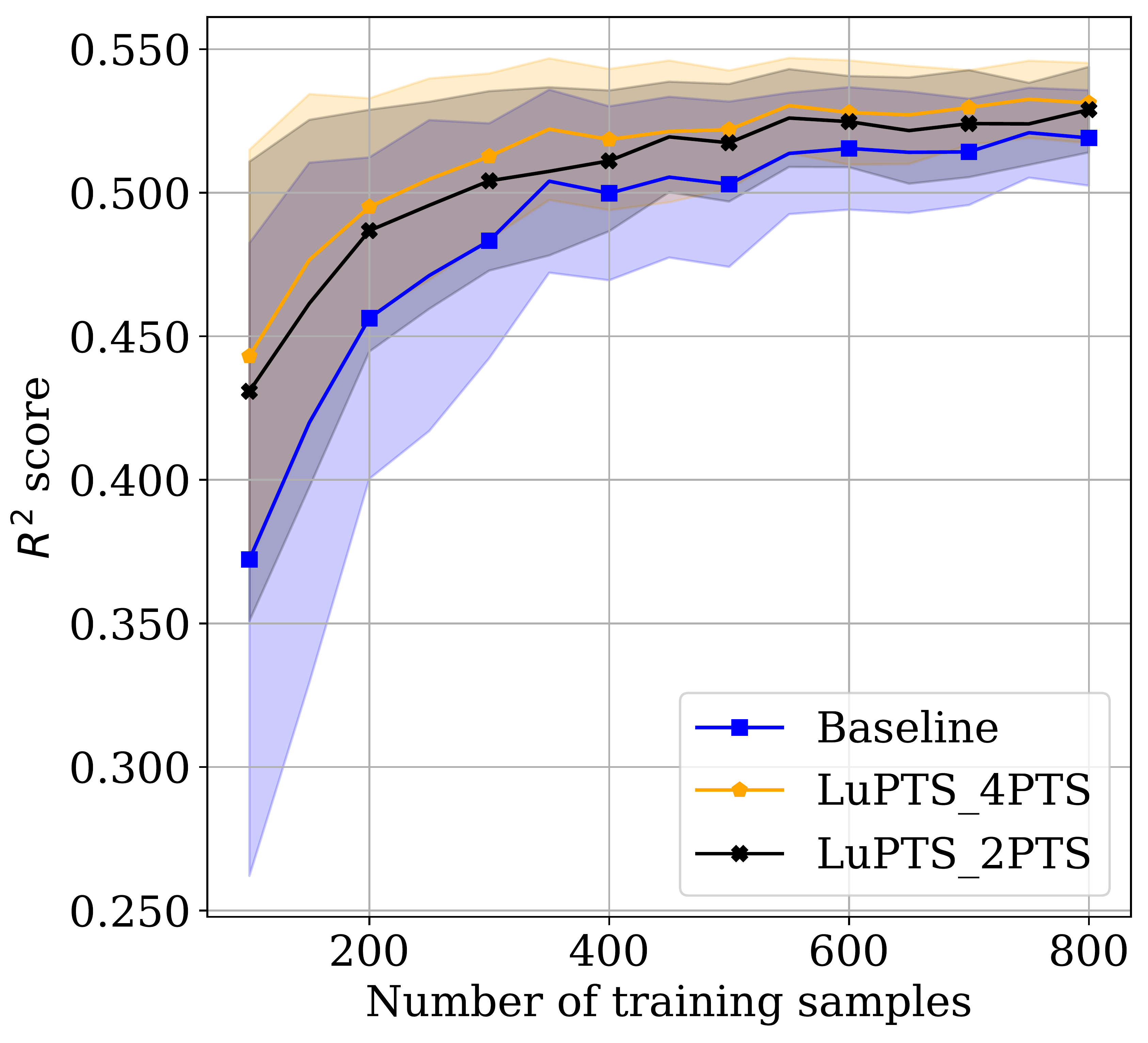}
        \caption{6 hour forecast}
        \label{fig:guangzhouT6_app}
    \end{subfigure}%
    ~
    \begin{subfigure}[t]{0.22\textwidth}
        \centering
        \includegraphics[width=\textwidth]{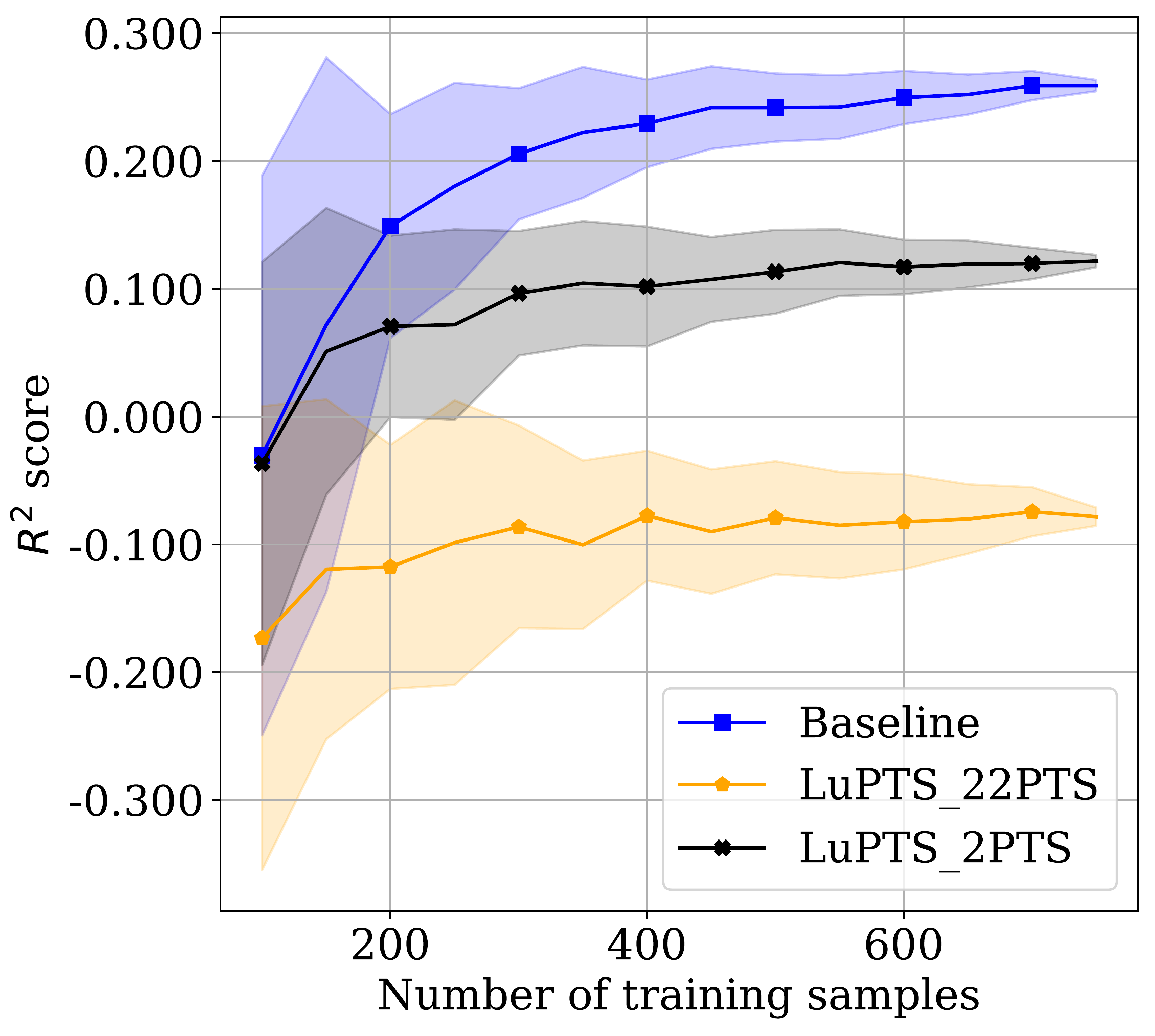}
        \caption{24 hour forecast}
        \label{fig:guangzhouT24_app}
    \end{subfigure}%
    ~
    \begin{subfigure}[t]{0.22\textwidth}
        \centering
        \includegraphics[width=\textwidth]{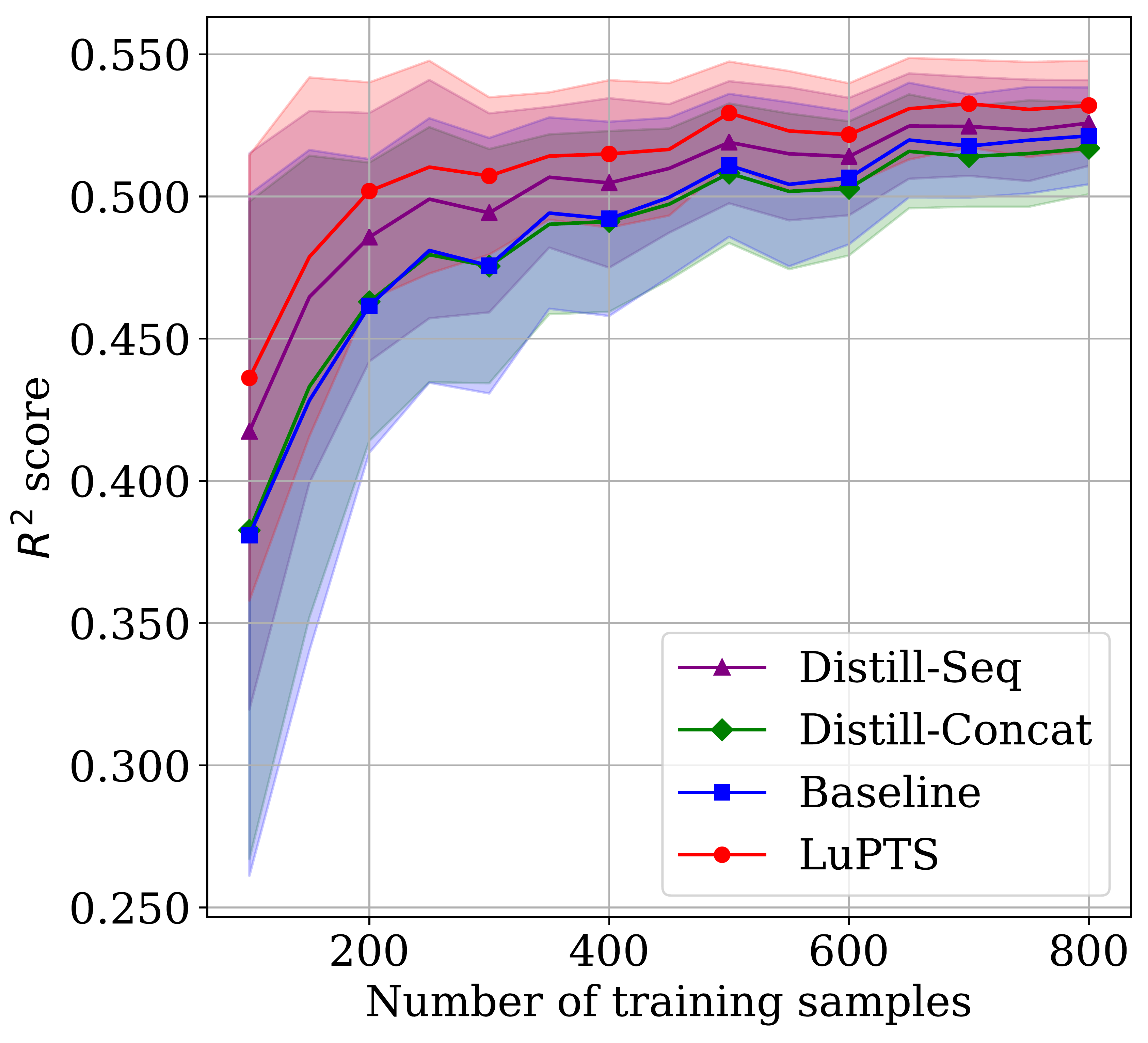}
        \caption{6 hour forecast}
        \label{fig:guangzhouT6_distill_app}
    \end{subfigure}%
    ~
    \begin{subfigure}[t]{0.22\textwidth}
        \centering
        \includegraphics[width=\textwidth]{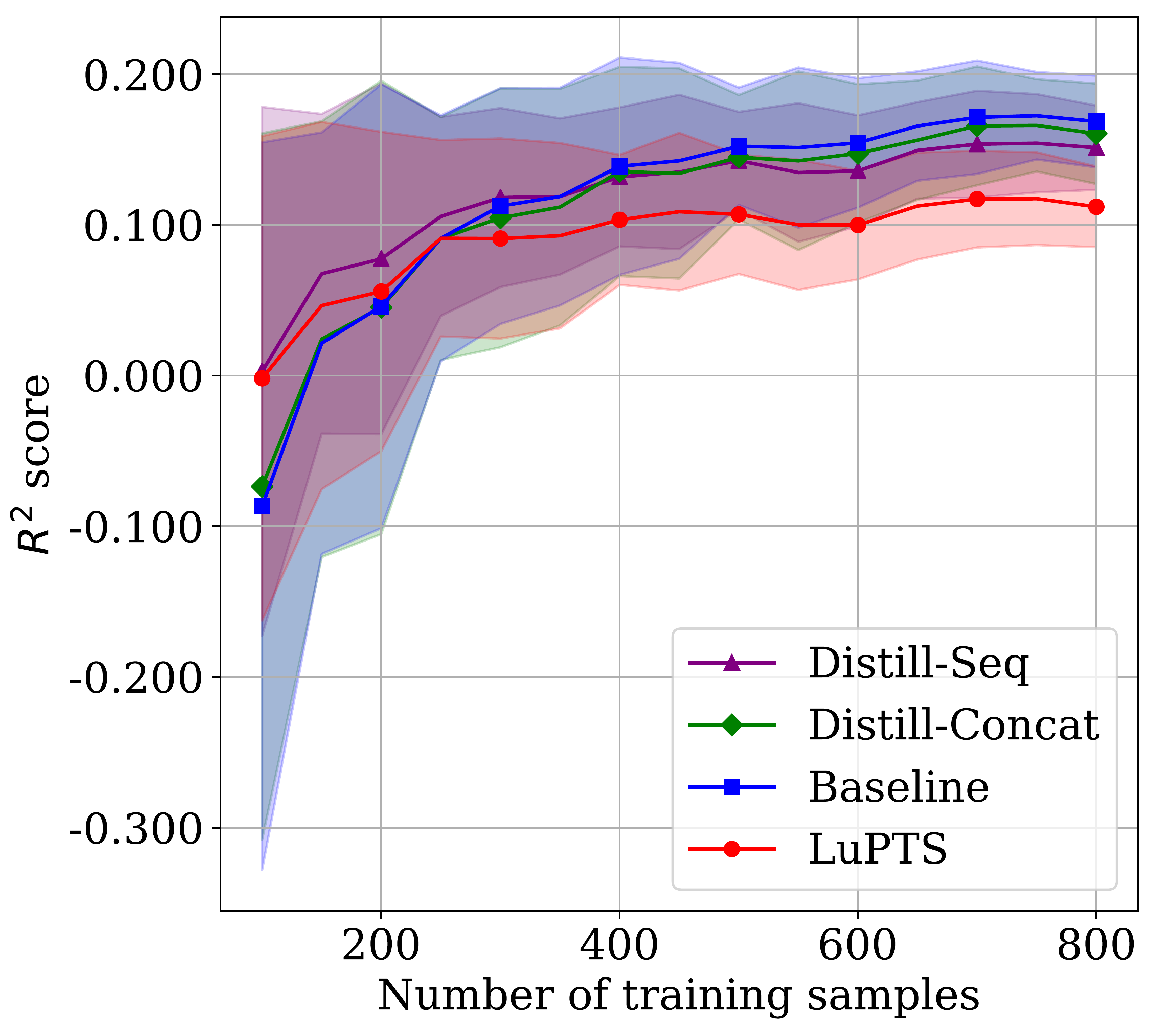}
        \caption{12 hour forecast}
        \label{fig:guangzhouT12_distill_app}
    \end{subfigure}
    \caption{\textbf{Guangzhou: }\ref{fig:guangzhouT6_app}, \ref{fig:guangzhouT24_app}) Changing the amount of privileged information for the LUPTS for different time horizons, where the \textit{X} in LuPTS\_\textit{X}PTS indicates the number of privileged time points. \ref{fig:guangzhouT6_distill_app}, \ref{fig:guangzhouT12_distill_app}) Comparing LuPTS to the distillation-based approaches, which use the same privileged information.
    Metric used is $R^2$ (Higher is better); shaded region indicates one standard deviation across 75 iterations.}
    \label{fig:pm25_guangzhou_app}
\end{figure}

\begin{table}[t!]
    \centering
    \caption{Comparison of regression methods on the air quality forecasting task with a fixed sample size $n=200$ and a prediction horizon of 6 hours. Metric used is $R^2$ (Higher is better); mean value for each method with standard deviation across 200 iterations. The methods with highest $R^2$ and lowest variance are marked in bold for each city.}
    \label{tab:fc_nonlinear_n200_6hour}
    \begin{tabular}{llllll}
    \toprule
        Method &     Beijing &    Shanghai &    Shenyang &     Chengdu &   Guangzhou \\
    \midrule
      Baseline  & 0.64 $\pm$ 0.02 & 0.58 $\pm$ 0.06 & 0.66 $\pm$ 0.04 & 0.63 $\pm$ 0.04 & 0.45 $\pm$ 0.06 \\
         LuPTS  & 0.64 $\pm$ 0.03 & \textbf{0.62 $\pm$ 0.03} & \textbf{0.70 $\pm$ 0.03} & 0.65 $\pm$ 0.03 & 0.49 $\pm$ 0.04 \\
    Stat-LuPTS  & 0.64 $\pm$ 0.02 & \textbf{0.62 $\pm$ 0.04} & 0.69 $\pm$ 0.03 & 0.65 $\pm$ 0.03 & \textbf{0.50 $\pm$ 0.04} \\
    Distill-Seq & \textbf{0.65 $\pm$ 0.02} & \textbf{0.62 $\pm$ 0.03} & 0.68 $\pm$ 0.03 & \textbf{0.67 $\pm$ 0.02} & 0.49 $\pm$ 0.05 \\
Distill-Concat  & \textbf{0.65 $\pm$ 0.02} & 0.60 $\pm$ 0.04 & 0.66 $\pm$ 0.04 & 0.66 $\pm$ 0.03 & 0.46 $\pm$ 0.07 \\
             RF & 0.62 $\pm$ 0.03 & 0.58 $\pm$ 0.07 & 0.53 $\pm$ 0.06 & 0.61 $\pm$ 0.04 & 0.48 $\pm$ 0.05 \\
            KNN & 0.57 $\pm$ 0.04 & 0.51 $\pm$ 0.23 & 0.49 $\pm$ 0.05 & 0.51 $\pm$ 0.04 & 0.26 $\pm$ 0.09 \\
    \bottomrule
    \end{tabular}
\end{table}
\begin{table}[t!]
    \centering
    \caption{Comparison of regression methods on the air quality forecasting task with a fixed sample size $n=200$ and a prediction horizon of 12 hours. Metric used is $R^2$ (Higher is better); mean value for each method with standard deviation across 200 iterations. The methods with highest $R^2$ and lowest variance are marked in bold for each city.}
    \label{tab:fc_nonlinear_n200_12hour}
    \begin{tabular}{llllll}
    \toprule
        Method &     Beijing &    Shanghai &    Shenyang &     Chengdu &    Guangzhou \\
    \midrule
      Baseline & 0.37 $\pm$ 0.05 & 0.29 $\pm$ 0.06 & \textbf{0.53 $\pm$ 0.07} & 0.35 $\pm$ 0.08 &  0.07 $\pm$ 0.11 \\
         LuPTS & 0.40 $\pm$ 0.04 & \textbf{0.33 $\pm$ 0.04} & 0.51 $\pm$ 0.07 & 0.42 $\pm$ 0.04 &  0.05 $\pm$ 0.14 \\
    Stat-LuPTS & 0.40 $\pm$ 0.04 & \textbf{0.33 $\pm$ 0.04} & 0.51 $\pm$ 0.07 & 0.42 $\pm$ 0.04 &  0.05 $\pm$ 0.14 \\
    Distill-Seq& \textbf{0.41 $\pm$ 0.03} & 0.31 $\pm$ 0.04 & 0.52 $\pm$ 0.08 & \textbf{0.43 $\pm$ 0.04} & 0.07 $\pm$ 0.11 \\
Distill-Concat & \textbf{0.41 $\pm$ 0.03} & 0.31 $\pm$ 0.04 & 0.52 $\pm$ 0.08 & \textbf{0.43 $\pm$ 0.04} & 0.07 $\pm$ 0.11 \\
            RF & 0.36 $\pm$ 0.05 & 0.23 $\pm$ 0.08 & 0.40 $\pm$ 0.07 & 0.38 $\pm$ 0.06 &  \textbf{0.14 $\pm$ 0.10} \\
           KNN & 0.30 $\pm$ 0.05 & 0.24 $\pm$ 0.07 & 0.35 $\pm$ 0.06 & 0.31 $\pm$ 0.06 & -0.05 $\pm$ 0.12 \\
    \bottomrule
    \end{tabular}
\end{table}

\newpage
\subsection{Alzheimer’s Progression Modeling}

In this section, we present the entire feature set used for the Alzheimer's disease progression modeling tasks (see Table~\ref{tab:features_adni}). 
All experimental results are also found in tabular form with values rounded to two decimals (see Tables~\ref{tab:MMSE_noselection_table} and~\ref{tab:AD_noselection_table}). Lastly, we give a detailed description of the data pre-processing that was performed for the dataset (ADNIMERGE).

\paragraph*{Pre-processing}
There are a significant number of missing values in the observations from the ADNI dataset. The missingness varies with the time of measurement, as does which subjects are present at certain follow-ups. Hence, a subset of the subjects in the study needs to be selected in order to carry out the experiments. This means that subjects without an observation of the target outcome at the follow-up at 48 months are excluded. Furthermore, it is required that the subjects with an observation of the target at this time point also are present at the intermediate follow-ups used as privileged information, which are 12 months, 24 months and 36 months after baseline. Categorical features, here considered to consist of biological sex (PTGENDER) and APOE4 gene expression, are one-hot encoded.
Additionally, if any feature has more than 70\% of the observations missing for the selected subjects at any of the time points in consideration, they are excluded. The features excluded as a result of this constraint are FDG, ABETA, TAU and PTAU. Finally, mean imputation is used for missing values and the data is zero-mean unit-variance standardized.

\begin{table}[t!]
    \centering
    \caption{Features used for the ADNI experiments}
    \label{tab:features_adni}
    \begin{tabular}{ccc}
        \toprule
        & \textbf{Feature tags} & \\ 
        \midrule
        AGE & PTGENDER & PTEDUCAT\\
        APOE4 & FDG & AV45\\ 
        ABETA & TAU & PTAU\\ 
        CDRSB & ADAS11 & ADAS13\\ 
        ADASQ4 & MMSE & RAVLT\_immediate\\
        RAVLT\_learning & RAVLT\_forgetting & RAVLT\_perc\_forgetting \\
        LDELTOTAL & TRABSCOR & FAQ\\
        MOCA & EcogPtMem & EcogPtLang\\
        EcogPtVisspat & EcogPtPlan & EcogPtOrgan\\
        EcogPtDivatt & EcogPtTotal & EcogSPMem\\
        EcogSPLang & EcogSPVisspat & EcogSPPlan\\
        EcogSPOrgan & EcogSPDivatt & EcogSPTotal\\
        Ventricles & Hippocampus & WholeBrain\\
        Entorhinal & Fusiform &  MidTemp\\
        ICV & & \\
        \bottomrule
    \end{tabular}
\end{table}

\begin{table}[t!]
    \centering
    \caption{MMSE prediction experiment results, average $R^2$ score with one standard deviation in parenthesis from 100 iterations. \textbf{Left}: One privileged time point used. \textbf{Right}: three privileged time points used.}
    \label{tab:MMSE_noselection_table}
    \begin{tabular}{lll}
    \toprule 
    Samples & Baseline     & LuPTS       \\ \midrule
    80      & -0.02 (0.34) & 0.14 (0.24) \\
    100     & 0.12 (0.29)  & 0.27 (0.17) \\
    120     & 0.27 (0.14)  & 0.36 (0.1)  \\
    140     & 0.36 (0.08)  & 0.42 (0.06) \\
    160     & 0.39 (0.08)  & 0.44 (0.05) \\
    180     & 0.42 (0.06)  & 0.46 (0.05) \\
    200     & 0.44 (0.05)  & 0.48 (0.04) \\
    220     & 0.45 (0.05)  & 0.49 (0.04) \\ \bottomrule
    \end{tabular}
    \quad
    \begin{tabular}{llll}
    \toprule
    Samples & Baseline     & LuPTS       & Stat-LuPTS  \\ \midrule
    80      & -0.02 (0.34) & 0.25 (0.16) & 0.39 (0.09) \\
    100     & 0.12 (0.29)  & 0.34 (0.11) & 0.43 (0.06) \\
    120     & 0.27 (0.14)  & 0.41 (0.06) & 0.46 (0.05) \\
    140     & 0.36 (0.08)  & 0.44 (0.05) & 0.47 (0.04) \\
    160     & 0.39 (0.08)  & 0.46 (0.04) & 0.47 (0.04) \\
    180     & 0.42 (0.06)  & 0.48 (0.05) & 0.49 (0.04) \\
    200     & 0.44 (0.05)  & 0.48 (0.04) & 0.49 (0.03) \\
    220     & 0.45 (0.05)  & 0.49 (0.03) & 0.5 (0.03)  \\ \bottomrule
    \end{tabular}

\end{table}

\begin{table}[t!]
    \centering
    \caption{AD prediction experiment results, average AUC with one standard deviation in parenthesis from 100 iterations \textbf{Left}: one privileged time point used. \textbf{Right}: three privileged time points used.}
    \label{tab:AD_noselection_table}
    \begin{tabular}{lll}
    \toprule
    Samples & Baseline    & LuPTS       \\ \midrule
    80      & 0.86 (0.04) & 0.85 (0.04) \\
    100     & 0.86 (0.06) & 0.87 (0.04) \\
    120     & 0.88 (0.03) & 0.9 (0.03)  \\
    140     & 0.89 (0.03) & 0.9 (0.02)  \\
    160     & 0.9 (0.02)  & 0.91 (0.02) \\
    180     & 0.9 (0.02)  & 0.92 (0.02) \\
    200     & 0.91 (0.02) & 0.92 (0.02) \\
    220     & 0.91 (0.02) & 0.92 (0.02) \\ \bottomrule
    \end{tabular}
    \quad
    \begin{tabular}{llll}
    \toprule
    Samples & Baseline    & LuPTS       & Stat-LuPTS  \\ \midrule
    80      & 0.86 (0.04) & 0.87 (0.04) & 0.9 (0.03)  \\
    100     & 0.86 (0.06) & 0.89 (0.03) & 0.91 (0.02) \\
    120     & 0.88 (0.03) & 0.91 (0.02) & 0.92 (0.02) \\
    140     & 0.89 (0.03) & 0.91 (0.02) & 0.92 (0.02) \\
    160     & 0.9 (0.02)  & 0.92 (0.02) & 0.93 (0.02) \\
    180     & 0.9 (0.02)  & 0.92 (0.02) & 0.93 (0.02) \\
    200     & 0.91 (0.02) & 0.93 (0.02) & 0.93 (0.02) \\
    220     & 0.91 (0.02) & 0.93 (0.02) & 0.93 (0.02) \\ \bottomrule
    \end{tabular}
\end{table}

\subsection{Multiple Myeloma Progression Modeling}
We elaborate on the specific features used for the multiple myeloma prediction tasks, as well as the preprocessing done on those features. The data is available via the Multiple Myeloma Research Foundation (MMRF) Researcher Gateway: \url{https://research.themmrf.org/}. 

\paragraph{Features} Patient biomarkers are real-valued numbers whose values evolve over time. They include: 
absolute neutrophil count (x$10^{9}$/l), albumin (g/l), blood urea nitrogen (mmol/l), calcium (mmol/l), serum creatinine (umol/l), glucose (mmol/l), hemoglobin (mmol/l), serum kappa (mg/dl),
serum m protein (g/dl),  platelet count x$10^9$/l, total protein (g/dl), white blood count  x$10^9$/l,  serum IgA (g/l), serum IgG (g/l), serum IgM (g/l), serum lambda (mg/dl).

We also have access to a set of static features, which we assume are available at each time step. These include demographics, risk metrics, and genomic data. With respect to the genomic features, RNA-sequencing data of CD38+ bone marrow cells are available for 769 patients. Samples from patients were collected at initiation into the study. For these patients, we use the Scanpy package in Python to identify the top 200 most variable genes, limiting the downstream analyses to these genes \citep{wolf2018scanpy}. Subsequently, we use principal component analysis (PCA) to further reduce the dimensionality of the RNA-seq data. The projection of each patient's RNA-seq data onto the first 40 principal components serves as the final genetic features in the model. 

Other static features include gender, age, and revised ISS stage, a common risk stratification score used in myeloma \citep{palumbo2015revised}. Finally, binary variables detailing the patient's myeloma subtype, including whether or not they have heavy chain myeloma and presence/absence of various monoclonal proteins, are part of this set of features as well.

\paragraph{Pre-processing}
We utilize the same preprocessing strategy used by \citet{hussain2021neural}. For the longitudinal patient biomarkers, we first clip the values to five times the median values to account for outliers or errors in the data. The biomarkers are then normalized by subtracting the maximum value of the biomarker's healthy range from the unnormalized value. Subsequently, the value is multiplied by a biomarker-dependent scaling factor that ensures that it lies within the range, $[-8,8]$. Aside from PCA done on the genomic data, we do zero mean, unit variance standardization on all the static features. 

The data has significant missingness, with around $\sim 66\%$ of the values missing. For static features aside from the genetic features, we use mean imputation. For the genetic features, a patient's missing values are imputed with the average genetic PCA features of their five nearest neighbors, which are determined using the Minkowski distance calculated on the ISS stage, age, and other demographic features. The missing values in the longitudinal biomarkers are forward-fill imputed from the previous 2-month time point. 

\paragraph{Evaluation}
We do repeated (50 repeats) 2-fold cross validation with different training and test splits across multiple training set sizes. For the early/late progression task, we exclude patients who are not eligible for autologous stem cell transplant (ASCT), resulting in 314 late progressors, 103 early progressors, and 84 right-censored patients. The privileged information for this task consists of labs taken at four equally spaced time points across the patient's first line. For the treatment response task, we restrict to patients who were given a second line of therapy, resulting in 149 PD patients, 181 non-PD patients, and 48 right-censored patients. We use two privileged time points in this case, which correspond to the end of the first line and the end of the second line, respectively. Censored patients are not included in computing AUCs. All labels have been checked with an oncologist for reliability.

\begin{figure}[t]
    \centering
    \caption{\textbf{Heatmap of Feature Weights for Early/Late Progression Task}: We display weights of the outcome logistic regression model for LuPTS and baseline OLS estimators. The x-axis shows each feature. Blue and white bounding boxes are around most important features for the LuPTS and OLS estimators, respectively.}
    \includegraphics[width=\textwidth]{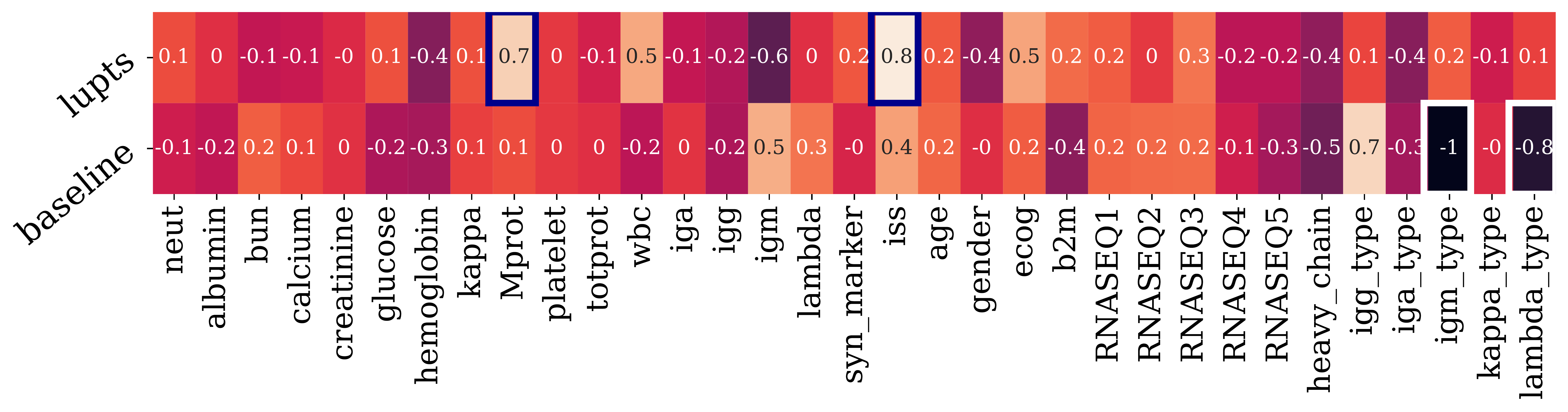}
    \label{fig:mm_qual}
\end{figure}

\end{document}